%% file: main.tex
\documentclass{article}

\PassOptionsToPackage{numbers, compress}{natbib}
    \usepackage[preprint]{neurips_2023}

\usepackage[utf8]{inputenc} % allow utf-8 input
\usepackage[T1]{fontenc}    % use 8-bit T1 fonts
\usepackage{hyperref}       % hyperlinks
\usepackage{url}            % simple URL typesetting
\usepackage{booktabs}       % professional-quality tables
\usepackage{amsfonts}       % blackboard math symbols
\usepackage{nicefrac}       % compact symbols for 1/2, etc.
\usepackage{microtype}      % microtypography
\usepackage{xcolor}         % colors
\usepackage{mdframed}

\usepackage[colorinlistoftodos]{todonotes}
\newcommand{\yf}[1]{{\leavevmode\color{blue} [yifang: #1]}}
\newcommand{\yifang}[1]{\todo[inline,color=yellow!10]{\textbf{YF: }#1}}

\input{header}
\title{Active Representation Learning for General Task Space with Applications in Robotics}

\author{%
    Yifang Chen$^1$, Yingbing Huang$^2$, Simon S. Du$^{1*}$,  Kevin Jamieson$^{1*}$, Guanya Shi$^{1*}$\\
    \\
    $^1$ Paul G. Allen School of Computer Science \& Engineering\\
    University of Washington, Seattle,WA\\
    \texttt{ \{yifangc, ssdu, jamieson, guanyas\}@cs.washington.edu}\\
    \\
    $^2$ Cornell University, Ithaca, NY\\
    \texttt{ \{yh374\}@cornell.edu}\\
    \\
    $^*$ Equal advising
}
\begin{document}
\maketitle

\input{abstract}

\input{Intro}
\input{preliminary}
\input{framework}
\input{theory-assumption}

\input{theory-algo}

\input{experiments}

\newpage
\bibliographystyle{unsrtnat}
\bibliography{ref.bib}

\newpage
\appendix
% \section{Results and analysis for Target-aware}

\paragraph{A typo in Theorem~\ref{them: main} in the main submission} We notice that there is a typo in the Theorem~\ref{them: main} in our submitted main paper, so we provide a correct version here. Therefore, you may find some terms in the main paper in this file are different from the one submitted before. We apologize for the inconvenience. We want to emphasize that this typo does not affect the main conclusion of our paper. It only affects the comparison between our upper bound and previous passive learning results in a specific setting. 

\tableofcontents
\newpage
% Results and analysis for Target-aware
\section{Result and analysis for target-aware}
\label{sec: target-aware (appendix)}
\input{appendix-trainingOracle}

\input{appendix-generalization1}

\input{appendix-sampleComplexity1}
\input{appendix-resultComparison}

% Results and analysis for Target-agnostic
\input{appendix-targetAgnostic}

% Future Directions
\input{appendix-theoryLimitations}

% More experiments
\input{appendix-experiment}

% \input{appendix-experiment}
% \input{theory-excessRisk}
% \input{theory-sampleComplexity}
% \newpage
% \yifang{Draft only}
% \input{Results1-CubeSpace/Results1-CubeSpace}
% \input{other_theory}

\end{document}

%% file: header.tex
\usepackage{graphicx}
\usepackage{mathtools}
\usepackage{footnote}
\usepackage{float}
\usepackage{xspace}
\usepackage{multirow}
\usepackage{wrapfig}
\usepackage{framed}
\usepackage[export]{adjustbox}
\usepackage{footnote}
\usepackage{soul}

\makesavenoteenv{tabular}
\makesavenoteenv{table}

\newcommand{\er}{{\textsc{er}}}

\usepackage{amsmath}
\usepackage{amsfonts}
\usepackage{amssymb}
\usepackage{amsthm}
\usepackage{bm}
\usepackage{bbm}
\usepackage{mathtools}
\usepackage{enumitem}
\usepackage{thmtools,thm-restate}
\usepackage{algorithm}
\usepackage{algorithmic}
\usepackage{graphicx}
\usepackage{comment}
\theoremstyle{plain}
\newtheorem{lemma}{Lemma}[section]

\newtheorem{theorem}{Theorem}[section]

\newtheorem{assumption}{Assumption}[section]
\newtheorem{definition}{Definition}[section]

\newtheorem{claim}{Claim}[section]

\theoremstyle{remark}
\newtheorem{remark}{Remark}[section]

\def\Ebb{\mathbb{E}}

\def\Rbb{\mathbb{R}}

\def\R{\Rbb}

\def\*{\star}

\DeclareMathOperator*{\argmin}{arg\,min}
\DeclareMathOperator*{\argmax}{arg\,max}

\newcommand{\E}{\Ebb}

\usepackage{graphicx}
\usepackage{mathtools}
\usepackage{footnote}
\usepackage{float}
\usepackage{xspace}
\usepackage{multirow}
\usepackage{wrapfig}
\usepackage{framed}
\usepackage{xcolor}

\newcommand{\field}[1]{\mathbb{#1}}

\newcommand{\fR}{\field{R}}

\newcommand{\fB}{\field{B}}

\newcommand{\calB}{{\mathcal{B}}}
\newcommand{\calS}{{\mathcal{S}}}

\newcommand{\calE}{{\mathcal{E}}}

\newcommand{\calX}{{\mathcal{X}}}

\newcommand{\calN}{{\mathcal{N}}}

\newcommand{\calW}{{\mathcal{W}}}
\newcommand{\calO}{{\mathcal{O}}}
\newcommand{\calR}{{\mathcal{R}}}

\newcommand{\order}{\ensuremath{\mathcal{O}}}
\newcommand{\otil}{\ensuremath{\widetilde{\mathcal{O}}}}

\newcommand{\trace}{\text{Tr}}

%% file: abstract.tex
\begin{abstract}
    % ***** Outline *****
    % Motivation: why (why active? why continuous)? For instance, in robotics, ...
    % General framework: which is the first algorithmic and theoretic framework for ...
    % Different instantiations and compatible with modern deep representation learning ...
    % In these instantiations leveraged xxx, we achieved the first xxx theoretical guarantees ...
    % We also empirically our algorithms in different settings, from xxx to xxx. In particular, our methods show performance improvement up to x% compared to standard passive learning baselines.

    % ***** Rephrase *****
    Representation learning based on multi-task pretraining has become a powerful approach in many domains. 
    %For instance, robotic systems often need to operate in various environmental conditions, and an efficient representation allows fast adaptation to new conditions. 
    % In particular, task-aware representation learning aims to learn an optimal representation for a specific target task by sampling data from a set of source tasks, while task-agnostic representation learning seeks to learn a universal representation for a class of tasks. 
    In particular, task-aware representation learning aims to learn an optimal representation for a specific target task by sampling data from a set of source tasks, while task-agnostic representation learning seeks to learn a universal representation for a class of tasks.  
    % In task-aware representation learning, the learner has access to a limited dataset sampled in the target task, and it needs to learn a representation with the best performance on the target task by sampling training data from a source task set.
    In this paper, we propose a general and versatile algorithmic and theoretic framework for \emph{active representation learning}, where the learner optimally chooses which source tasks to sample from. This framework, along with a tractable meta algorithm, allows most arbitrary target and source task spaces (from discrete to continuous), covers both task-aware and task-agnostic settings, and is compatible with deep representation learning practices. 
    We provide several instantiations under this framework, from bilinear and feature-based nonlinear to general nonlinear cases. In the bilinear case, by leveraging the non-uniform spectrum of the task representation and the calibrated source-target relevance, we prove that the sample complexity to achieve $\varepsilon$-excess risk on target scales with 
    % $ \otil(k\varepsilon^{-\frac{4}{3}} + k^*\varepsilon^{-2})$ whereas standard random sampling gives $k\varepsilon^{-2}$. Here $k$ is the dimension of shared representation space and $k^*$ measures the effective dimension in the source space that is target-relevant (in the target-agnostic case $k^*=k$). \yf{I am not sure what is the right way to state this. The actual bound is kinda complicated and requires many definitions. But the current bound is not easy for people who are familar with this area to compare} This is the first active representation learning guarantee under continuous task space with any arbitrary target task set to the best of our knowledge.
     $ (k^*)^2 \|v^*\|_2^2 \varepsilon^{-2}$ where $k^*$ is the effective dimension of the target and $\|v^*\|_2^2 \in (0,1]$ represents the connection between source and target space. Compared to the passive one, this can save up to $\frac{1}{d_W}$ of sample complexity, where $d_W$ is the task space dimension. 
    Finally, we demonstrate different instantiations of our meta algorithm in synthetic datasets and robotics problems, from pendulum simulations to real-world drone flight datasets. On average, our algorithms outperform baselines by $20\%-70\%$.

\end{abstract}

%% file: Intro.tex
\section{Introduction}

% \guanya{Need one paragraph to talk about the importance of multi-task-pretraining-based representation learning, and the importance of making it active (e.g., data is expensive, switching between tasks is even more expensive in many applications such as robotics).}

% \guanya{
% Here is the structure for intro we can consider:
% \begin{itemize}
%     \item Representation learning is so importance (why rep learning?). 
%     \item It is quite expensive especially task-wise (why AL?). Examples in robotics and beyond. 
%     \item Previous AL for rep learning lacks XXX. In particular, theoretically, there are X challenges 
%     \item In this paper, we directly deal with these X challenges ...
% \end{itemize}
% }

Recently, few-shot machine learning has enjoyed significant attention and has become increasingly critical due to its ability to derive meaningful insights for target tasks that have minimal data, a scenario commonly encountered in real-world applications.
This issue is especially prevalent in robotics where data collection and training data is prohibitive to collect or even non-reproducible (e.g., drone flying with complex aerodynamics~\cite{shi2019neural} or legged robots on challenging terrains~\cite{lee2020learning}). One promising approach to leveraging the copious amount of data from a variety of other sources is multi-task learning, which is based on a key observation that different tasks may share a common low-dimensional representation. This process starts by pretraining a representation on source tasks and then fine-tuning the learned representation using a limited amount of target data (\citep{radford2021learning,brown2020language,yu2022coca,alayrac2022flamingo,o2022neural}).

In conventional supervised learning tasks, accessing a large amount of source data for multi-task representation learning may be easy, but processing and training on all that data can be costly. In real-world physical systems like robotics, this challenge is further amplified by two factors: (1) switching between different tasks or environments is often significantly more expensive (e.g., reset giant wind tunnels for drones~\cite{o2022neural}); (2) there are infinitely many environments to select from (i.e., environmental conditions are continuous physical parameters like wind speed). Therefore, it is crucial to minimize not only the number of samples, but the number of sampled source tasks, while still achieving the desired performance on the target task. Intuitively, not all source tasks are equally informative for learning a universally good representation or a target-specific representation. This is because source tasks can have a large degree of redundancy or be scarce in other parts of the task space.
% learning the representation and maximizing performance on the target tasks. 
In line with this observation, \citet{chen2022active} provided the first provable active representation learning method that improves training efficiency and reduces the cost of processing source data by prioritizing certain tasks during training with theoretical guarantees. On the other hand, many existing works \cite{du2021fewshot, tripuraneni2020theory, tripuraneni2021provable, thekumparampil2021sample, xu2021representation} prove that it is statistically possible to learn a universally good representation by randomly sampling source tasks (i.e., the passive learning setting). 

The previous theoretical work of \cite{chen2022active} on active multi-task representation learning has three main limitations. First, it only focuses on a finite number of discrete tasks, treating each source independently, and therefore fails to leverage the connection between each task. This could be sub-optimal in many real-world systems like robotics for two reasons: (1) there are often infinitely many sources to sample from (e.g., wind speed for drones); (2) task spaces are often highly correlated (e.g., perturbing the wind speed will not drastically change the aerodynamics).
% For example, \yf{add a robotics example that when we know the physics relation between environment A and B is close and environment A and C is far-away, there is no reason to treat them equally}. \guanya{``This could be sub-optimal in many real-world systems like robotics for two reasons: (1) there are often infinitely many sources to sample from (e.g., wind speed for a drone); (2) task spaces are often highly correlated (e.g., perturbing the wind speed will not drastically change the aerodynamics).''}
In our paper, by considering a more general setting where tasks are parameterized in a vector space $\calW$, we can more effectively leverage similarities between tasks compared to treating them as simply discrete and different.
Secondly, the previous work only considers a single target, 
% which is impractical in the robotics where the physical system is naturally unstable
while we propose an algorithm that works for an arbitrary target space and distribution. This is particularly useful when the testing scenario is time-variant. Thirdly, we also consider the task-agnostic setting by selecting $\order(k)$ representative tasks among the $d_W$ high dimension task space, where $k \ll d_W$ is the dimension of the shared representation. Although this result does not improve the total source sample complexity compared to the passive learning result in the bilinear setting~\cite{thekumparampil2021sample}, it reduces the number of tasks used in the training and therefore implicitly facilitates the training process. % \guanya{what is task number cost? The last sentence is a bit confusing.} \yf{modified}
 
In addition to those theoretical contributions, we extend our proposed algorithmic framework beyond a pure bilinear representation function, including the \textit{known} nonlinear feature operator with unknown linear representation (e.g., random features with unknown coefficients), and the totally \textit{unknown nonlinear representation} (e.g., deep neural network representation). 
% 
% While there are some prior works that consider nonlinear representations \cite{du2021fewshot, tripuraneni2020theory, collins2022maml, xu2021representation} in passive learning, the studies in active learning are relatively limited. Previous active learning conducts experiments on the nonlinear representation of input but not the task parameter. (Since we are doing task-wise AL, such nonlinear will add further difficulty.) In this paper, we did more comprehensive experiments on different model instantiations and show the effectiveness of our algorithm.
% 
While some prior works have considered nonlinear representations  \cite{du2021fewshot, tripuraneni2020theory, collins2022maml, xu2021representation} in passive learning, the studies in active learning are relatively limited \cite{chen2022active}. All of these works only consider non-linearity regarding the input, rather than the task parameter. In this paper, we model task-parameter-wise non-linearity and show its effectiveness in experiments. Note that it particularly matters for task selections because the mapping from the representation space to task parameters to is no longer linear.

\subsection{Summery of contributions}

\begin{itemize}[leftmargin=*]
    \item We propose the first generic active representation learning framework that admits any arbitrary source and target task space. This result greatly generalizes previous works where tasks lie in the discrete space and only a single target is allowed. To show its flexibility, we also provide discussions on how our framework can accommodate various supervised training oracles and optimal design oracles. (Section~\ref{sec: framework})
    \item We provide theoretical guarantees under a benign setting, where inputs are i.i.d. and a unit ball is contained in the overall task space, as a compliment to the previous work where tasks lie on the vertices of the whole space. In the target-aware setting, to identify an $\varepsilon$-good model our method requires a sample complexity of $\otil(kd_X (k^*)^2 \|v^*\|_2^2 \min\{k^*, \kappa^2\}\varepsilon^{-2})$ where $k^*$ is the effective dimension of the target, $\kappa$ is the conditional number of representation matrix, and $\|v^*\|_2^2 \in (0,1]$ represents the connection between source and target space that will be specified in the main paper.  Compared to passive learning, our result saves up to a factor of $\frac{k^2}{d_W}$ in the sample complexity when targets are uniformly spread over the $k$-dim space and up to $\frac{1}{d_W}$ when targets are highly concentrated. Our results further indicate the necessity of considering the continuous space by showing that directly applying the previous algorithm onto some discretized sources in the continuous space (e.g., orthonormal basis) can lead to worse result. Finally, ignoring the tasks used in the warm-up phases, in which only a few samples are required,  both the target-aware and the target-agnostic cases can save up to $\otil(k^* + k)$ number of tasks compared to the passive one which usually requires $d_W$ number of tasks. (Section~\ref{sec: theory}) 
    % \yf{Add "ignore ..." explanation, see if it makes sense.}
    % \yf{pending, add the bound}, where $k \gg d$ is the dimension of the shared embedding space and the $k_\text{effect} \leq k$ is the target-related dimension. \yf{I think we should also emphasize why continuous space is important in a robotics setting. (Or maybe put that in the intro)} \guanya{need the same notation $k^*$}
    % 
    \item We provide comprehensive experimental results under different instantiations beyond the benign theoretical setting, studying synthetic and real-world scenarios: 1) For the synthetic data setting in a continuous space, we provide results for pure linear, known nonlinear feature operator $\psi_X$ and unknown nonlinear representation $\phi_X$. Our target-aware active learning (AL) approach shows up to a significant budget saving (up to $68\%$) compared to the passive approach and the target-agnostic AL approach also shows an advantage in the first two cases. 2) In a pendulum simulation with continuous task space, we provide the results for known nonlinear feature operator $\psi_X$ and $\psi_W$ and show that our target-aware AL approach has up to $20\%$ loss reduction compared to the passive one, which also translates to better nonlinear control performance. 3) Finally, in the real-world drone dataset with a discrete task space, we provide results for unknown linear and nonlinear representation $\phi_X$ and show that our target-aware AL approach converges much faster than the passive one. (Section~\ref{sec: experiment})
    % \yf{? depends on how many we have in the end} simulation results that not only support our theoretical proofs but also show that flexibility \yf{scalability? adaptability?} when 1) the high dim inputs features come from the non-independent Fourier kernel, 2) input embedding functions are nonlinear MLP. Our offline optimization oracle is no longer problem-specific, therefore has more practical potential to plug any modern deep learning oracles, Specially, we show \yf{Add a generalization of our results.} \yifang{Or talk less.}
    % \item Finally, we show the effectiveness of our algorithm on the pendulum simulator in continuous cases with fourier augmented features and the real-world drone dataset in discrete cases \yf{Add more description of dataset}, \yf{add more result.} Those results paves the next step in hardware simulation.
\end{itemize}

% \guanya{The last two items seem quite repetitive. The third one should be more about we provide different instantiations and implementations beyond the benign setting, ... Perhaps some (weak) theoretical result?}

%% file: preliminary.tex
\section{Preliminary}

\paragraph{Multi-task (or multi-environments).}
Each task or environment is parameterized by a known vector $w \in \fR^{d_W}$. We denote the source and target task parameter space as $\calW_\text{source} \subset \fR^{d_W}, \calW_\text{target} \subset \fR^{d_W}$. These spaces need not be the same (e.g., they could be different sub-spaces). In the discrete case, we set $w$ as a one-hot encoded vector and therefore we have in total $d_W$ number of candidate tasks while in the continuous space, there exist infinitely many tasks. For convenience, we also use $w$ as the subscript to index certain tasks. In addition, we use $\nu_\text{source} \in \Delta(\calW_\text{source}), \nu_\text{target} \in \Delta(\calW_\text{target})$ to denote the task distribution for the sources and targets.

\paragraph{Data generation.}
Let $\calX \in \fR^{d_X}$ be the input space.
We first assume there exists some \textit{known} feature/augmentation operator $\psi_X: \calX \to \fR^{d_{\psi_X} \geq d_W}, \psi_W: \calW \to \fR^{d_{\psi_W} \geq d_W}$, that can be some non-linear operator that lifts $w, x$ to some higher dimensional space (e.g., random Fourier features~\cite{rahimi2008uniform}). Notice that the existence of non-identical $\psi$ indicates the features are not pairwise independent and the design space of $\calW_\text{source}$ is not benign (e.g., non-convex), which adds extra difficulty to this problem. 

Then we assume there exists some \textit{unknown} underlying representation function $\phi_X: \psi(\calX) \rightarrow \calR$ which maps the augmented input space $\psi(\calX)$ to a shared representation space $\calR \in \R^k$ where $k \ll d_{\psi_X}, k \leq d_{\psi_W}$, and its task counterparts $\phi_W: \psi(\calW) \rightarrow \calR$ which maps parameterized task space to the feature space.
Here the representation functions are restricted to be in some function classes $\Phi$, e.g., linear functions, deep neural networks, etc. 

In this paper, we further assume that $\phi_W$ is a linear function $B_W \in \fR^{k \times d_{\psi_W}}$. 
% and leave the non-linear one in future work. 
To be more specific, for any fixed task $w$, we assume each sample $(x,y) \sim \nu_w$ satisfies
\begin{equation}\label{equ_data_gen}
    y = \phi_X ( \psi_X(x))^\top B_W \psi_W(w) + \xi, \quad \xi \sim \mathcal{N}(0, \sigma^2)
\end{equation}
For convenience, we denote $Z_w$ as the collection of $n_w$ sampled data $(x_{w}^1, y_{w}^1),...,(x_{w}^{n_w}, y_{w}^{n_w}) \sim \mu_w$. 
% , and $Y_w = [y_{w,1}, ..., y_{w,n_t}]^\top \in \R^{n_w}$ from the $w$-th task.
% For convenience, we denote $X_w = [x_{w,1}, .., x_{w,n_w}]^\top \in \R^{n_w \times d_X}$ to be the input data matrix which contained $n_w$ sampled data $(x_{w,1}),...,(x_{w,n_w}) \sim \mu_w$ from the $w$-th task, and $Y_w = [y_{w,1}, ..., y_{w,n_t}]^\top \in \R^{n_w}$, $Z_w = [z_{w,1}, ..., z_{w,n_w}] \in \R^{n_w}$ to be the labels and noise terms aligned to the inputs. And therefore, $N_{tot}$ is the total sampling number from all the source tasks. 
We note that when $\psi_X, \psi_W$ is identity and $\phi_X$ is linear, this is reduced to standard linear setting in many previous papers \cite{du2021fewshot,tripuraneni2021provable, thekumparampil2021sample,chen2022active}.

\paragraph{The task diversity assumption.} 
There exists some distribution $p \in \Delta(\calW_\text{source})$ that 
$\E_{w \sim p}\lambda_{\min}(B_W \psi_W(w) \psi_W(w)^\top B_W^\top) > 0$, which suggests the source tasks are diverse enough to learn the representation. 

\paragraph{Data collection protocol.}
\label{paragraph: target-aware data collection}
% We assume there exists some i.i.d data sampling oracle given the environment and the budget. 
% % Theoretically, we assume each sample is collected i.i.d. Practically, we use the random walk to simulate the close-to-i.i.d sampling. \guanya{We don't have to talk about any approximation/implementations before the exp section} Extending this to the online Markov decision process remains to be the future work. 
% To learn a proper representation, we are allowed access to an \textit{unlimited} $n_\text{source}$ number of data from source tasks during the learning process.
% % 
% Moreover, in a target-aware setting $\calW_\text{target} \not\subset \calW_\text{source}$, we assume only a \textit{few-shot} of \textit{mix} target data $Z_\text{target}: = \{x_{w_i}^i,y_{w_i}^i\}_{i \in [n_\text{target}]}$ where $w_i \sim \nu_\text{target}$. In addition, we assume a few-shot of \text{known environment} target data $\dot{Z}_\text{target}: = \{x_{w}^i,y_{w}^i, w\}_{i \in [\dot{n}_\text{target}], w \in \dot{W}_\text{target}}$, where $|\dot{W}_\text{target}| = \text{dim}(\calW_\text{target})$ and $ \dot{W}_\text{target} = \argmax_{W \in \calW_\text{target}} \lambda_{\min} (W W^\top)$. Usually $\dot{n}_\text{target} \approx n_\text{target}^{3/4}$. \guanya{Do we need this for single-task-specific setting? This assumption is indeed strong so we need to make it compelling. Is $\dot{Z}_\text{target}$ itself enough to learn a rep?}

We assume there exists some i.i.d. data sampling oracle given the environment and the budget. To learn a proper representation, we are allowed access to an \textit{unlimited} $n_\text{source}$ number of data from source tasks during the learning process by using such an oracle.
Then at the end of the algorithm, we are given a few-shot of \textit{mix} target data $Z_\text{target} = \{Z_w\}_{w \sim \nu_\text{target}}$ which is used for fine-tuning based on learned representation $\hat{\phi}_X$. Denote $n_\text{target}$ as the number of data points in $Z_\text{target}$.

\paragraph{Data collection protocol for target-aware setting.} When the target task is not a singleton, we additionally assume a few-shot of \textit{known environment} target data $\dot{Z}_\text{target}: = \{Z_w, w\}_{w \in \dot{W}_\text{target}}$, where $|\dot{W}_\text{target}| = \text{dim}(\calW_\text{target})$ and $ \dot{W}_\text{target} = \{ \argmax_{W \in \calW_\text{target}} \lambda_{\min} (W W^\top) \}$. Again denote $\dot{n}_\text{target}$ as the number of data points in $\dot{Z}_\text{target}$, we have $\dot{n}_\text{target} \approx n_\text{target}^{2/3} \ll n_\text{source}$. 

\begin{remark}
    Here $|\dot{W}_\text{target}|$ represents vectors that can cover every directions of $\calW_\text{target}$ space. 
    % \kevin{I don't know what ``support vector'' means in this context.} \yf{How about now}
    This extra $\dot{Z}_\text{target}$ requirement comes from the non-linearity of $l_2$ loss and the need to learn the relationship between sources and targets. We want to emphasize that such an assumption implicitly exists in previous active representation learning \cite{chen2022active} since $\dot{Z}_\text{target} = Z_\text{target}$ in their single target setting.
    Nevertheless, in a passive learning setting, only \text{mixed} $Z_\text{target}$ is required since no source selection process involves. Whether such a requirement is necessary for target-aware active learning remains an open problem.
\end{remark}

% On the other hand, When $\calW_\text{target} \not\subset \calW_\text{source}$, we assume only a few-shot or even zero-shot (target-agnostic) of target data. Specifically, in the target-aware setting, In addition, we assume a few-shot of \text{known} environment target data $\{x_{w}^i,y_{w}^i, w_i\}_{i \in [n_\text{target}]}$ 

% we assume there exists  $\{x_{w_i}^i,y_{w_i}^i\}_{i \in [n_\text{target}]}$ where $w_i \sim \nu_\text{target}$. In addition, we assume a few-shot of \text{known} environment target data

\paragraph{Other notations.} Let $e_i$ to be one-hot vector with $1$ at $i$-th coordinates and let $\epsilon_i = 2^{-i}$.  

% Especially when $w$ has strong heterogeneity property, more formally, when $B_W$ is with high conditional number.

\subsection{Goals}

\paragraph{Expected excess risk.} 
For any target task space $\calW_\text{target}$ and its distribution $\nu_\text{target}$ over the space, as well as a few-shot examples as stated in section~\ref{paragraph: target-aware data collection}, our goal is to minimize the expected excess risk with our estimated $\hat{\phi}_X$
\begin{align*}
    \er(\hat{\phi}_X, \nu_\text{target})  = \E_{w_0 \sim \nu_\text{target}} \E_{(x,y) \sim \nu_{w_0}}  \| \hat{\phi}_X(\psi_X(x))^\top \hat{w}_\text{avg} - y \|_2
\end{align*}
where $\hat{w}_\text{avg} = \argmin_w \sum_{(x,y) \in Z_\text{target}} \|\hat{\phi}_X(\psi_X(x)) w - y\|_2$, which average model estimation that captures the data behavior under the expected target distribution. Note that the $\calW_\text{target},\nu_\text{target}$ are given in advance in the target-aware setting.

\paragraph{The number of tasks.}
Another side goal is to save the number of long-term tasks we are going to sample during the learning process. Since a uniform exploration over $d_W^\text{source}$-dimension is unavoidable during the warm-up stage, we define long-term task number as 
\begin{align*}
    \left\lvert \left\{ w \in \calW_\text{source} \mid \ n_w \geq \Tilde{\Omega}(\varepsilon^{- \alpha}) \right\} \right\rvert
    % \quad \text{where } \alpha \text{ is some arbitrary exponent and } \varepsilon \text{ is the target accuracy}
\end{align*}
where $\alpha$ is some arbitrary exponent and $\varepsilon$ is the target accuracy and $n_w$ is number of samples sampled from task $w$ as defined above.

%% file: framework.tex
\section{A general framework}
\label{sec: framework}

\begin{algorithm}[ht]
\caption{Active multi-task representation learning (general templates)}
\label{algo: main-general}
\begin{algorithmic}[1]
\STATE \textbf{Inputs:} Candidate source set $\calW_\text{source}$. Classes of candidate representation function $\Phi_X, \Phi_W$ and the known feature operator $\psi_X, \psi_W$. 
\STATE \textbf{[Target-aware only] Inputs:} Target set $\calW_\text{target}$ and distribution $\nu_\text{target}$. Few-shot sample $\dot{Z}_\text{target}$ as defined in the preliminary.
\STATE \textcolor{blue}{Stage 1: Coarse exploration. (Warm-up stage) }
\STATE Set initial sampling distribution $q_0 = g(\psi_W, I_{d_{\psi_W}})$ where $g$ is defined in Eqn.~\ref{eqn: optimal design}
\STATE Set $n_0 \approx \text{poly}(d_{\psi_X},k) + \text{poly}(d_{\psi_W},k)$. Collect $n_0 q_0(w)$ data for each task denoted as $\{Z_w\}_{w | q_0(w) \neq 0}$ and update $\hat{\phi}_X \leftarrow \calO_\text{offline 0}^X(\{Z_w\}_{w | q_0(w) \neq 0 }, \psi_X)$ and $\hat{B}_W \leftarrow \calO_\text{offline}^W(\{Z_w\}_{w | q_0(w) \neq 0 }, \hat{\phi}_X)$ 
\FOR{$j = 1,2,3, \ldots$}
    \STATE \textcolor{blue}{Stage 2: Fine target-agnostic exploration (Directly choose $q_1^j = q_0$ when $k = \Theta(d_W)$)}
    \STATE Compute the exploration sampling distribution $q_1^j = g(\hat{B}_W \circ \psi_W, I_k)$
    \STATE $n_1^j \approx\text{poly}(d_{\psi_X},k) \epsilon_j^{-\frac{4}{3}} $. Collect $n_1^j q_1^j(w)$ data for each task denoted as $\{Z_w\}_{w | q(w) \neq 0}$ and update $\hat{\phi}_X \leftarrow \calO_\text{offline 1}^X(\{Z_w\}_{w | q_1^j(w) \neq 0 }, \psi_X)$ and $\hat{B}_W \leftarrow \calO_\text{offline}^W(\{Z_w\}_{w | q_1^j(w) \neq 0 }, \dot{Z}_\text{target}, \hat{\phi}_X)$ 
    \STATE \textcolor{blue}{\textbf{[Target-aware only]} Stage 3: Fine target-aware exploration}
    \STATE Compute the exploitation sampling distribution $q_2^j = g(\hat{B}_W \circ \psi_W, \Sigma_\text{regu} )$
    % where $\Sigma_\gamma$ is the approximate sparse version $\hat{B}_W (\E_{w_0 \sim \nu_0} w_0 w_0^\top) \hat{B}_W^\top$ \\
    where $\Sigma_\text{regu}$ is the regularized version of $\hat{B}_W (\E_{w_0 \sim \nu_0} w_0 w_0^\top) \hat{B}_W^\top$ 
    after clipping out insignificant eigenvalues. \\
    \STATE Set $n_2^j \approx \text{poly}(d_{\psi_X},k) \epsilon_j^{-2} $. Collect $n_2^j q_2^j (w)$ data for each task denoted as $\{Z_w\}_{w | q_2^j(w) \neq 0}$ and update $\hat{\phi}_X \leftarrow \calO_\text{offline 3}^X(\{Z_w\}_{w | q_1^j(w) \neq 0 \text{ and } q_2^j(w) \neq 0}, \psi_X)$.
\ENDFOR
\STATE \textbf{Return }  $\hat{\phi}_X$
\end{algorithmic}
\end{algorithm}

% \kevin{reword this sentence, I'm not sure what ``the next effective...'' means} \yf{How about now}
Our algorithm~\ref{algo: main-general} iteratively estimates the shared representation $\hat{\phi}_X, \hat{B}_W$ and the next 
% effective sampling 
target relevant source tasks which the learner should sample from by solving several optimal design oracles 
\begin{align}
\label{eqn: optimal design}
    g(f, A) = \min_{q \in \Delta(\calW_{source})} \lambda_{\max} \left( (\int q(w) f(w) f(w)^\top)^{-1} A \right)
\end{align}
This exploration and exploitation (target-aware exploration here) trade-off is inspired by the classical $\epsilon$-greedy strategy, but the key difficulty in our work is to combine that with multi-task representation learning and different optimal design problems. The algorithm can be generally divided into three parts, and some parts can be skipped depending on the structure and the goal of the problem.
\begin{itemize}[leftmargin=*]
    \item \textbf{Coarse exploration:}  The learner uniformly explores all the directions of the $\calW_\text{source}$ (denoted by distribution $q_0$) in order to find an initial $k$-dimension subspace $V$ that well spans over the representation space (i.e., $ \frac{1}{c}B_W B_W^\top \leq B_W V V^\top B_W^\top \leq c B_W B_W^\top$ for some arbitrary constant $c \leq \frac{d_{\psi_W}}{k}$). To give an intuitive example, suppose $B_W \in \fR^{2 \times d_{\psi_W}^\text{source} + 1}$ has the first half column equals $e_1$ and the second hard equals $e_2$. Then instead of uniformly choosing $\{e_i\}_{i \in [d_{\psi_W}^\text{source}]}$ task, we only need explore over two tasks $V[1] = \sqrt{\frac{2}{d_{\psi_W}^\text{source}}}[1,1, \ldots, 0, 0, \ldots], V[2] = \sqrt{\frac{2}{d_W^\text{source}}}[0,0, \ldots, 1, 1, \ldots]$. We want to highlight that the sample complexity of this warm-up stage only scales with $d_{\psi_X}, k$ and the spectrum-related parameters of $B_W$ (i.e., $\kappa(B_W), \sigma_{\min}(B_X)$), not the desired accuracy $\varepsilon$.
    % \yifang{Or I can use the example that all the columns are $e_1$ except one column is $e_2$, which is the imbalance case, which one do you think is better?}
    \item \textbf{Fine target-agnostic exploration: } The learner iteratively updates the estimation of $V$ and uniformly explore for $\otil(\epsilon_j^{-\frac{4}{3}})$ times on this $k$, instead of $d_{\psi_W}$ subspace, denoted by distribution $q_1$. (Note this $\epsilon_j^{-\frac{4}{3}}$ comes from the exploration part in $\epsilon$-greedy, which is $(n_2^j)^{\frac{2}{3}}$)
    Such reduction not only saves the cost of maintaining a large amount of physical environment in real-world experiments but also simplifies the non-convex multi-task optimization problem. Of course, when $k = \Theta(d_{\psi_W})$, we can always uniformly explore the whole $(d_{\psi_W}$ space as denoted in the algorithm. 
    Note that theoretically, $q_1$ only needs to be computed once as shown in \ref{sec: theory}. In practice, to further improve the accuracy while saving the task number, the $q_1$ can be updated only when a significant change from the previous one happens, which is adopted in our experiments as shown in appendix~\ref{sec: synthetic data (appendix)}.
    \item \textbf{Fine target-aware exploration.} In the task-awareness setting, the learner estimates the most-target-related sources parameterized by $\{w\}$ based on the current representation estimation and allocates more budget on those, denoted by distribution $q_2$. By definition, $q_2$ should be more sparse than $q_1$ and thus allowing the final sample complexity only scales with $k^*$, which measures the effective dimension in the source space that is target-relevant.
\end{itemize}

\paragraph{Computational oracle for optimal design problem.}
\label{para: optimal design oracle}
Depending on the geometry of $\{\psi_W(w)\}_{w \in \calW_\text{source}}$, the learner should choose proper offline optimal design algorithms to solve $g(f,A)$. Here we propose several common choices. 1). When $\calW_\text{source}$ contains a ball, we can approximate the solution via an eigendecomposition-based closed-form solution with an efficient projection as detailed in Section~\ref{sec: theory}. 2) When $\calW_\text{source}$ is some other convex geometry, we can approximate the result via the Frank-Wolfe type algorithms \citep{todd2016minimum}, which avoids explicitly looping over the infinite task space. 3) For other even harder geometry, we can use discretization or adaptive sampling-based approximation \citep{akimoto2012theoretical}. In our experiments, we adopt the latter one and found out that its running time cost is almost neglectable in our pendulum simulator experiment in Section~\ref{sec: experiment}, where the $\psi_W$ is a polynomial augmentation.

\paragraph{Offline optimization oracle $\calO_\text{offline}^X$.}
Although we are in the continuous setting, the sampling distribution $q_0, q_1, q_2$ is sparse. Therefore, our algorithm allows any proper passive multi-task learning algorithm, either theoretical or heuristic one, to plugin the $\calO_\text{offline}^X$. Some common choices include gradient-based joint training approaches\citep{raghu2019rapid,nichol2018reptile,antoniou2018train,hospedales2021meta}, the general non-convex ERM \cite{du2021fewshot} and other more carefully designed algorithms \cite{thekumparampil2021sample, chen2021weighted}. We implement the first one in our experiments (Section~\ref{sec: experiment}) to tackle the nonlinear $\psi_X, \phi_X$ and give more detailed descriptions of the latter two in Section~\ref{sec: theory} and Appendix~\ref{sec: training oracle (appendix)} to tackle the bilinear model.

%% file: theory-assumption.tex
\section{A theoretical analysis under the benign $\calW_\text{source}$ setting}
\label{sec: theory}

\subsection{Assumptions}

\begin{assumption}[Geometry of the task space]
    We assume the source task space $\calW_\text{source}$ is a unit ball $\fB^{d_W^\text{source}}(1)$ that span over the first $d_W^\text{source} \geq \frac{1}{2}d_W $ without loss of generality, while the target task space  $\calW_\text{target} \subset \fR^{d_W}$ can be any arbitrary $\fB^{d_W^\text{target}}(1)$.
\end{assumption}
Under this assumption, we let $B_W^\text{source}$ denote the first $d_W^\text{source}$ columns of $B_W$, which stands for the source-related part of $B_W$. And $B_W^\text{target}$

Then we assume the bilinear model where $\phi_X = B_X \in \fB^{d_X \times k}$ and $\psi_X, \psi_W = I$. Therefore, $d_{\psi_X} = d_X, d_{\psi_W} = d_W$. Moreover the model satisfies the following assumptions 
\begin{assumption}[Benign $B_X, B_W$]
    $B_X$ is an orthonormal matrix. Each column of $B_W$ has magnitude $\Theta(1)$ and $\sigma_{\min}(B_W^\text{source}) > 1$. Suppose we know $\bar{\kappa} \geq \kappa(B_W^\text{source}), \sigma_{\max}(B_W^\text{target})$ and $\underline{\sigma} \leq \sigma_\text{min}(B_W^\text{source}), \sigma_\text{min}(B_W^\text{target})$. Trivially, $\bar{\kappa} = \sqrt{d_W}, \underline{\sigma} = 1$. 
\end{assumption}

Finally, the following assumption is required since we are using a training algorithm in \cite{thekumparampil2021sample} and might be able to relax to sub-gaussian by using other suboptimal oracles.
\begin{assumption}[Isotropic Gaussian Input]
\label{ass: isotropic input}
    For each task $w$, its input $i$ satisfies $x_{i,w} \sim \calN(0, I_d)$. 
\end{assumption}

%% file: theory-algo.tex
\subsection{Algorithm}
Here we provide the target-aware theory and postpone the target-agnostic in the Appendix.~\ref{sec: target-agnostic (appendix)} since its analysis is covered by the target-aware setting.

\begin{algorithm}[ht]
\caption{Target-aware algorithm for benign source space}
\label{algo: main}
\begin{algorithmic}[1]
\STATE \textbf{Inputs:} Target probability $\delta$, $\bar{\kappa}, \underline{\sigma}$. Some constant $\beta_1, \beta_2, \beta_3$. Others same as Algo.~\ref{algo: main-general}.
\STATE Set $q_0$ as
$
    q_0(e_t) = \frac{1}{d_W}, \forall t \in d_W, \text{and} q_0(w) = 0 \text{ otherwise}
$
\STATE Set $n_0 = \beta_1 \Bar{\kappa}^2 \left(k^3d_X\Bar{\kappa}^2 + d_W^{\frac{3}{2}}\underline{\sigma}^{-2}\sqrt{k + \log(1/\delta)}\right)$. Collect $n_0 q_0(w)$ data for each task denoted as $\{Z_w\}_{w | q(w) \neq 0}$
\STATE Update $\hat{B}_X \leftarrow \calO_\text{offline 1}^X(\{Z_w\}_{w | q_0(w) \neq 0 })$ and $\hat{B}_W^\text{source} \leftarrow \calO_\text{offline}^W(\{Z_w\}_{w | q_0(w) \neq 0 }, \hat{B}_X)$ 
\STATE Compute $q_1$ as 
$q_1(v_i) = \frac{1}{k}, \forall i \in k$, and $q_0(w) = 0$ otherwise.
\text{Here} $v_i$ is the $i$-th vector of $V$,  where $U,D,V \leftarrow \text{SVD}(\hat{B}_W^\text{source})$ 
\label{line: compute q_1}
\FOR{$j = 1,2,3, \ldots$}
    \STATE Set $n_1^j = \beta_2  \epsilon_j^{-\frac{4}{3}}k^{\frac{5}{3}}d_W^{\frac{2}{3}}d_X^\frac{1}{3}
    \left( k^{\frac{2}{3}}d_W^{\frac{1}{3}}\underline{\sigma}^{-\frac{4}{3}} + \Bar{\kappa}^2\underline{\sigma}^{-\frac{2}{3}}\right)$. Collect $n_1^j q_1(w)$ data for each task denoted as $\{Z_w\}_{w | q_1(w) \neq 0}$.
    \STATE Update $\hat{B}_X \leftarrow \calO_\text{offline 2}^X(\{Z_w\}_{w | q_1(w) \neq 0})$,  $\hat{B}_W^\text{source} \leftarrow \calO_\text{offline}^W(\{Z_w\}_{w | q_1(w) \neq 0}, \hat{B}_X)$ \\
    and  $\hat{B}_W^\text{target} \leftarrow \calO_\text{offline}^W(\dot{Z}_\text{target}, \hat{B}_X)$
    \STATE Find a set of target-aware tasks parameterized by $\Tilde{W}_j$ with each column $i$ as
    \begin{align*}
        \Tilde{W}_j(i) = \text{Proj}_{\calW_\text{source}} w_i^\prime =
        \frac{w_i^\prime}{\|w_i^\prime\|_2} 
        % \quad \text{when } \calW_\text{source} = \fB^{d_W}(1)
    \end{align*}
    \vspace{-10px}
    \begin{align*}
        \text{where } 
        w_i^\prime = \argmin_{w} \|w\|_2 \quad
        \text{s.t.} \quad & \hat{B}_{W,j}^\text{source} w =  u_i\sqrt{\lambda_i} \quad \forall \Lambda_i \geq 8(kd_W)^\frac{3}{2} \sqrt{\frac{d_X}{n_1}}  \\
        & \text{where } U, \Lambda \leftarrow  \text{Eig}\left(\E_{w_0 \sim \nu_\text{target}}\left[\hat{B}_{W,j}^\text{target}  w_0 (\hat{B}_{W,j}^\text{target} w_0)^\top \right] \right)
    \end{align*}
    \label{line: eigdecompo}
    \STATE Compute $q_2^j$ as 
    $ q_2^j(w) = \frac{1}{\text{\# col}(\tilde{W}_j)}, \forall w \in \text{col}(\Tilde{W}_j)$ and $ q_2^j(w) = 0$ otherwise
    
    % \STATE Assign $n_2^j(w) \approx \frac{q_2^j(w)}{\text{\# col}(\tilde{W}_j)\epsilon^2}* {\color{red}\tilde{\beta}} \frac{\trace\left( \hat{B}_{W,j} W' \right)}{\trace\left( \hat{B}_{W,j} \text{Proj}_{\calW_\text{source}} W' \right)} $ sampling budget to each task parameterized by each column of $\Tilde{W}_j$ 
    \STATE Assign $n_2^j$ total sampling budget as
    $\text{\# col}(\tilde{W}_j) \beta_3 \max_i \|W_j^\prime(i)\|_2^2 \epsilon_j^{-2}$
    % \begin{align*} 
    %     % \text{\# col}(\tilde{W}_j) \max_i \frac{\tilde{\beta}\Lambda_i}{\|\hat{B}_{W,j} \tilde{W}_j(i)\|_2^2} \epsilon_j^{-2} 
    %     % = 
    %     \text{\# col}(\tilde{W}_j) \beta_3 \max_i \|W_j^\prime(i)\|_2^2 \epsilon_j^{-2}  
    %     % \quad \text{when } \calW_\text{source} = \fB^{d_W}(1)
    % \end{align*}
    \STATE Collect $n_2^j(w) = n_2^j q_2^j(w)$ data for each task denoted as $\{Z_w\}_{w | q_2(w) \neq 0}$. 
    \STATE Update the model, note that both data collected from stage 2 and stage 3 are used.
    \begin{align*}
        &\tilde{B}_X \leftarrow \calO_\text{offline 3}^X(\{Z_w\}_{w | q_1(w) \neq 0 \text{ and } q_2(w) \neq 0}) 
        % \\
        % & \tilde{B}_W^\text{target} \leftarrow \calO_\text{offline}^W(Z_\text{target})
    \end{align*}
\ENDFOR
\STATE \textbf{Return }  $\tilde{B}_X$
\end{algorithmic}
\end{algorithm}

This target-aware algorithm~\ref{algo: main} follows the 3-stage which corresponds to sampling distribution $q_0, q_1, q_2$ with explicit solutions. Notice that calculating $q_1$ once is enough for theoretical guarantees. 

We use existing passive multi-task training algorithms as oracles for $\calO_\text{offline 1}^X, \calO_\text{offline 2}^X$ and use the simple ERM methods for $\calO_\text{offline}^W$ based on the learned $\hat{B}$. For the coarse exploration and fine target-agnostic exploration stage, the main purpose is to have a universal good estimation in all directions of $B_X$. ( i.e., upper bound the $\sin(\hat{B}_X, B_X)$) Therefore we choose the alternating minimization (MLLAM) proposed in \cite{thekumparampil2021sample}.
% 
% , as restated below. Note $t = d_W$ in stage 1 and $t = k$ in stage 2.
% 
% \yifang{Maybe move it to supplementary to save the space.}
% \begin{theorem}[Informal statement of Theorem 1 \cite{thekumparampil2021sample}]
% For total $t$ linear regression tasks, each with $n_0$ samples satisfying Assumptions~\ref{ass: isotropic input},
% \begin{align*}
%     & n_0 \geq \widetilde{\Omega}\left(\left(1+k\left(1 / \sigma_{\min}(B_W)\right)^2\right) k \log t+k^2\right)
% \end{align*}
% Then MLLAM, with a proper initialization and running time, w.h.p, outputs $B_X$ that :
% \begin{align*}
%     \left\|\left(\mathbf{I}-B_X^\text{source}\left(B_X^\text{source}\right)^{\top}\right)\hat{B}_X^\text{source} \right\|_F
%     \leq \widetilde{O}\left(\left(\frac{1}{\sigma_{\min}(B_W)}\right) k \sqrt{\frac{\mu d_X}{N_\text{total}}}\right)
% \end{align*}
% \end{theorem}
% 
On the other hand, for the fine target-aware exploration, we mainly care about final transfer learning performance on learned representation. Therefore, we use a non-convex ERM from \cite{du2021fewshot}. We defer the details and its theoretical guarantees for $\calO_\text{offline}$ into Appendix~\ref{sec: training oracle (appendix)}.

Note the major disadvantage from \cite{du2021fewshot} comes from its sample complexity scaling with a number of training source tasks, which will not be a problem here since in $\calO_\text{offline, 3}^X$ since only $k + k^* \ll d_W$ number of tasks are used. The major benefit of using non-convex ERM comes from its generality that it works even for the non-linear setting and is not tied with a specific algorithm. That is to say, as long as there exists other theoretical or heuristic oracles  $\calO_\text{offline, 1}^X, \calO_\text{offline, 2}^X$ giving a similar guarantee, stage 3 always works.

\subsection{Results}
\begin{theorem}[Informal]
\label{them: main}
    By running Algo.~\ref{algo: main}, in order to let $ER(\hat{\phi}_X, \nu_\text{target}) \leq \varepsilon^2$ with probability $1-\delta$, the number of source samples $n_\text{source}$ is at most
    \begin{align*}
        \otil \left(\left( kd_X + \log(1/\delta) \right) (k^*)^2 \min\{k^*, \kappa^2(B_W)\} \max_i \|W_i^*\|_2^2 \varepsilon^{-2}
        + \text{low-order} \right)
    \end{align*}
    Here $k^* = \text{rank}(\E_{w_0 \sim \nu_\text{target}}B_W w_0 w_0^\top B_W^\top)$ represents the effective dimension of target and 
    % $W_i^*$ is the solution of 
    \begin{align*}
        W^*_i = \argmin_{w \in \calW_\text{source}} \|w\|_2 
        \quad \text{s.t} \quad B_W^\text{source} w = u_i \sqrt{\lambda_i} 
        \text{ where } U, \Lambda \leftarrow \text{Eig}(\E_{w_0 \sim \nu_\text{target}}B_W w_0 w_0^\top B_W^\top).
    \end{align*}
    As long as the number of target samples satisfies 
    \begin{align*}
       & n_\text{target} \geq \tilde{\Omega}((k + \log(1/\delta)) \varepsilon^{-2}),
       & \dot{n}_\text{target} \gtrapprox 
       % \tilde{\Omega}(k^2 (\bar{\kappa})^{\frac{4}{3}} \log(1/\delta)\varepsilon^{-\frac{4}{3}}
        \widetilde{\Omega}\left(\varepsilon^{-\frac{4}{3}} (k^*)^{\frac{2}{3}}
                \sqrt{k}\left( d_W^{\frac{1}{2}}\underline{\sigma}^{-\frac{4}{3}} 
                + k^{-\frac{2}{3}}d_W^\frac{1}{6}\Bar{\kappa}^2\underline{\sigma}^{-\frac{1}{3}}\right)
                \right)
    \end{align*}
\end{theorem}

\paragraph{Comparison with passive learning.} By choosing $\{e_i\}_{i \in [d_W^\text{source}]}$ as a fixed source set, we reduce the problem to a discrete setting and compare it with the passive learning. In \cite{du2021fewshot}, the authors get $N_\text{total}$ as most 
$
    \frac{k d_X d_W \|\E_{w_0 \sim \nu_\text{target}}B_W w_0 w_0^\top B_W^\top \|}{\sigma_{\min}^2(B_W^\text{source}) } \varepsilon^{-2}.
$
We first consider the cases in their paper that the target task is uniformly spread $\|\E_{w_0 \sim \nu_\text{target}}B_W w_0 w_0^\top B_W^\top \| = \frac{1}{k}$.
\begin{itemize}[leftmargin=*]
    \item When the task representation is well-conditioned $\sigma_{\min}^2(B_W^\text{source}) = \frac{d_W}{k}$. We have a passive one as $ \otil(k d_X \varepsilon^{-2})$ while the active one $\otil(kd_X \frac{k^2}{d_W} \varepsilon^{-2})$ (See Lemma~\ref{lem: uniform case} for details), which suggests as long as $d_W \gg k^2$, our active learning algorithm gain advantage even in a relatively uniform spread data and representation conditions. 
    \item Otherwise, we consider the extreme case that  $\sigma_{\min}^2(B_W^\text{source}) = 1$. We have passive one  $\otil(d_X d_W\varepsilon^{-2})$ while the active one $\otil(k^3 d_X \varepsilon^{-2})$. Notice here we require $d_W \gg k^3$.
\end{itemize}
Both of them indicate the necessity of considering the continuous case with large $d_W$ even if everything is uniformly spread. On the other hand, whether we can achieve the same result as the passive one when $d_W \leq k^3$ remains to be explored in the future.  

We then consider the single target $w_0$ case.
\begin{itemize}[leftmargin=*]
    \item With well-conditioned $B_W$, the passive one now has sample complexity $\order(k^2 d_X \varepsilon^{-2})$ while the active gives a strictly improvement $\order(\frac{k^2 d_X}{d_W}  \varepsilon^{-2})$.
    \item With ill-conditioned $B_W$ where $\sigma_{\min}(B_W) = 1$ and $\max_i \|W_i^*\| = 1$, that is, only a particular direction in source space contributes to the target. The Passive one now has sample complexity $\order(k d_X d_W \varepsilon^{-2})$ while our active one only has $kd_X \varepsilon^{-2}$, which demonstrates the benefits of our algorithm in unevenly distributed source space.
\end{itemize}

\paragraph{Comparison with previous active learning.} By using the same discrete reduction and set single target $w_0$, we compare our result with the current state-of-art active representation algorithm in \cite{chen2022active}. They achieves $\otil(kd_X\|\nu\|_1^2 \varepsilon^{-2})$, where
$\nu = \argmin_{\nu} \|\nu\|_1 \text{ s.t } B_W \nu = B_W w_0$. 
On the other hand, our active one gives $\otil(kd_X\|w^*\|_2^2 \varepsilon^{-2})$, where
$w^* = \argmin_{\nu} \|\nu\|_2 \text{ s.t } B_W \nu = B_W w_0$, which is strictly better than the discrete one. This again indicates the separation between continuous and discrete cases where in fixed discrete sets, the $L_1$ norm regularization is strictly better than $L_2$. \footnote{In fact, \cite{chen2022active} get even worse bound, but we are aware that there exists some concurrent work to achieve this $\otil(kd_X\|\nu\|_1^2 \varepsilon^{-2})$ using $L_1$ norm regularization and a tighter analysis. No matter which bound we compared to, the conclusion will not be affected.} 

% because as discussed in \yf{ref}, for fixed discrete sets, the $L_1$ norm regularization is strictly better than $L_2$. \yf{Can I put this icml paper temporary on my personal website and cite it ?.. Otherwise I will change the comparison to the previous one.}

\paragraph{Save task number.} When ignoring the short-term  initial warm-up stage, we only require maintaining $\otil(k + \log(N_\text{total}k^*))$ number of source tasks, where the first term comes from $q_1$ in the target-agnostic stage and the second term comes from $q_2$ in the target-aware stage.

% \subsubsection{Overview}
%  We first show that the excess risk highly depends on 
%  \begin{align*}
%       \trace\left(\left( B_W \tilde{W}_\calS \tilde{W}_\calS^\top B_W ^\top \right)^{-1} B_W \tilde{W}_\text{target} \tilde{W}_\text{target}^\top B_W^\top \right)
%  \end{align*}
%  where $\tilde{W}_\calS$ and $\calW_\text{target}$ are training source and target-related parameters which we will specify later.

% Denoting $\left( B_W \tilde{W}_\calS \tilde{W}_\calS^\top B_W ^\top \right)^{-1}$ as $\square$, we can further decompose it into
%  \begin{align*}
%      \underbrace{\E_{w_0 \sim \nu_0}(\left[(B_W w_0)^\top \square B_W w_0\right] -  \trace\left( \Tilde{\beta}(B_W W')^\top \square B_W W' \right)}_\text{target-agnostic exploration error} +   \Tilde{\beta} \underbrace{\trace\left( (B_W W')^\top \square B_W W'\right)}_\text{taregt-aware exploration error} 
%  \end{align*}

%  To the upper bound the target-agnostic exploration error, the key is to show the $q_1$ distribution well approximate the $k$-dim embed space. \yifang{Add more?}
 
%  Then, to upper bound the target-aware exploration error, we will firstly prove the increasingly accurate estimation on $B_X,B_W$ and then show a proper regularization that avoids over-exploration. \yifang{TODO}

%% file: experiments.tex
\section{Experiment}
\label{sec: experiment}

In this section, we provide experimental results under different instantiations of the Algorithm~\ref{algo: main-general}, and all of them show the effectiveness of our strategy both in target-aware and target-agnostic settings.
\subsection{Settings}

\paragraph{Datasets and problem definition.} 
Our results cover the different combinations of $\psi_X, \phi_X, \psi_W$ as shown in Table~\ref{table: experiements}. 
% Note the real-word drone data is discrete and others are continuous. 
Here we provide a brief introduction for the three datasets and postpone the details into Appendix~\ref{sec: experiments (appendix)}. \footnote{ Github Link: \href{https://github.com/cloudwaysX/ALMultiTask_Robotics}{https://github.com/cloudwaysX/ALMultiTask\_Robotics}}

\vspace{-5px}
\begin{table}[ht]
\begin{center}    
\begin{tabular}{|l|l|l|}
\hline
                                         & identity $\psi_W$ & nonlinear $\psi_W$ \\ \hline
identity $\psi_X$ and linear $\phi_X$    & synthetic, drone  & NA                 \\ \hline
nonlinear $\psi_X$ and linear $\phi_X$   & synthetic        & pendulum simulator \\ \hline
identity $\psi_X$ and nonlinear $\phi_X$ & synthetic, drone  & NA                 \\ \hline
\end{tabular}
\end{center}
\caption{Summary of different instantiations}
\vspace{-1.5em}
\label{table: experiements}
\end{table}

\begin{itemize}[leftmargin=*]
    \item \textbf{Synthetic data.} We generate data that strictly adhere to our data-generating assumptions and use the same architecture for learning and predicting. When $\phi_X$ is nonlinear, we use a neural network  $\phi_X$ to generate data and use a slightly larger neural net for learning. The goal for synthetic data is to better illustrate our algorithm as well as serve as the first step to extend our algorithm on various existing datasets.
    
    \item \textbf{Pendulum simulator.} To demonstrate our algorithm in the continuous space. we adopt the multi-environment pendulum model in~\cite{shi2021meta} and the goal is to learn a $w$-dependent residual dynamics model $f(x,w)\in\mathbb{R}$ where $x$ is the pendulum state and $w\in\mathbb{R}^5$ including external wind, gravity and damping coefficients. $f(x,w)$ is highly nonlinear with respective to $x$ and $w$.
    % \guanya{Let's move all these details to sup. Just say ``we adopt the multi-environment pendulum model in~\cite{shi2021meta} and the goal is to learn a $w$-dependent residual dynamics model $f(x,w)$ where $x$ is the robot state and $w\in\mathbb{R}^5$ including external wind, gravity and damping coefficients. $f(x,w)$ is highly nonlinear with respective to $x$ and $w$ ...''}
    % We collect data from the continuous-time model $m l^2 \ddot{\theta}-m l \hat{g} \sin \theta=u+f(\theta, \dot{\theta}, w)$ by injecting random gaussian $u$ to simulate the i.i.d data collection. Here we set input $x = [\theta, \dot{\theta}]$ and $w$ as a five-dimension parameter, and our goal is to learn $f(\theta, \dot{\theta}, w)$. Of course, $f$ is nonlinear regarding $x,w$,
    Therefore we use known non-linear feature operators $\psi_X, \psi_W$. In other words, this setting can be regarded as a misspecified linear model. It is also worth noting that due to the non-invertibility of $\psi_W$, the explicit selection of a source via a closed form is challenging. Instead, we resort to an adaptive sampling-based method discussed in Section~\ref{para: optimal design oracle}. Specifically, we uniformly sample $w$ from the source space, select the best $w'$, and then uniformly sample around this $w'$ at a finer grain. Our findings indicate that about 5 iterations are sufficient to approximate the most relevant source.
    
    \item \textbf{Real-world drone flight dataset~\cite{o2022neural}.} 
    % \yifang{Yingbing: add something}
    % \guanya{Make it short and move details to sup. Just need to say dimension, what is $w$, etc.}
    The Neural-Fly dataset~\cite{o2022neural} includes real flight trajectories using two different drones in various wind conditions. The objective is to learn the residual aerodynamics model $f(x,w)\in\mathbb{R}^3$ where $x\in\mathbb{R}^{11}$ is the drone state (including velocity, attitude, and motor speed) and $w$ is the environment condition (including drone types and wind conditions). We collect 6 different $w$ and treat each dimension of $f(x,w)$ as a separate task. Therefore $w$ is reformulated as a one-hot encoded vector in $\mathbb{R}^{18}$. 

    % We consider a 6-DoF quadrotor model with unknown wind-dependent aerodynamic force on each $[x,y,z]$ direction $f(x,w) \in \mathbb{R}$, where $x \in \mathbb{R}^{11}$ is the shared variables (including velocity, attitude, motor speed command). We consider 18 \textit{discrete} environments from the combination of three variables - drone types, wind speeds, and directions, denoted as "droneType\_windSpeeds\_direction``. Therefore, $w \in  \mathbb{R}^{18}$ is one-hot encoded.
\end{itemize}

For each dataset/problem, we can choose different targets. For simplicity, in the following subsection, we present results for one target task for each problem with 10 random seeds regarding random data generation and training, and put more results in Appendix~\ref{sec: experiments (appendix)}.  In all the experiments, we use a gradient-descent joint training oracle, which is a standard approach in representation learning. 

% For the synthetic datasets, each unique underlying task data generation model (including source and target) is regarded as one case. \guanya{one case?} For the pendulum simulator and the drone dataset, since the model is fixed, so each different target task is regarded as one case. \guanya{the previous sentence is confusing}  In all the experiments, we use a gradient-descent joint training oracle, which is a standard approach in representation learning. For each case, we repeat about 10 random seeds regarding random data generation and training. Here we present one case for each and put more results in the appendix. \guanya{cite appendix. What does one case for each mean?} \yf{It's like for each setting we can choose different target or different ground truth model. And here I just present one. Not sure what is the right way to say this.} \guanya{Oh I see. Probably just saying target is better?}

\subsection{Results}

Those results encapsulate the effectiveness of active learning in terms of budget utilization and test loss reduction. In the drone dataset, we further demonstrate its ability in identifying relevant source tasks (see Figure~\ref{fig:flight}). We note that in two robotics problems (pendulum simulation and real-world drone dataset), the active learning objective is to learn \emph{a better dynamics model}. However, in the pendulum simulation, we deploy a model-based nonlinear controller which translates better dynamics modeling to enhanced control performance (see Figure~\ref{fig: pendulum} and Appendix~\ref{sec: pendulum (appendix)}).

\vspace{-7px}
\begin{table}[ht]
\begin{center}
\begin{tabular}{|l|l|l|l|}
\hline
                                          & Target-aware AL & Target-agnostic AL  \\ \hline
identity $\psi_X$ and linear $\phi_X$     & 38.7\%           & 51.6\%               \\ \hline
nonlinear $\psi_X$ and linear $\phi_X$    & 38.7\%           & 45.2\%                \\ \hline
identity $\psi_X$ and non-linear $\phi_X$ & 32.0\%           & 68.0\%                \\ \hline
\end{tabular}
\end{center}
\label{table: synthetic}
\caption{Results on synthetic data. Using the test loss of the final output model from passive learning as a baseline, we show the ratio between the budget required by target-aware/target-agnostic active learning to achieve a similar loss and the budget required by passive learning.
}
\end{table}
\vspace{-11px}
\begin{figure}[ht]
    \centering
    \includegraphics[scale = 0.38]{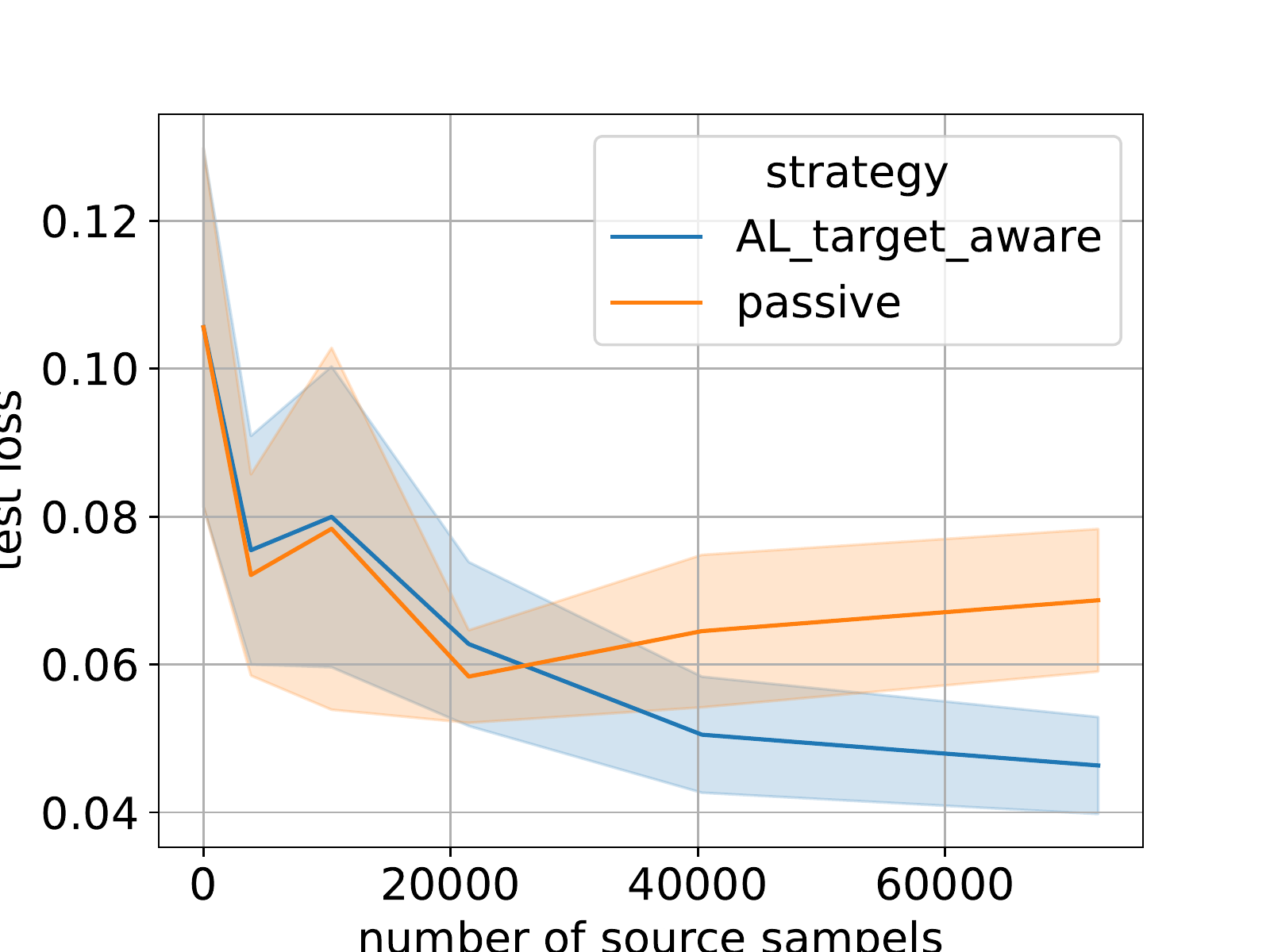}
    \includegraphics[scale = 0.38]{Plots/embed_dim8_actual_target_0,_0,_1,_0.5,_0,_0_control.pdf}
    \caption{Results on pendulum simulator for a specific target.
    \textbf{Left: } The test loss of the estimated model $\hat{f}$. The passive strategy suffers from negative transfer while the active strategy steadily decreases. 
    \textbf{Right: } The control error using final output $\hat{f}$. Here we use a model-based nonlinear policy $\pi(x,\hat{f})$. The model learned from active strategy leads to better control performance.
    }
    \label{fig: pendulum}
\end{figure}
\vspace{-11px}
\begin{figure}[ht]
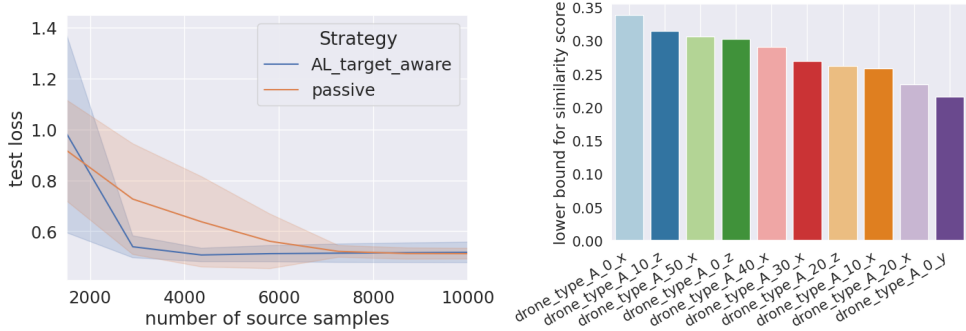

    \centering
    \includegraphics[scale = 0.36]{Plots/flight_taskdim18.pdf}
    \includegraphics[scale = 0.32]{Plots/similarities.pdf}
    \caption{Results on the real drone dataset~\cite{o2022neural} with target \texttt{drone\_type\_A\_30\_z}. Source data includes two drone types A and B, six wind speeds from 0 to 50, and three directions x-y-z. We present results for linear $\phi_X$ here and postpone the non-linear $\phi_X$ case in Appendix~\ref{sec: drone (appendix}. \textbf{Left:} The test loss of the estimated bilinear model $\hat f$.  The passive strategy converges slower than the active strategy.
    \textbf{Right:} Top 10 the most similar source tasks. Given the target environment, the algorithm successfully finds the other \texttt{drone\_type\_A} environments as relevant sources. See more explanations in Appendix~\ref{sec: drone (appendix}.
    }
    \label{fig:flight}
\end{figure}

%% file: appendix-trainingOracle.tex
\subsection{Offline training oracles used in Algorithm}
\label{sec: training oracle (appendix)}

\subsubsection{Choice of $\calO_\text{offline 1}^X$}
\label{sec: training oracle 1 (appendix)}

% To better illustrate this oracle, we first give the following definition.
% \begin{assumption}[Modified from Assumption 2 in \cite{thekumparampil2021sample}]
%     For any $t$ tasks with parameter matrix $W'=[w_1, w_2, \ldots, w_t] \in \fR^{d_W \times t}$, let $\lambda_1^*$ and $\lambda_k^*$ denote the largest and smallest eigenvalues of the task diversity matrix $(k / t) B_W' W (B_W W')^\top \in \mathbb{R}^{k \times k}$ respectively. We assume that $B_W W'$ is $\mu$-incoherent, i.e.,
%     \begin{align*}
%         \max _{i \in[t]}\left\| B_W w_i\right\|^2 \leq \mu \lambda_k^*
%     \end{align*}
% \end{assumption}

% \yifang{$|\dot{V}|$ is a weird notation, change later}
% \yifang{ Check in the algorithm what we exactly need to know}
% Define $\dot{V}$ be the column collection of $w$ used in $q_0$ or $q_1$. 

To better illustrate this oracle $\calO_\text{offline}^X$, we first give the following definition.
\begin{definition}[Modified from Assumption 2 in \cite{thekumparampil2021sample}]
    For any $t$ tasks with parameter matrix $\dot{V}=[\dot{v}_1, \dot{v}_2, \ldots, \dot{v}_t] \in \fR^{d_W \times t}$. Let $\lambda_1^*$ and $\lambda_k^*$ denote the largest and smallest eigenvalues of the task diversity matrix $(k / t)B_W^\text{source} \dot{V}  \dot{V}^\top (B_W^\text{source})^\top \in \mathbb{R}^{k \times k}$ respectively. Then we say $\dot{V}$ is $\mu$-incoherent, i.e.,
    \begin{align*}
        \max _{i \in[t]}\left\|B_W^\text{source} \dot{v}_i\right\|^2 \leq \mu \lambda_k^*
    \end{align*}
\end{definition}
Notice that here $\dot{V}$ is a general representation of collected source tasks used for training in the different stages. Therefore, the $\lambda_k^*, \mu$ is also defined differently corresponding to each stage. Specially, we have
\begin{itemize}
    \item \textbf{Stage 1( data collected by $q_0$)}: 
    \begin{itemize}
        \item $t = d_W, \dot{V} = I_{d_W}$
        \item $\lambda_k^* = \frac{k}{d_W}\sigma_k^2(B_W^\text{source})$
        \item $\mu \geq \frac{1}{\lambda_k^*}$
    \end{itemize}
    \item \textbf{Stage 2( data collected by $q_1$)}:
    \begin{itemize}
        \item $t=k, \dot{V} = V$ where $\_,\_,V \leftarrow \text{SVD}(\hat{B}_W^\text{source})$ as defined in line~\ref{line: compute q_1}
        \item $\lambda_k^* = \sigma_k^2(B_W^\text{source})$
        \item $\mu \geq \frac{\sigma_\text{max}^2(B_W^\text{source})}{\lambda_k^*}$
    \end{itemize}
\end{itemize}
Note that  $\lambda_k^* = \sigma_k^2(B_W^\text{source})$ in the stage 2 comes from $B_W^\text{source} \dot{V}  \dot{V}^\top (B_W^\text{source})^\top = \Theta(B_W^\text{source}(B_W^\text{source})^\top)$ which will be proved later. Therefore, applying these results to

Now we restate the generalization guarantees from a fixed design (passive learning)

\begin{theorem}[Restate Theorem 1 in \cite{thekumparampil2021sample}]
     Let there be $t$ linear regression tasks, each with $m$ samples, and 
     \begin{align*}
         m \geq \widetilde{\Omega}\left(\left(1+k\left(\sigma / \sqrt{\lambda_k^*}\right)^2\right) k \log t+k^2\right), \quad \text { and } m t \geq \widetilde{\Omega}\left(\left(1+\left(\sigma / \sqrt{\lambda_k^*}\right)^2\right)\left(\lambda_1^* / \lambda_k^*\right) \mu d_X k^2\right)
     \end{align*}
     Then MLLAM, initialized at $\hat{B}_X = U_{\mathrm{init}}$ s.t. $\left\|\left(\mathbf{I}-B_X\left(B_X\right)^{\top}\right) U_{\mathrm{init}}\right\|_F \leq \min \left(3 / 4, O\left(\sqrt{\lambda_k^* / \lambda_1^*}\right)\right)$ and run for $K=\left\lceil\log _2\left(\lambda_k^* \lambda_k^* m t / \lambda_1^* \sigma^2 \mu d_X k^2\right)\right\rceil$ iterations, outputs $\hat{B}_X$ so that the following holds $\left(w . p\right.$. $\left.\geq 1-K /(d_X k)^{10}\right)$
     \begin{align*}
     \sin(\hat{B}_X, B_X)
        \leq 
         \left\|\left(\mathbf{I}-B_X\left(B_X\right)^{\top}\right) \hat{B}_X\right\|_F \leq \widetilde{O}\left(\left(\frac{\sigma}{\sqrt{\lambda_k^*}}\right) k \sqrt{\frac{\mu d_X}{m t}}\right)
     \end{align*}
\end{theorem}

Specifically, suppose we satisfy all the requirements in the theorem and run the proper amount of times, then we can guarantee $\hat{B}_X$ after each stage $j$ with w.h.p $\geq 1- 2K /(d_X k)^{10}$
\begin{itemize}
   \item \textbf{Stage 1( data collected by $q_0$)}: 
   $
       \sin(\hat{B}_X, B_X) \leq \otil\left(\sigma k \sqrt{\frac{d_X}{n_0}}\right)
   $
   \item  \textbf{Stage 2( data collected by $q_1$)}: 
   $
       \sin(\hat{B}_X, B_X) \leq \otil\left(\sigma k \sqrt{\frac{ d_X \sigma_\text{max}^2(B_W^\text{source})}{n_0}}\right)
   $
\end{itemize}
Let Event $\calE_\text{offline 1}$ denote the above guarantees hold for all epochs.

\subsubsection{Choice of $\calO_\text{offline 2}^X$}

We use the ERM from \cite{du2021fewshot}. For readers' convenience, we restate the formal definition of oracle 
% \begin{align*}
%     &\tilde{B}_X \leftarrow \calO_\text{offline 3}^X(\{Z_w\}_{w | q_1(w) \neq 0 \text{ and } \tilde{q}(w) \neq 0}) 
% \end{align*}
below 
\begin{align*}
    \hat{B}_X = \argmin_{B} \sum_{w | q_1(w) \neq 0 \text{ and } q_2(w) \neq 0} \argmin_w \sum_{(x,y) \in Z_w} \|x^\top w -y\|_2
\end{align*}
By using this ERM with the follow-up finetune on $Z_\text{target}$, we get the following claims. Note that this claim comes from some part of Proof of Theorem 4.1 in the previous paper and has also been used in Claim 3 in \cite{chen2022active}.
\begin{claim}
\label{claim: ERM oracle}
    By running the ERM-based algorithm, we get the following upper bounds,
    \begin{align*}
    \textsc{ER}(\tilde{B}_X, \nu_\text{target})
    &\leq \E_{w_0 \sim \nu_\text{target}} \left[ \frac{\|P_{X_\text{target}\hat{B}_X}^{\bot}X_\text{target}B_X B_W w_0\|^2}{n_\text{target}} +\sigma^2\frac{k + \log(1/\delta)}{n_\text{target}} \right]
    \end{align*}
\end{claim}

We need to admit that, from a theoretical perspective, we choose this oracle since we can directly use their conclusions. But other oracles like $\calO_\text{offline 2}^X$ might also work.

\subsubsection{Choice of $\calO_\text{offline}^W$}
This is the ERM oracle based on learned $\hat{B}_X$. 
Specially, we have $\hat{B}_W^\text{source/target} \leftarrow \calO_\text{offline}^W(\{Z_w\}_{w | q(w) \neq 0}, \hat{B}_X)$ defined as 
\begin{align*}
    \hat{B}_W^\text{source/target} = \sum_{w | q(w) \neq 0} \hat{w}_w w^\top,
    \text{ where }
    \hat{w}_w =\argmin_{\hat{w} \in \fR^k} \sum_{(x,y) \sim Z_{w}} \|x^\top\hat{B}_X^\top \hat{w}  - y\|_2, 
\end{align*}

%% file: appendix-generalization1.tex
\subsection{Excess risk analysis}

\begin{theorem}[Excess risk guarantees]
\label{them: excess risk (appendix)}
    By running the Algo.~\ref{algo: main}, after epoch $j$, as long as $\calE_\text{offline 1}$ holds, we have w.h.p $1-\delta$,
    \begin{align*}
        \textsc{ER}(\tilde{B}_X, \nu_\text{target})
        \leq \otil \left(\sigma^2 kd_X k^* \epsilon_j^2 \right)
    \end{align*}
    as long as 
    \begin{align*}
        & \dot{n}_\text{target}
        \geq \epsilon_j^{-\frac{4}{3}}d_X^{-\frac{2}{3}}
        \left(k^{-\frac{2}{3}}d_W^{\frac{1}{2}}\underline{\sigma}^{-\frac{4}{3}}
        + k^{-\frac{4}{3}}d_W^\frac{1}{6}\Bar{\kappa}^2\underline{\sigma}^{-\frac{1}{3}}\right)
        \sqrt{k + \log(d_W/\delta)}\\
        & n_\text{target} \geq \epsilon_j^{-2} d_X^{-1}(k^*)^{-1} \frac{k}{k+ \log(d_W/\delta)}
    \end{align*}
\end{theorem}
\begin{proof}
    Here we provide the proof sketches, which will be specified in the following sections.
    
    In Section~\ref{sec: reduce to optimal design (appendix) }, we first reduce $\textsc{ER}(\tilde{B}_X, \nu_\text{target})$ to an optimal design problem by showing that, with a proper number of $n_\text{target}$,
    \begin{align*}
        \textsc{ER}(\tilde{B}_X, \nu_\text{target})
        \lessapprox (kd_X + \log(1/\delta)) \trace\left(\left( (B_W^\text{source}) \left(\sum_{w \in \calS} n_w w w^\top \right) (B_W^\text{source})^\top \right)^{-1} B_W \left(\E_{\nu_\text{target} } w w^T \right) B_W^\top \right)
    \end{align*}
    It is easy to see that, as long as $B_W$ is known. The problem is reduced to an optimal design problem with fixed optimization target. 
    
    So the main challenge here is to iteratively estimate $B_X, B_W$ and design the budget allocation to different sources. Therefore, in Section~\ref{sec: bound decomposition}, we further decompose the it into 
    \begin{align*}
        & \trace\left(\left( (B_W^\text{source}) \left(\sum_{w \in \calS} n_w w w^\top \right) (B_W^\text{source})^\top \right)^{-1} B_W \left(\E_{\nu_\text{target} } w w^T \right) B_W^\top \right) \\
        & \leq \underbrace{\E_{w_0 \sim \nu_\text{target}}(\left[(B_W w_0)^\top \square B_W w_0\right] -  \trace\left( \beta_3(B_W W')^\top \square B_W W' \right)}_\text{target agnostic exploration error} +   \beta_3 \underbrace{\trace\left( (B_W W')^\top \square B_W W'\right)}_\text{target-aware exploration error} 
    \end{align*}
    where $\square = \left( B_W \left(\sum_{w \in \calS} n_w w w^\top \right) B_W ^\top \right)^{-1}$. Here the \textbf{target-aware exploration error} captures the error from selecting the target-related sources (defined by $q_2$). On the other hand, the \textbf{target agnostic exploration error} captures the error from model estimation and the uniform exploration. 
    
    Now the main challenge here is to upper-bound the model estimation error. Specifically, the estimation comes from Coarse exploration (Stage 1) and Fine target-agnostic exploration (Stage 2). Specifically, in Section~\ref{sec: warm-up(appendix)}, we show that the $k$-dim-subspace represented by $q_1$ is a good course approximation up to \textit{multiplicative} error. Then in Section~\ref{sec: stage2 (appendix)}, we further tight the upper bound using data collected according to up to some \textit{additive} error.
\end{proof}

\subsubsection{Reduce to an optimal design problem}
\label{sec: reduce to optimal design (appendix) }

For any fixed epoch $j$, let $n_w^j$ denotes the samples collected so far for task $w$ and 
$\calS$ denotes the set of tasks used in computing $\tilde{B}_X$.
Therefore, we have $\calS = \{w | q_1(w) \neq 0 \text{ and } \tilde{q}(w) \neq 0\}$ and $ n_w \geq n_2(w) + n_2^j(w)$. For convenience, we omit the superscript $j$ in the rest of the proofs.
 
% Now by using the result in \cite{chen2022active}, we have 
% \begin{align*}
%     \textsc{ER}(\tilde{B}_X, \nu_\text{target})
%     &\leq \E_{w_0 \sim \nu_\text{target}} \left[ \frac{\|P_{X_\text{target}\hat{B}_X}^{\bot}X_\text{target}B_X B_W w_0\|^2}{n_\text{target}} +\sigma^2\frac{k + \log(1/\delta)}{n_\text{target}} \right]
% \end{align*}
From Claim~\ref{claim: ERM oracle}, it is easy to see that our main target is to optimize
$\E_{w_0 \sim \nu_\text{target}}\|P_{X_\text{target}\hat{B}_X}^{\bot}X_\text{target}B_X B_W w_0\|^2$.  Decompose $B_W \left(\sum_{w \in \calS} n_w w w^\top \right) B_W^\top$ as $UDU^\top$ and let $\Sigma_W = U\sqrt{D}U^\top$. As long as $\Sigma_W$ is full rank, which we will prove later in Section~\ref{sec: warm-up(appendix)}, we have with probability $1-\delta$,
\begin{align*}
    &\E_{w_0 \sim \nu_\text{target}}\|P_{X_\text{target}\hat{B}_X}^{\bot}X_\text{target}B_X B_W w_0\|^2 \\
    & = \E_{w_0 \sim \nu_\text{target}} \|P_{X_\text{target}\hat{B}_X}^{\bot}X_\text{target}B_X \Sigma_W^{\frac{1}{2}} \Sigma_W^{- \frac{1}{2}} B_W w_0\|^2 \\
    & \leq \E_{w_0 \sim \nu_\text{target}} \|P_{X_\text{target}\hat{B}_X}^{\bot}X_\text{target}B_X \Sigma_W^{\frac{1}{2}}\|_F^2 \|\Sigma_W^{- \frac{1}{2}} B_W w_0\|^2 \\
    & = \E_{w_0 \sim \nu_\text{target}} \|P_{X_\text{target}\hat{B}_X}^{\bot}X_\text{target}B_X B_W \tilde{W}_\calS\|_F^2 (B_W w_0)^\top \left( B_W \left(\sum_{w \in \calS} n_w w w^\top \right) B_W ^\top \right)^{-1} B_W w_0 \\
    & = \|P_{X_\text{target}\hat{B}_X}^{\bot}X_\text{target}B_X B_W \tilde{W}_\calS\|_F^2 
         \trace\left(\left( B_W \left(\sum_{w \in \calS} n_w w w^\top \right) B_W ^\top \right)^{-1} B_W \left(\E_{\nu_\text{target} \in \Delta(\calW_\text{target})} w w^T \right) B_W^\top \right) \\
    & \lessapprox \sigma^2 n_\text{target} \left(kd_X + \log(1/\delta) \right) \trace\left(\left( B_W \left(\sum_{w \in \calS} n_w w w^\top \right) B_W ^\top \right)^{-1} B_W \left(\E_{\nu_\text{target} } w w^T \right) B_W^\top \right) \\
    & = \sigma^2 n_\text{target} \left(kd_X + \log(1/\delta) \right) \trace\left(\left( (B_W^\text{source}) \left(\sum_{w \in \calS} n_w w w^\top \right) (B_W^\text{source})^\top \right)^{-1} B_W \left(\E_{\nu_\text{target} } w w^T \right) B_W^\top \right)
\end{align*}
Therefore, we aim to minimize the $\trace\left(\left( (B_W^\text{source}) \left(\sum_{w \in \calS} n_w w w^\top \right) (B_W^\text{source})^\top \right)^{-1} B_W \left(\E_{\nu_\text{target} } w w^T \right) B_W^\top \right)$. As we mentioned before, this is a pure optimal design problem if $B_W$ is known in advance.

\subsubsection{Bound decomposition and the excess risk result}
\label{sec: bound decomposition}
Let $\square = \left( B_W \left(\sum_{w \in \calS} n_w w w^\top \right) B_W ^\top \right)^{-1}$, we have
\begin{align*}
    & \E_{w_0 \sim \nu_\text{target}}\left[(B_W w_0)^\top \square B_W w_0\right] \\
    & = \underbrace{\E_{w_0 \sim \nu_\text{target}}(\left[(B_W w_0)^\top \square B_W w_0\right] -  \trace\left( \beta_3(B_W W')^\top \square B_W W' \right)}_\text{target agnostic exploration error} +   \beta_3 \underbrace{\trace\left( (B_W W')^\top \square B_W W'\right)}_\text{target-aware exploration error} 
\end{align*}
\textbf{We first deal with the target-aware exploration error}. 
It is easy to see that 
\begin{align*}
    & \beta_3\trace\left((B_W W' )^\top \square B_W W' \right)\\
    & =  \beta_3\trace\left( \left(B_W \sum_{w} q_1(w) n_1 w w^\top (B_W)^\top 
        + B_W \sum_{w} q_2(w) n_2 w w^\top (B_W)^\top \right)^{-1} B_W W' (B_W W')^\top \right)\\
    & \leq \trace\left( \left(  \max_i  \|\Tilde{W}(i)\|_{2(\infty)}^2 B_W \epsilon^{-2} \tilde{W} \tilde{W}^\top (B_W)^\top \right)^{-1} B_W W' (B_W W')^\top \right) \\
    & \leq \trace\left(  \left(  B_W \epsilon^{-2} W' (W')^\top (B_W)^\top \right)^{-1} B_W W' (B_W W')^\top \right) \\
    & = \epsilon^2 \text{rank}(\hat{B}_W W' (W')^\top \hat{B}_W^\top) \\
    & \leq \epsilon^2 \text{rank}(B_W \E_{\nu_\text{target}}[w_0 w_0^\top] B_W^\top) 
\end{align*}
where the last equality comes from Lemma~\ref{lem: est rank}.

\textbf{We then deal with the target-agnostic exploration term.} 
Let the clipping threshold in Line~\ref{line: eigdecompo} be $\bar{\gamma}_j$. That is, ignoring all $\lambda_i \leq \bar{\gamma}$. 
Now, for $\beta_3 \geq 8$, when event $\calE_\text{offline 1}$, holds, we have w.h.p $1-d_W\delta$
\begin{align*}
    & \E_{w_0 \sim \nu_0}\left[(B_W w_0)^\top \square B_W w_0\right] -  \beta_3\trace\left((B_W W')^\top \square B_W W' \right) \\
    & = \E_{w_0 \sim \nu_0}\trace\left(\square \left(B_W w_0(B_W w_0)^\top - 4 \hat{B}_W^\text{target} w_0(\hat{B}_W^\text{target} w_0)^\top\right)\right) \\
        & \quad + \E_{w_0 \sim \nu_0}\trace\left(\square \left(  \frac{1}{2}\beta_3\hat{B}_W^\text{source} w'(\hat{B}_W w')^\top -  \beta_3B_W^\text{source} w'(B_W w')^\top\right)\right) \\
        & \quad + \E_{w_0 \sim \nu_0}\trace\left(\square \left( 4 \hat{B}_W^\text{target} w_0(\hat{B}_W^\text{target} w_0)^\top -  \frac{1}{2}\beta_3\hat{B}_W^\text{source} W'(\hat{B}_W^\text{source} W')^\top\right)\right)\\
    & \leq \E_{w_0 \sim \nu_0}\trace\left( \left(4 B_W^\text{target} w_0(B_W^\text{target} w_0)^\top - 4 \hat{B}_W^\text{target} w_0(\hat{B}_W^\text{target} w_0)^\top\right)\right) \|\square \|\\
        & \quad + \beta_3\trace\left( \left(  \frac{1}{2}\hat{B}_W^\text{source} w'(\hat{B}_W w')^\top -  \frac{1}{2} \dot{B}_W W'(\dot{B}_W W')^\top\right)\right) \|\square \|\\
        & \quad +k \overline{\gamma} \|\square\| \\
    & \leq \|\square \| \|\E[w_0 w_0^T]\|_*  \|(B_W^\text{target})^\top B_W^\text{target}  - (\hat{B}_W^\text{target})^\top \hat{B}_W^\text{target}  \| 
        + \|\square \| \|W' (W')^\top\|_* \|\dot{B}_W^\top \dot{B}_W  - \hat{B}_W^\top \hat{B}_W  \|
        + k \overline{\gamma} \|\square\| \\
    & \leq 2\|\square \| \|\E[w_0 w_0^T]\|_*  \|B_W^\text{target}- \hat{B}_W^\text{target}\| \|B_W^\text{target}\| \\
        & \quad +  2\|\square \| \| \|W' (W')^\top\|_* \|\dot{B}_W- \hat{B}_W^\text{source} \| \|\hat{B}_W^\text{source}\| \\
        & \quad  + k \overline{\gamma} \|\square\| \\
    % & \leq  \|\square \| \left(\|\E[w_0 w_0^T]\|_* + \|W' (W')^\top\|_*\right)\| \|\dot{B}_W- \hat{B}_W \| \left(\|\dot{B}_W\| + \|\hat{B}_W\| \right)
    %     + k \overline{\gamma} \|\square\| \\
    % & \lessapprox \|\square \| \left(\|\E[w_0 w_0^T]\|_* + \|W' (W')^\top\|_*\right)\| \|\dot{B}_W- \hat{B}_W \| \|\dot{B}_W\|
    %     + k \overline{\gamma} \|\square\| \\
    & \leq \epsilon^2
\end{align*}
where the second two terms in the first inequality come from Section~\ref{sec: warm-up(appendix)} and the last term in the first inequality comes from the definition of $W'$.
Here $\dot{B}_W = B_W V V^\top = B_W^\text{source}V V^\top$ is a pseudo representation of $B_W^\text{source}$, where $V$ is the one calculated in Line~\ref{line: compute q_1}. And the last inequality comes from the results in Section~\ref{sec: stage2 (appendix)}. Notice that the probability $1-d_W\delta$ comes from the union bound on all the calls of $\calO_\text{offline}^W$.

\textbf{Now combine the bounds above, we have}
\begin{align*}
    \textsc{ER}(\tilde{B}_X, \nu_\text{target})
    &\leq \sigma^2 \left(kd_X\log( (\kappa N_i )/d_W) + \log\frac{1}{\delta} \right) k^* \epsilon^2
\end{align*}

\subsubsection{Detail proofs for warm-up stage}
\label{sec: warm-up(appendix)}

After the first stage, according to Section~\ref{sec: training oracle 1 (appendix)}, as long as $\calE_\text{offline 1}$ holds, we have 
\begin{align*}
    \sin(\hat{B}_X, B_X) \leq \otil\left(\sigma k \sqrt{\frac{d_X}{n_0}}\right)
\end{align*}
Therefore, by Lemma~\ref{lem: B_W est error}, we have with probability $1-d_W\delta$,
\begin{align*}
    \|\hat{B}_W^\text{source} - B_W^\text{source} \| 
    &\leq 2\sqrt{k}\sin(\hat{B}_X, B_X) \|B_W\| + \sqrt{\frac{d_W}{n_0}} (k+ \log(2/\delta)^\frac{1}{4} d_W^{\frac{1}{4}}  \\
    &\leq 2k^{\frac{3}{2}} \sqrt{\frac{d_X}{n_0}} \|B_W\| + 2d_W^{\frac{3}{4}}(k+ \log(2/\delta)^{\frac{1}{4}} \sqrt{\frac{1}{n_0}}
\end{align*}
As long as $n_0 \geq 1024\Bar{\kappa}^2 \left(k^3d_X\Bar{\kappa}^2 + \frac{d_W^{\frac{3}{2}}}{\underline{\sigma}^2}\sqrt{k + \log(1/\delta)}\right)$, by using the Lemma~\ref{lem: helper base est} below, we have for any arbitrary matrix $M$,
\begin{align*}
    \frac{1}{2}  B_W M (B_W)^\top
    \leq \dot{B}_W M \dot{B}_W^\top
    \leq  \frac{3}{2} B_W M (B_W)^\top
\end{align*}

In the other word, $\dot{B}$ can be regarded as a pseudo representation of $B_W^\text{source}$. In all the later epochs, when exploring $k$-subspace according to $q_1^j$, the learner actually learns $\dot{B}_W$.

\begin{lemma}[Guarantee on exploration basis 1]
\label{lem: helper base est}
    Suppose we have the estimated $\hat{B}_W$ satisfies
    \begin{align*}
        & 8\| B_W - \hat{B}_W\| \|B_W \| \leq \frac{1}{2}\lambda_{\min}(B_W  B_W^\top)\\
        & \dot{V} \leftarrow \text{column space of }\textsc{SVD}(\hat{B}_W),
    \end{align*}
    then let $\dot{B}_W = B_W \dot{V} \dot{V}^\top$, we have, for any arbitrary matrix $M$,
    \begin{align*}
        \frac{1}{2}  B_W M (B_W)^\top
        \leq \dot{B}_W M \dot{B}_W^\top
        \leq  \frac{3}{2} B_W M (B_W)^\top
    \end{align*}
\end{lemma}
\begin{proof}
    \begin{align*}
        &\dot{B}_W M \dot{B}_W^\top - B_W M (B_W)^\top\\
        & = \dot{B}_W M \dot{B}_W^\top
            - \hat{B}_W M (\hat{B}_W)^\top 
            + \hat{B}_W M (\hat{B}_W)^\top
                - B_W M B_W^\top \\
        & = (\dot{B}_W - \hat{B}_W) M B_W^\top 
            + \hat{B}_W M (\dot{B}_W -  \hat{B}_W)^\top
            +  ( \hat{B}_W -  B_W) M (\hat{B}_W)^\top + B_W M ( \hat{B}_W -  B_W)^\top \\
        & = (B_W - \hat{B}_W)  \dot{V} \dot{V}^\top M B_W^\top 
            + \hat{B}_W  M \dot{V} \dot{V}^\top (B_W - \hat{B}_W)^\top 
            + (\hat{B}_W -  B_W) M (\hat{B}_W)^\top + B_W M (\hat{B}_W -  B_W)^\top
    \end{align*}
    Therefore, according to our assumption, we can upper bound the above as
    \begin{align*}
        \dot{B}_W M \dot{B}_W^\top - B_W M (B_W)^\top
        & \leq 2 \| B_W - \hat{B}_W\| 
        \left(\|\hat{B}_W\| + \|B_W\|\right) M \\
        & \leq \left(4\| B_W - \hat{B}_W\| \|B_W \| +  2\| B_W - \hat{B}_W\|_2^2\right) M \\
        & \leq 8\| B_W - \hat{B}_W\| \|B_W \| M \\
        & \leq \frac{1}{2}\lambda_{\min}(B_W B_W^\top) M
        \leq \frac{1}{2}  B_W M B_W^\top
    \end{align*}
    Similarly, it can be lower bounded by $-\frac{1}{2}  B_W M B_W^\top$. Therefore we can get the target result by rearranging.
\end{proof}

\subsubsection{Detail proofs for task-agnostic exploration strategy}
\label{sec: stage2 (appendix)}

First, we upper bound two $ \|B_W- \hat{B}_W \| $ terms. 
From section~\ref{sec: training oracle 1 (appendix)}, as long as $\calE_\text{offline 1}$ holds, we have
\begin{align*}
    \sin(\hat{B}_X, B_X) \leq \otil\left( k\sqrt{\frac{ d_X}{n_1}} \|B_W^\text{source}\|\right)
\end{align*}
Therefore, by Lemma~\ref{lem: B_W est error}, we have w.h.p at least $1-(k+d_W^\text{target})\delta$
\begin{align*}
    \|\hat{B}^\text{source}_W - B_W^\text{source} \| 
    &\leq 2\sqrt{k}\sin(\hat{B}_X, B_X) \|B_W^\text{source}\| + \sqrt{\frac{k}{n_1}} (k+ \log(2/\delta)^\frac{1}{4} k^{\frac{1}{4}} \\
    &\leq 2k^\frac{3}{2} \sqrt{\frac{d_X}{n_1}} \|B_W^\text{source}\|^2\\
    \|\hat{B}^\text{target}_W - B_W^\text{target} \| 
    &\leq 2k\sin(\hat{B}_X, B_X) \|B_W^\text{target}\| + \sqrt{\frac{1}{\dot{n}_\text{target}}} (k+ \log(2/\delta)^\frac{1}{4} (d_W^\text{target})^{\frac{1}{4}} \\
    &\leq 2k^\frac{3}{2} \sqrt{\frac{d_X}{n_1}} \|B_W^\text{target}\|^2 + 2\sqrt{\frac{1}{\dot{n}_\text{target}}} (k+ \log(2/\delta)^\frac{1}{4}(d_W^\text{target})^{\frac{1}{4}} \\
    & \leq 4k^\frac{3}{2} \sqrt{\frac{d_X}{n_1}} \|B_W^\text{target}\|^2
\end{align*}
where the last equality holds as long as 
$\dot{n}_\text{target} \geq n_1 \frac{\sqrt{(k+ \log(2/\delta)) d_W^\text{target}}}{k^3d_X \|B_W^\text{target}\|^2}$.

Next, we upper bound the $\|W'(W')\|$ according to Lemma~\ref{lem: upper bound of WW' norm}.
\begin{align*}
    \| W' (W')^\top\|_*
    & \lessapprox \frac{1}{\sigma_{\min}^2(B_W^\text{source})}\|B_W^\text{target} \E_{\nu_\text{target}}[w_0 w_0^\top] (B_W^\text{target})^\top\|_* \\
    & \leq \frac{1}{\sigma_{\min}^2(B_W^\text{source})} \|B_W^\text{target}\|^2 \| \E_{\nu_\text{target}}[w_0 w_0^\top]\|_* \\
    & \leq \frac{1}{\sigma_{\min}^2(B_W^\text{source})} \|B_W^\text{target}\|^2 
    % \leq \bar{\kappa}^2
\end{align*}
% where the last inequality comes from our assumption on $\underline{\sigma}, \bar{\kappa}$.

Finally, we have, by definition
\begin{align*}
     & \Bar{\gamma} \leq 2 \|\hat{B}_W^\text{target} -  B_W^\text{target}\| \|B_W^\text{target}\| \\
     & \|\square\| 
     \leq \frac{k}{n_1\sigma_{\min}^2(\dot{B}_W)}
     \lessapprox \frac{k}{n_1\sigma_{\min}^2(B_W^\text{source})}
\end{align*}

Combine all above, we have the upper bound
\begin{align*}
    & \|\square \| \left(\|\E[w_0 w_0^T]\|_*  \|B_W^\text{target}- \hat{B}_W^\text{target}\| \|B_W^\text{target}\| 
        +  \| \|W' (W')^\top\|_* \|\dot{B}_W- \hat{B}_W^\text{source} \| \|B_W^\text{source}\|
        + k \overline{\gamma} \right) \\
    &\lessapprox \frac{k}{n_1\sigma_{\min}^2(B_W^\text{source})}*k^\frac{3}{2} \sqrt{\frac{d_X}{n_1}}
        *\left( k\|B_W^\text{target}\|^3 + \|B_W^\text{source}\|\|B_W^\text{target}\|^2\bar{\kappa}^2\right) \\
    & \leq k^\frac{5}{2}d_X^\frac{1}{2} n_1^{-\frac{3}{2}}\|B_W^\text{target}\|^2 \left( \frac{k\|B_W^\text{target}\|}{\underline{\sigma^2}} + \frac{\bar{\kappa}^3}{\underline{\sigma}} \right)   \\
    & \leq k^\frac{5}{2}d_X^\frac{1}{2} n_1^{-\frac{3}{2}} d_W \left(k\sqrt{d_W}\underline{\sigma}^{-2} + \Bar{\kappa}^3 \underline{\sigma}^{-1}\right)
\end{align*}
% \begin{align*}
%     & \|\square \| \left(\|\E[w_0 w_0^T]\|_* + \|W' (W')^\top\|_*\right)\| \|\dot{B}_W- \hat{B}_W \| \|\dot{B}_W\|
%         + k \overline{\gamma} \|\square\|\\
%     &\frac{k}{n_1\sigma_{\min}^2(B_W^\text{source})}* \left(\frac{1}{\sigma_{\min}^2(B_W^\text{source})}\|B_W^\text{target} \E_{\nu_\text{target}}[w_0 w_0^\top] (B_W^\text{target})^\top\|\right)*\|B_W\|* \left( 2k \sqrt{\frac{d_X}{n_1}} \|B_W\|^2\right) \\
%         &\quad + \frac{k}{n_1\sigma_{\min}^2(B_W^\text{source})} *k * k \sqrt{\frac{d_X}{n_1}} \|B_W\|^2* \|B_W^\text{target}\|\\
%     &\leq  \bar{\kappa}^3\underline{\sigma}^{-1}k^2d_X^\frac{1}{2}n_1^{-\frac{3}{2}} \|B_W^\text{target} \E_{\nu_\text{target}}[w_0 w_0^\top] (B_W^\text{target})^\top\| 
%         + k^3 \Bar{\kappa}^2 d_X^\frac{1}{2} \|B_W^\text{target}\| n_1^{-\frac{3}{2}} \\
%     & \leq k^2\Bar{\kappa}^2\|B_W^\text{target}\|d_X^\frac{1}{2} n_1^{-\frac{3}{2}} 
%     \left(\Bar{\kappa}\underline{\sigma}^{-1}\|B_W^\text{target}\| + k\right)
% \end{align*}

As long as $n_1 \geq \epsilon_j^{-\frac{4}{3}}k^{\frac{5}{3}}d_W^{\frac{2}{3}}d_X^\frac{1}{3}
    \left( k^{\frac{2}{3}}d_W^{\frac{1}{3}}\underline{\sigma}^{-\frac{4}{3}} + \Bar{\kappa}^2\underline{\sigma}^{-\frac{2}{3}}\right)$,  we have the final bound $\epsilon_j^2$.

\subsubsection{Auxillary lemmas}

\begin{lemma}
\label{lem: B_W est error}
    Consider any $t$ regression tasks parameterized by $\{\dot{v}_i\}_{i \in [n]}$. Denote $\dot{V} = [\dot{v}_1, \dot{v}_2, \ldots, \dot{v}_t]$ and $|X_{\dot{v}_i}| = n$ for all $i \in [t]$, define
    \begin{align*}
        \hat{B}_W = \sum_{i \in k} \hat{w}_i \dot{v}_i^\top,
        \text{ where }
        \hat{w}_i =\argmin_{w \in \fR^k} \|X_{\dot{v}_i}\hat{B}_X^\top w  - Y_{\dot{v}_i}\|_2, 
    \end{align*}
    then we have with probability at least $1-\delta$,
    \begin{align*}
        \|\hat{B}_W - \dot{B}_W\| = \|\hat{B}_W - B_W \dot{V} \dot{V}^\top\| 
        \leq 2\sqrt{k}\sin(\hat{B}_X, B_X) \|\dot{B}_W\| + \sqrt{\frac{1}{n}} (k+ \log(2/\delta)^\frac{1}{4} |\dot{V}|^{\frac{1}{4}}
    \end{align*}
\end{lemma}
\begin{proof}
    From \cite{chen2022active}, we get that the explicit form of $\hat{w}_i$, which is the estimation of actual $B_w \dot{v}_i$ as
    \begin{align*}
        \left(\hat{B}_X X_{\dot{v}_i}^\top X_{\dot{v}_i} \hat{B}_X^\top\right)^{-1} \hat{B}_X X_{\dot{v}_i}^\top X_{\dot{v}_i} B_X^\top B_X \dot{v}_i
            + \left(\hat{B}_X X_{\dot{v}_i}^\top X_{\dot{v}_i} \hat{B}_X^\top\right)^{-1} \hat{B}_X X_{\dot{v}_i}^\top \xi_w
    \end{align*}
    By abusing notation a little bit, here we use subscription $i$ to denote the items that associate the task encoded by $\dot{v}_i$. Therefore, we have 
    \begin{align*}
        \hat{B}_W
        & = \sum_{i =1}^t \left(\hat{B}_X X_i^\top X_i \hat{B}_X^\top\right)^{-1} \hat{B}_X X_i^\top X_i B_X^\top B_W \dot{v}_i \dot{v}_i^\top
            + \sum_{i =1}^t\left(\hat{B}_X X_i^\top X_i \hat{B}_X^\top\right)^{-1} \hat{B}_X X_i^\top \mathbf{\xi}_i \dot{v}_i^\top \\
        & = \sum_{i =1}^t \left(\hat{B}_X X_i^\top X_i \hat{B}_X^\top\right)^{-1} \hat{B}_X X_i^\top X_i \left(\hat{B}_X^\top \hat{B}_X + \hat{B}_{X,\perp}^\top \hat{B}_{X,\perp}\right)  B_X^\top B_W \dot{v}_i \dot{v}_i^\top \\
            & \quad + \sum_{i =1}^t \left(\hat{B}_X X_i^\top X_i \hat{B}_X^\top\right)^{-1} \hat{B}_X X_i^\top \mathbf{\xi}_i \dot{v}_i^\top \\
        & = \hat{B}_X B_X^\top \dot{B}_W 
            + \sum_{i =1}^t \left(\hat{B}_X X_i^\top X_i \hat{B}_X^\top\right)^{-1} \hat{B}_X X_i^\top X_i \hat{B}_{X,\perp}^\top \hat{B}_{X,\perp}  B_X^\top B_W \dot{v}_i \dot{v}_i^\top\\
            & \quad + \sum_{i =1}^d\left(\hat{B}_X X_i^\top X_i \hat{B}_X^\top\right)^{-1} \hat{B}_X X_i^\top \mathbf{\xi}_i \dot{v}_i^\top
    \end{align*}
    And the estimation difference between $B_W, \hat{B}_W $ can be decomposed into three parts
    \begin{align*}
        \| \dot{B}_W -  \hat{B}_W \|
         & \leq \|\left( \hat{B}_X B_X^\top - I_k \right) \dot{B}_W \| \\
            & \quad + \| \sum_{i =1}^t \left(\hat{B}_X X_i^\top X_i \hat{B}_X^\top\right)^{-1} \hat{B}_X X_i^\top X_i \hat{B}_{X,\perp}^\top \hat{B}_{X,\perp}  B_X^\top B_W \dot{v}_i \dot{v}_i^\top \|\\
            & \quad + \|  \sum_{i =1}^t\left(\hat{B}_X X_i^\top X_i \hat{B}_X^\top\right)^{-1} \hat{B}_X X_i^\top \mathbf{\xi}_i \dot{v}_i^\top\| \\
        & \leq \|\left( \hat{B}_X B_X^\top - I_{k}\right)\| \|\dot{B}_W \| \\
           & \quad + \max_i \| \left(\hat{B}_X X_i^\top X_i \hat{B}_X^\top\right)^{-1} \hat{B}_X X_i^\top X_i \hat{B}_{X,\perp}^\top \hat{B}_{X,\perp} B_X^\top\| \|  \sum_{i=1}^t B_W \dot{v}_i \dot{v}_i^\top \|\\
            % & \quad + \| \sum_{i =1}^d \left(\hat{B}_X X_i^\top X_i \hat{B}_X^\top\right)^{-1} \hat{B}_X X_i^\top X_i \hat{B}_{X,\perp}^\top \hat{B}_{X,\perp}  B_X^\top\| \|\dot{B}_W \| \yf{double check} \\
            & \quad + \|  \sum_{i=1}^t\left(\hat{B}_X X_i^\top X_i \hat{B}_X^\top\right)^{-1} \hat{B}_X X_i^\top \mathbf{\xi}_i \dot{v}_i^\top\| 
    \end{align*}
    By using Lemma~\ref{lem: rep learning 1}  and Lemma~\ref{lem: rep learning 2}, we can bound the first two terms by 
    \begin{align*}
        2\sqrt{k}\sin(\hat{B}_X, B_X) \|\dot{B}_W\|
    \end{align*}
    Now we are going to bound the last term which is the noise term.  
    \begin{align*}
        &\|  \sum_{i =1}^{|\dot{V}|} \left(\hat{B}_X X_i^\top X_i \hat{B}_X^\top\right)^{-1} \hat{B}_X X_i^\top \mathbf{\xi}_i \dot{v}_i^\top\|^2 \\
        &=  \lambda_{\max} \left( \sum_{i =1}^{|\dot{V}|} \left(\hat{B}_X X_i^\top X_i \hat{B}_X^\top\right)^{-1} \hat{B}_X X_i^\top \mathbf{\xi}_i \dot{v}_i^\top\right) \left( \sum_{i =1}^{|\dot{V}|} \left(\hat{B}_X X_i^\top X_i \hat{B}_X^\top\right)^{-1} \hat{B}_X X_i^\top \mathbf{\xi}_i \dot{v}_i^\top\right)^\top \\
        & \leq \lambda_{\max} \left( \sum_{i =1}^{|\dot{V}|} \left(\hat{B}_X X_i^\top X_i \hat{B}_X^\top\right)^{-1} \hat{B}_X X_i^\top \xi_i \xi_i^\top X_i \hat{B}_X^\top \left(\hat{B}_X X_i^\top X_i \hat{B}_X^\top\right)^{-1} \right) 
    \end{align*}
    Note that, $x_i \sim \calN(0, I_d)$ and
    \begin{align*}
        \left(\hat{B}_X X_i^\top X_i \hat{B}_X^\top\right)^{-1} \hat{B}_X X_i^\top \xi_i 
        & \sim \calN \left(0, \left( \left(\hat{B}_X X_i^\top X_i \hat{B}_X^\top\right)^{-1} \hat{B}_X X_i^\top X_i \hat{B}_X^\top \left(\hat{B}_X X_i^\top X_i \hat{B}_X^\top\right)^{-1} \right) \right)\\
        & \sim \calN \left(0,  \left(\hat{B}_X X_i^\top X_i \hat{B}_X^\top\right)^{-1} \right)
    \end{align*}
    Therefore, by the concentration inequality of the covariance matrix, we have, w.h.p $1-\delta$,
    \begin{align*}
        \lambda_{\max} \left( \sum_{i =1}^{|\dot{V}|} \left(\hat{B}_X X_i^\top X_i \hat{B}_X^\top\right)^{-1} \hat{B}_X X_i^\top \xi_i \xi_i^\top X_i \hat{B}_X^\top \left(\hat{B}_X X_i^\top X_i \hat{B}_X^\top\right)^{-1} \right) 
        \leq \frac{1}{n} \sqrt{(k+ \log(2/\delta) |\dot{V}|}
    \end{align*}
    Combining everything above, we have the final bound.
\end{proof}

% \begin{lemma}
%     Let $\dot{B}_W = B_W \dot{V} \dot{V}^\top$. For any arbitrary $w_0$, we define
%     \begin{align*}
%         & \|\nu^*\|_2 = \argmin \|\nu\|_2
%         \text{ s.t. }
%         B_X \nu = B_W w_0 \\
%         & \|\dot{\nu}\|_2 = \argmin \|\nu\|_2
%         \text{ s.t. }
%         \dot{B}_X \nu = B_W w_0.
%     \end{align*}   
%     Then, as long as $ \frac{1}{c} I_k \leq V^\top \dot{V} \leq c I_k$ for some constant $c$, we have $\|\dot{\nu}\|_2 \leq \order(\|\nu^*\|_2)$.
% \end{lemma}
% \begin{proof}
%     By applying svd decomposition, we have $B_X = UDV^\top$ and $\dot{B}_X = UDV^\top \dot{V} \dot{V}^\top$, then by definition, we have
%     \begin{align*}
%         UDV^\top \dot{V} \dot{V}^\top \dot{V} \dot{\lambda}
%         = UDV^\top V \lambda^*.
%     \end{align*}
%     where $\nu^* = V^* \lambda^*, \dot{\nu} = \dot{V} \dot{\lambda}$. Given $\lambda_{\min}(D) \geq 1$, it can be further arranged to
%     \begin{align*}
%         \dot{\lambda} = V^\top \dot{V} \lambda^*
%     \end{align*}
%     Therefore, $\|\dot{\nu}\| = \|\dot{\lambda}\| \leq \|V^\top \dot{V} \lambda^*\| \leq \order(\|\nu^*\|)$.
% \end{proof}

\begin{lemma}
\label{lem: rep learning 1}
Given $\hat{B}_X, B_X$ are orthonormal matrices, as well as $\E[x x^T] = I_{d_X}$ for all tasks $w$, we have
\begin{align*}
     \|I_{k}  - \hat{B}_X B_X^\top\|
     \leq \order\left( \sqrt{k} \sin(\hat{B}_X, \hat{B}_X)\right)
\end{align*}
\end{lemma}
\begin{proof}
    Denote $B_X \hat{B}_X^\top  = UDV^\top$, by definition, we have
    $D = \text{diag}(\cos{\theta_1}, \cos{\theta_2}, \ldots, \cos{\theta_k})$ from the largest singular value to minimum singular value and
    $\sin{\theta_k} \leq \sin(\hat{B}_X, \hat{B}_X)$. Therefore we have,
    \begin{align*}
       \text{Tr}(\hat{B}_X B_X^\top) 
       \geq k \sqrt{1- \sin^2(\hat{B}_X, B_X)}
       \geq k - k \sin^2(\hat{B}_X, B_X)
    \end{align*}
    And 
    \begin{align*}
        \|I_{k}  - \hat{B}_X B_X^\top\|^2
        & =  \lambda_{\max} \left(I_{k}  - \hat{B}_X B_X^\top\right)^\top \left(I_{k}  - \hat{B}_X B_X^\top\right)\\
        & \leq \trace \left(I_{k}  - \hat{B}_X B_X^\top\right)^\top \left(I_{k}  - \hat{B}_X B_X^\top\right)\\
        & \leq \trace \left( I_k + \left(\hat{B}_X B_X^\top\right)^\top \hat{B}_X B_X^\top
        - \left(\hat{B}_X B_X^\top\right)^\top -  \hat{B}_X B_X^\top\right) \\
        & \leq 2k - 2k + 2k \sin^2(\hat{B}_X, B_X) \leq 2k \sin^2(\hat{B}_X, B_X)
    \end{align*}
\end{proof}

\begin{lemma}[Restate from \cite{tripuraneni2021provable}]
\label{lem: rep learning 2}
    Given $\hat{B}_X, B_X$ are orthonormal matrices, as well as $\E[x x^T] = I_{d_X}$ for any fixed task $w$, we have
    \begin{align*}
        \|\left(\hat{B}_X X_w^\top X_w \hat{B}_X^\top\right)^{-1} \hat{B}_X X_w^\top X_w \hat{B}_{X,\perp}^\top \hat{B}_{X,\perp} B_X^\top\|
        \leq \sin(\hat{B}_X, \hat{B}_X)
    \end{align*}
\end{lemma}
\begin{proof}
    Here we follow the same proof step as in \cite{tripuraneni2021provable}. (Bound on the second error term in Lemma 19)
    \begin{align*}
        &\|\left(\hat{B}_X X_w^\top X_w \hat{B}_X^\top\right)^{-1} \hat{B}_X X_w^\top X_w \hat{B}_{X,\perp}^\top \hat{B}_{X,\perp} B_X^\top\| \\
        & \leq \|\left(\hat{B}_X X_w^\top X_w \hat{B}_X^\top\right)^{-1} \hat{B}_X X_w^\top X_w \hat{B}_{X,\perp}^\top \| \sin(\hat{B}_X, \hat{B}_X) \\
        & \leq \sin(\hat{B}_X, \hat{B}_X)
    \end{align*}
    % where $C_\text{condi} = C_\text{max}/C_\text{min}$ = 1 in this case.
\end{proof}

\subsection{Lemmas about the properties of $W^\prime$}
\begin{lemma}
\label{lem: est rank}
    \begin{align*}
         \text{rank}(\hat{B}_W W W' \hat{B}_W^\top) 
         \leq \text{rank}(B_W \E_{\nu_\text{target}}[w_0 w_0^\top] B_W^\top) 
    \end{align*}
\end{lemma}
\begin{proof}
    By using Welys inequality, we have for any eigenvalue $i \in [k]$,
    \begin{align*}
        & |\lambda_i\left(\hat{B}_W^\text{target} \E[w_0 w_0^\top] (\hat{B}_W^\text{target})^\top \right)
        - \lambda_i\left(B_W^\text{target} \E[w_0 w_0^\top] (B_W^\text{target})^\top \right)| \\
        & \leq \|\hat{B}_W^\text{target} \E[w_0 w_0^\top] (\hat{B}_W^\text{target})^\top
            - B_W^\text{target} \E[w_0 w_0^\top] (B_W^\text{target})^\top\| \\
        & \leq \|\hat{B}_W^\text{target}  (\hat{B}_W^\text{target})^\top
            - B_W^\text{target} (B_W^\text{target})^\top\| \\
        & \leq 2 \|\hat{B}_W^\text{target} -  B_W^\text{target}\| \|B_W^\text{target}\| \\
        & \leq  \left(2k^{\frac{3}{2}} \sqrt{\frac{d_X}{n_1}} \|B_W^\text{source}\|^2 + 2\sqrt{\frac{k}{\dot{n}_\text{target}}} \right)\|B_W^\text{target}\|
    \end{align*}
    where the last inequality comes from Lemma~\ref{lem: B_W est error} and the fact $
    \sin(\hat{B}_X, B_X) \leq \otil\left( k\sqrt{\frac{ d_X}{n_1}} \|B_W^\text{source}\|\right)
    $. Therefore, for all the $i \geq k^*$,
    \begin{align*}
        \lambda_i\left(\hat{B}_W^\text{target} \E[w_0 w_0^\top] (\hat{B}_W^\text{target})^\top \right) \geq \left(2k \sqrt{\frac{d_X}{n_1}} \|B_W^\text{source}\|^2 + 2\sqrt{\frac{k}{\dot{n}_\text{target}}} \right)\|B_W^\text{target}\|
    \end{align*}
    Clipping those non-significant directions leads to the result.
\end{proof}

\begin{lemma}
\label{lem: bound on max W_i}
Define $W_i^* = \argmin_{v} \|v\|_2 \quad, \text{s.t.} \hat{B}_W^\text{source} v = \hat{u}_i \hat{\Lambda}_i$, we have
    \begin{align*}
        \max_i \|W_i^\prime\| \leq \min\{k^*, \kappa^2(B_W^\text{source})\} \max_i \|W_i^*\|
    \end{align*}
\end{lemma}
\begin{proof}
    By definition of $W'$, we have, for any $W'_i$,
    \begin{align*}
        W'_i = \argmin_{v} \|v\|_2 \quad, \text{s.t.} \hat{B}_W^\text{source} v = \hat{u}_i \hat{\Lambda}_i
    \end{align*}
    where $\hat{U}, \hat{\Lambda}_i \leftarrow \text{Eig}(\E[\hat{B}_W^\text{target}w_0w_0^\top(\hat{B}_W^\text{target})^\top 
    ])$.
    By solving this optimization, we get 
    \begin{align*}
        W'_i
        =  (\hat{B}_W^\text{source})^\top \left( \hat{B}_W^\text{source}  (\hat{B}_W^\text{source})^\top\right)^{-1} \hat{u}_i \sqrt{\hat{\Lambda}_i}
    \end{align*}
    and therefore, 
    \begin{align*}
        \max_i \|W'_i\|^2
        & = \max_i \sqrt{\hat{\lambda}_i} \hat{u}_i^\top \left( \hat{B}_W^\text{source}  (\hat{B}_W^\text{source})^\top\right)^{-1} (\hat{B}_W^\text{source}) (\hat{B}_W^\text{source})^\top \left( \hat{B}_W^\text{source}  (\hat{B}_W^\text{source})^\top\right)^{-1} \hat{u}_i \hat{\Lambda}_i \\
        & = \max_i \sqrt{\hat{\lambda}_i} \hat{u}_i^\top \left( \hat{B}_W^\text{source}  (\hat{B}_W^\text{source})^\top\right)^{-1} \hat{u}_i \sqrt{\hat{\Lambda}_i} \\
        % & =  \max_i \trace \left(\left( \hat{B}_W^\text{source}  (\hat{B}_W^\text{source})^\top\right)^{-1} \hat{u}_i \hat{\Lambda}_i \hat{u}_i^\top\right)\\
        % & \lessapprox \max_i \trace \left(\left( B_W^\text{source}  (B_W^\text{source})^\top\right)^{-1} \hat{u}_i \hat{\Lambda}_i \hat{u}_i^\top\right)
        & \lessapprox \max_i \hat{\lambda}_i \hat{u}_i^\top \left( B_W^\text{source}  (B_W^\text{source})^\top\right)^{-1} \hat{u}_i 
    \end{align*}
    where the last inequality comes from Lemma~\ref{lem: helper base est}. Similarly, the ground truth $W^*$ can be represented as 
    \begin{align*}
        \max_i \|W_i^*\|^2
        & = \max_i \lambda_i u_i^\top \left( B_W^\text{source}  (B_W^\text{source})^\top\right)^{-1} u_i \\
        \text{where, }
        & \E_{w_0} \left[B_W^\text{target}w_0 w_0^\top (B_W^\text{target})^\top\right].
    \end{align*}
    and denote $H = \hat{U} \hat{\Lambda} \hat{U}^\top - \E_{w_0} \left[B_W^\text{target}w_0 w_0^\top (B_W^\text{target})^\top\right] $. 
    
     Now we are now going to upper bound $\max_i\|W'_i\|$ in terms of $\max_i\|W_i^*\|$. Suppose $j = \argmax\|W_i'\|$ and $B_W^\text{target} = U \Lambda U^\top$. 
     
     Firstly, we will lower bound the $\hat{\lambda}_i$.  Given $\|\E_{w_0} \left[B_W^\text{target}w_0 w_0^\top (B_W^\text{target})^\top\right]\| \leq \frac{1}{2k}$, we can always found an $\|W_i^\prime\|^2 \geq \frac{1}{2k\sigma^2_{\max}(B_W^\text{source})}$. Therefore, we have
    \begin{align*}
        \hat{\lambda}_j \geq \frac{1}{2k\kappa(B_W^\text{source})}
    \end{align*}
    Then we consider the following two cases.

    \textbf{(Case 1) When $\kappa(B_W^\text{source})$ is small: }
    By Wely's inequality, there always exists some $u_m, \lambda_m$ that $\hat{\lambda}_j \leq \order(\lambda_m)$. Therefore, 
    \begin{align*}
        \hat{\lambda}_j \hat{u}_j^\top \left( B_W^\text{source}  (B_W^\text{source})^\top\right)^{-1} \hat{u}_j 
        & \leq \hat{\lambda}_j u_m^\top \left( B_W^\text{source}  (B_W^\text{source})^\top\right)^{-1} u_m  \kappa(B_W^\text{source})^2 \\
        & \leq \lambda_m u_m^\top \left( B_W^\text{source}  (B_W^\text{source})^\top\right)^{-1} u_m  \kappa(B_W^\text{source})^2 \\
        & \leq \max_i \|W^*_i\|^2 \kappa(B_W^\text{source})^2
    \end{align*}

    \textbf{(Case 2) When $\kappa(B_W^\text{source})$ is large: } Decompose $\hat{B}_W^\text{target} W' (W')^\top (\hat{B}_W^\text{target})^\top$ as follows
    \begin{align*}
        \hat{B}_W^\text{target} W' (W')^\top (\hat{B}_W^\text{target})^\top 
        &= \hat{U}_0 \hat{\Lambda}_0 \hat{U}_0^\top + \hat{U}_1 \hat{\Lambda}_1 \hat{U}_1^\top\\
        \text{where, }
        & \hat{u}_j \in \hat{U}_0
        \text{ and } \lambda_{\min}(\hat{\Lambda}_0) - \lambda_{\min}(\hat{\Lambda}_0) \geq \frac{1}{4} \hat{\lambda}_j
    \end{align*}
    Correspondingly, we can decompose $\E_{w_0} B_W^\text{target} w_0 w_0^\top (B_W^\text{target})^\top $ as the same shape
    \begin{align*}
        \E_{w_0} \left[ B_W^\text{target} w_0 w_0^\top (B_W^\text{target})^\top \right]
        = U_0 \Lambda_0 U_0^\top + U_1 \Lambda_1 U_1^\top
    \end{align*}    
    By using Davis-Kahn theorem, we have
    \begin{align*}
        \| U_1^\top \hat{u}_j \| 
        \leq \| U_1^\top \hat{U}_0 \| 
        \leq \frac{\| U_1^\top H \hat{U}_0\| }{\frac{1}{4}\hat{\lambda}_j}
        \lessapprox k\|H\| \kappa(B_W^\text{source})
    \end{align*}
    Since
    \begin{align*}
        \|H\|
        & \leq \Bar{\gamma} 
        + \left\|\E_{w_0} \left[B_W^\text{target}w_0 w_0^\top (B_W^\text{target})^\top\right]
            - \E_{w_0} \left[\hat{B}_W^\text{target}w_0 w_0^\top (\hat{B}_W^\text{target})^\top\right] \right\|\\
        & \leq 2 \left\|\E_{w_0} \left[B_W^\text{target}w_0 w_0^\top (B_W^\text{target})^\top\right] 
            - \E_{w_0} \left[\hat{B}_W^\text{target}w_0 w_0^\top (\hat{B}_W^\text{target})^\top\right] \right\|\\
        & \leq 2 \left(2k \sqrt{\frac{d_X}{n_1}} \|B_W^\text{source}\|^2 + 2\sqrt{\frac{k}{\dot{n}_\text{target}}} \right)\|B_W^\text{target}\|
    \end{align*}
    then we have
    \begin{align*}
        \| U_1^\top \hat{u}_j \| 
        \lessapprox 8 k \left(2k \sqrt{\frac{d_X}{n_1}} \|B_W^\text{source}\|^2 + 2\sqrt{\frac{k}{\dot{n}_\text{target}}} \right) \|B_W^\text{target}\| \kappa(B_W^\text{source})
        \leq \frac{1}{2}
    \end{align*}
    % and 
    % \begin{align*}
    %     \frac{\| U_1^\top H \hat{U}_0\| }{\frac{1}{4}\hat{\lambda}_j}
    %     \lessapprox 4\frac{\|H\|}{\hat{\lambda}_j}
    %     \leq 8 \left(2k \sqrt{\frac{d_X}{n_1}} \|B_W^\text{source}\|^2 + 2\sqrt{\frac{k}{\dot{n}_\text{target}}} \right) \frac{\|B_W^\text{target}\|}{\lambda_{k^*}}
    % \end{align*}
    which suggests
    $
        \|U_0^\top \hat{u}_j\|
        = \|[U_0, U_1]^\top \hat{u}_j - [0, U_1]^\top \hat{u}_j \|
        % \geq \|\|[U_0, U_1]^\top \hat{U}_0\| - |[0, U_1]^\top \hat{U}_0 \|
        \geq 1 - \| U_1^\top \hat{u}_j \| 
        \geq \frac{1}{2}.
    $
    Therefore, there exists some $u_m$ as one of the columns of $U_0$ that such $u_m^\top \hat{u}_j \leq \order(\sqrt{\frac{1}{k^*}})$. And therefore, we have
    \begin{align*}
        \hat{\lambda}_j \hat{u}_j^\top \left( B_W^\text{source}  (B_W^\text{source})^\top\right)^{-1} \hat{u}_j 
        & \leq k^* \lambda_m (\hat{u}_m^\top \hat{u}_j)  \hat{u}_j^\top \left( B_W^\text{source}  (B_W^\text{source})^\top\right)^{-1} \hat{u}_j ( \hat{u}_j^\top \hat{u}_m) \\
        & \leq k^* \lambda_m \hat{u}_m^\top \left( B_W^\text{source}  (B_W^\text{source})^\top\right)^{-1} \hat{u}_m\\
        & \leq k^*  \max_i \|W^*_i\|^2
    \end{align*}
\end{proof}

\begin{lemma}
\label{lem: upper bound of WW' norm}
\begin{align*}
    \| W' W'^\top\|_*
    \leq \order\left(\frac{1}{\sigma_{\min}^2(B_W^\text{source})}\|B_W^\text{target} \E_{\nu_\text{target}}[w_0 w_0^\top] (B_W^\text{target})^\top\|_*\right)
\end{align*}
\end{lemma}
\begin{proof}
    \begin{align*}
        \| W' W'^\top\|_*
        & \leq \frac{1}{\sigma_{\min}^2(B_W^\text{source})}\|B_W^\text{source} W' (W')^\top (B_W^\text{source})^\top\|_* \\
        & \leq \frac{1}{\sigma_{\min}^2(B_W^\text{source})}\|\hat{B}_W^\text{source} W' (W')^\top (\hat{B}_W^\text{source})^\top\|_* \\
        & \leq \frac{1}{\sigma_{\min}^2(B_W^\text{source})}\|\hat{B}_W^\text{target} \E_{\nu_\text{target}}[w_0 w_0^\top] (\hat{B}_W^\text{target})^\top\|_* \\
        & \leq \frac{1}{\sigma_{\min}^2(B_W^\text{source})}\|B_W^\text{target} \E_{\nu_\text{target}}[w_0 w_0^\top] (B_W^\text{target})^\top\|_* 
    \end{align*}
\end{proof}

%% file: appendix-sampleComplexity1.tex
\subsection{Sample complexity analysis -- Formal version of Theorem~\ref{them: main}}

\begin{theorem}[Formal theorem]
\label{them: sample complexity (appendix)}
    By running Algo.~\ref{algo: main}, in order to let $ER(\hat{\phi}_X, \nu_\text{target}) \leq \varepsilon^2$ with probability $1-\delta$, where $\delta \geq (d_X k)^{10}$, then the number of source samples $n_\text{source}$ is at most
    \begin{align*}
       & \otil\left(\sigma^2 (k^*)^2 \min\{\kappa(B_W^\text{source},k^*\}\max_i \|W_j^*(i)\|_2^2kd_X \varepsilon^{-2}\right)\\
        & \quad  + \otil\left(\varepsilon^{-\frac{4}{3}}k^{\frac{7}{3}} d_W^{\frac{2}{3}}d_X
        \left( k^{\frac{2}{3}}d_W^{\frac{1}{3}}\underline{\sigma}^{-\frac{4}{3}} + \Bar{\kappa}^2\underline{\sigma}^{-\frac{2}{3}}\right)\right)\\
        & \quad + \otil\left(\Bar{\kappa}^2 \sqrt{k} \left(k^2 d_X\Bar{\kappa}^2 + \frac{d_W^{\frac{3}{2}}}{\underline{\sigma}^2}\sqrt{k + \log(d_W/\delta)}\right)\right)
    \end{align*}
    Here $k^* = \text{rank}(\E_{w_0 \sim \nu_\text{target}}B_W w_0 w_0^\top B_W^\top)$ represents the effective dimension of target and 
    \begin{align*}
        W^*_i = \argmin_{w \in \calW_\text{source}} \|w\|_2 
        \quad \text{s.t} \quad B_W^\text{source} w = u_i \sqrt{\lambda_i} 
        \text{ where } U, \Lambda \leftarrow \text{Eig}(\E_{w_0 \sim \nu_\text{target}}B_W w_0 w_0^\top B_W^\top).
    \end{align*}
    as long as,
    \begin{align*}
        & \dot{n}_\text{target} 
        \geq \widetilde{\Omega}\left(\varepsilon^{-\frac{4}{3}} (k^*)^{\frac{2}{3}}
        \left( d_W^{\frac{1}{2}}\underline{\sigma}^{-\frac{4}{3}} 
        + k^{-\frac{2}{3}}d_W^\frac{1}{6}\Bar{\kappa}^2\underline{\sigma}^{-\frac{1}{3}}\right)
        \sqrt{k+ \log(d_W/\delta)}
        \right) \\
        &n_\text{target} \geq \widetilde{\Omega}\left((k+\log(1/\delta))\varepsilon^{-2} \right)
    \end{align*}
\end{theorem}
\begin{proof}
    By setting the target excess risk $\varepsilon^2$ and the generalization guarantees in Theorem~\ref{them: excess risk (appendix)}, we have
    \begin{align}
        \sigma^2 \left(kd_X\log( (\kappa N_i )/d_W) + \log\frac{1}{\delta} \right) k^* \epsilon_j^2 = \varepsilon^2
    \end{align}
    After some rearrangement, we can directly have the guarantees for $n_1^j,n_0, \dot{n}_\text{target},n_\text{target}$. Sum over the epoch gives our desired result. Now we will focus on $n_2^j$.
    \begin{align*}
        n_2^j
        & \leq \otil(k^* \max_i \|W_j^\prime(i)\|_2^2 \epsilon_j^{-2}) \\
        & \leq \otil(k^* (\kappa(B_W^\text{source} + k^*))\max_i \|W_j^*(i)\|_2^2 \epsilon_j^{-2}) \\
        & \leq  \otil\left(\sigma^2 (k^*)^2 \min\{\kappa(B_W^\text{source},k^*\}\max_i \|W_j^*(i)\|_2^2\left(kd_X+ \log(1/\delta)) \right) \varepsilon^{-2}\right)
    \end{align*}
    where the first inequality comes from the definition and the second inequality comes from the Lemma~\ref{lem: bound on max W_i}. 

    Finally, by union bounding on the $1-\delta$ from Theorem~\ref{them: excess risk (appendix)} and the event $\calE_\text{offline 1}$ over all the epochs, we get the target result.
    % And 
    % \begin{align*}
    %     n_1^j 
    %     & \leq \otil \left(\epsilon_j^{-\frac{4}{3}}k^{\frac{5}{3}}d_W^{\frac{2}{3}}d_X^\frac{1}{3}
    %     \left( k^{\frac{2}{3}}d_W^{\frac{1}{3}}\underline{\sigma}^{-\frac{4}{3}} + \Bar{\kappa}^2\underline{\sigma}^{-\frac{2}{3}}\right) \right)\\
    %     n_0 
    %     &\leq \otil\left(\Bar{\kappa}^2 \sqrt{k} \left(k^{\frac{3}{2}}d_X\Bar{\kappa}^2 + \frac{d_W^{\frac{3}{2}}}{\underline{\sigma}^2}\right)\right)
    % \end{align*}
    % Similarly, we replace the value of $\epsilon_j$ into 
    % \begin{align*}
    %     & \dot{n}_\text{target}
    %     \geq \epsilon_j^{-\frac{4}{3}}d_X^{-\frac{2}{3}}
    %     \left(k^{-\frac{1}{6}}d_W^{\frac{1}{2}}\underline{\sigma}^{-\frac{4}{3}}
    %     + k^{-\frac{5}{6}}d_W^\frac{1}{6}\Bar{\kappa}^2\underline{\sigma}^{-\frac{1}{3}}\right) \\
    %     & n_\text{target} \geq \epsilon_j^{-2} d_X^{-1}(k^*)^{-1}
    % \end{align*}
    % and finish the proof.
\end{proof}

%% file: appendix-resultComparison.tex
\subsection{More interpretation on results}

\begin{lemma}
\label{lem: uniform case}
    When = the target task is uniformly spread $\|\E_{w_0 \sim \nu_\text{target}}B_W w_0 w_0^\top B_W^\top \| = \frac{1}{k}$ and the task representation is well-conditioned $\sigma_{\min}^2(B_W^\text{source}) = \frac{d_W}{k}$, we have
    \begin{align*}
        \|W_i^*\|_2^2 = \frac{1}{d_W}
    \end{align*}
\end{lemma}
\begin{proof}
Do a svd decomposition on the $B_W^\text{source}$ gives $\sqrt{\frac{d_W}{k}} U_1 V_1^\top $. For any $i$, let $w$ satisfies
\begin{align*}
    \sqrt{\frac{d_W}{k}} U_1 V_1^\top w = \sqrt{\frac{1}{k}} u_i  
\end{align*}
Rearranging the above equality gives
$
    V_1^\top w = \sqrt{\frac{1}{d_W}} U_1^\top u_i.
$
Because $W_i = \argmin \|w\|_2$ satisfy the above constraints, we have
\begin{align*}
    \|W_i^*\|_2^2 = \| V_1^\top w \|_2^2 = \left\|\sqrt{\frac{1}{d_W}} U_1^\top u_i \right\|_2^2 = \frac{1}{d_W}
\end{align*}
\end{proof}

\begin{lemma}
\label{lem: compare L_1 with L_2}
Let 
\begin{align*}
    & \nu_1 = \argmin_{\nu} \|\nu\|_1 \text{ s.t } B_W \nu = B_W w_0 \\
    & \nu_2 = \argmin_{\nu} \|\nu\|_2 \text{ s.t } B_W \nu = B_W w_0
\end{align*}
Then $\|\nu_1\|_1^2 \geq \| \nu_2\|_2^2$.
\end{lemma}
\begin{proof}
    \begin{align*}
        \|\nu_1\|_1^2 \geq \|\nu_1\|_2^2 \geq \|\nu_2\|^2
    \end{align*}
\end{proof}

%% file: appendix-targetAgnostic.tex
\section{Results and analysis for target-agnostic}
\label{sec: target-agnostic (appendix)}

\subsection{Algorithm for target-agnostic}

\begin{algorithm}[ht]
\caption{Target-agnostic algorithm for benign source space}
\label{algo: main-targetAgnostic}
\begin{algorithmic}[1]
\STATE \textbf{Inputs:} Target probability $\delta$, $\bar{\kappa}, \underline{\sigma}$. Some constant $\beta_1, \beta_2, \beta_3$. Others same as Algo.~\ref{algo: main-general}.
\STATE Set $q_0$ as
$
    q_0(e_t) = \frac{1}{d_W}, \forall t \in d_W, \text{and} q_0(w) = 0 \text{ otherwise}
$
\STATE Set $n_0 = \beta_1\beta_1 \Bar{\kappa}^2 \left(k^3d_X\Bar{\kappa}^2 + d_W^{\frac{3}{2}}\underline{\sigma}^{-2}\sqrt{k + \log(1/\delta)}\right)$. Collect $n_0 q_0(w)$ data for each task denoted as $\{Z_w\}_{w | q(w) \neq 0}$
\STATE Update $\hat{B}_X \leftarrow \calO_\text{offline 1}^X(\{Z_w\}_{w | q_0(w) \neq 0 })$ and $\hat{B}_W^\text{source} \leftarrow \calO_\text{offline}^W(\{Z_w\}_{w | q_0(w) \neq 0 }, \hat{B}_X)$ 
\STATE Compute $q_1$ as 
$q_1(v_i) = \frac{1}{k}, \forall i \in k$, and $q_0(w) = 0$ otherwise.
\text{Here} $v_i$ is the $i$-th vector of $V$,  where $U,D,V \leftarrow \text{SVD}(\hat{B}_W^\text{source})$ 
\STATE For any given budget $n_1$, collect $n_1 q_1(w)$ data for each task denoted as $\{Z_w\}_{w | q_1(w) \neq 0}$.
\STATE Update $\hat{B}_X \leftarrow \calO_\text{offline 2}^X(\{Z_w\}_{w | q_1(w) \neq 0})$,  $\hat{B}_W^\text{source} \leftarrow \calO_\text{offline}^W(\{Z_w\}_{w | q_1(w) \neq 0}, \hat{B}_X)$
\STATE \textbf{Return }  $\tilde{B}_X$
\end{algorithmic}
\end{algorithm}

\subsection{Results and analysis}

\begin{theorem}
    In order to get $\textsc{ER}(\hat{B}_X, \nu_\text{target}) \leq \varepsilon^2$, we have w.h.p $1-\delta$, source samples complexity is at most
    \begin{align*}
        \otil \left( \frac{k^2 d_X \trace(B_W \E[w_0 w_0^\top] B_W^\top)}{\sigma_k^2(B_W^\text{source})} \varepsilon^{-2} \right)
        + \otil\left(\Bar{\kappa}^2 \left(k^2 d_X\Bar{\kappa}^2 + \frac{d_W^{\frac{3}{2}}}{\underline{\sigma}^2}\sqrt{k + \log(d_W/\delta)}\right)\right)
    \end{align*}
    as long as,
    \begin{align*}
        & \dot{n}_\text{target} 
        \geq \widetilde{\Omega}\left(\varepsilon^{-\frac{4}{3}} (k^*)^{\frac{2}{3}}
        \left( d_W^{\frac{1}{2}}\underline{\sigma}^{-\frac{4}{3}} 
        + k^{-\frac{2}{3}}d_W^\frac{1}{6}\Bar{\kappa}^2\underline{\sigma}^{-\frac{1}{3}}\right)
        \sqrt{k+ \log(d_W/\delta)}
        \right) \\
        &n_\text{target} \geq \widetilde{\Omega}\left((k+\log(1/\delta))\varepsilon^{-2} \right)
    \end{align*}
\end{theorem}

\begin{proof}
    Again from Section~\ref{sec: reduce to optimal design (appendix) }, we have w.h.p at least $1-\delta$
    \begin{align*}
        \textsc{ER}(\hat{B}_X, \nu_\text{target})
        &\lessapprox \sigma^2 n_\text{target} \left(kd_X + \log(1/\delta) \right) \trace\left(\left( (B_W^\text{source}) \left(\sum_{w \in \calS} n_w w w^\top \right) (B_W^\text{source})^\top \right)^{-1} B_W \left(\E_{\nu_\text{target} } w w^T \right) B_W^\top \right)\\
            &\quad +\frac{k+ \log(1/\delta)}{n_\text{target}}
    \end{align*}
    then by using similar steps in Section~\ref{sec: bound decomposition}, we have
    \begin{align*}
        &  \trace\left(\left( (B_W^\text{source}) \left(\sum_{w \in \calS} n_w w w^\top \right) (B_W^\text{source})^\top \right)^{-1} B_W \left(\E_{\nu_\text{target} } w w^T \right) B_W^\top \right)\\
        & \leq \|\square\| \trace(B_W \E[w_0 w_0^\top] B_W^\top) \\
        & \leq \frac{k}{n_1 \sigma_k^2(B_W^\text{source})}\trace(B_W \E[w_0 w_0^\top] B_W^\top)
    \end{align*}
    and therefore, 
    \begin{align*}
        \textsc{ER}(\hat{B}_X, \nu_\text{target})
        \leq \otil\left(\frac{k^2d_X}{n_1 \sigma_k^2(B_W^\text{source})}\trace(B_W \E[w_0 w_0^\top] B_W^\top)\right)
            + \otil\left(\frac{k+ \log(1/\delta)}{n_\text{target}}\right)
    \end{align*}
    Rearranging the inequality gives the final bound.
\end{proof}

\subsection{Compare to previous passive learning and the target-aware one}

Again we want to compare this result with the previous one.
\paragraph{Comparison with passive learning.} 
We first consider the cases in their paper that the target task is uniformly spread $\|\E_{w_0 \sim \nu_\text{target}}B_W w_0 w_0^\top B_W^\top \| = \frac{1}{k}$. (See detailed setting in Section~\ref{sec: theory})
\begin{itemize}[leftmargin=*]
    \item When the task representation is well-conditioned $\sigma_{\min}^2(B_W^\text{source}) = \frac{d_W}{k}$. We have a passive one as $ \otil(k d_X \varepsilon^{-2})$ while the target-agnostic active one $\otil(kd_X \frac{k^2}{d_W} \varepsilon^{-2})$. 
    \item Otherwise, we consider the extreme case that  $\sigma_{\min}^2(B_W^\text{source}) = 1$. We have passive one  $\otil(d_X d_W\varepsilon^{-2})$ while the target-agnostic active one $\otil(k^2 d_X \varepsilon^{-2})$. Note this is better than the  $\otil(k^3 d_X \varepsilon^{-2})$ in the target-aware case.
\end{itemize}
These two results indicate that when the targets are uniformly spread, target-agnostic AL can perform even better than the target-aware. But we want to emphasize that whether it is uniformly spread or not is unknown to the learner. Even $\kappa\left(\E_{w_0 \sim \nu_\text{target}}[w_0 w_0^\top]\right) = 1$ can leads to ill-conditioned $B_W\E_{w_0 \sim \nu_\text{target}}[w_0 w_0^\top]B_W^\top$. 

We then consider the single target $w_0$ case.
\begin{itemize}[leftmargin=*]
    \item With well-conditioned $B_W$, the passive one now has sample complexity $\order(k^2 d_X \varepsilon^{-2})$ while the active gives a strictly improvement $\order(\frac{k^3 d_X}{d_W}  \varepsilon^{-2})$.
    \item With ill-conditioned $B_W$ where $\sigma_{\min}(B_W) = 1$ and $\max_i \|W_i^*\| = 1$, that is, only a particular direction in source space contributes to the target. The Passive one now has sample complexity $\order(k d_X d_W \varepsilon^{-2})$ while our target-agnostic active one has $k^2d_X \varepsilon^{-2}$.
\end{itemize}
These two results indicate that the target-agnostic approach gives a worse bound when the targets are not well-spread, which meets our intuition since the target-agnostic tends to learn uniformly well over all the levels. But it can still perform better than the passive one under the discrete case, which again indicates the necessity of considering the continuous setting.

\paragraph{Save task number.} Again when ignoring the short-term  initial warm-up stage, we only require maintaining $\otil(k)$ number of source tasks.

%% file: appendix-theoryLimitations.tex
\section{Limitations from the theoretical perspective}

Here we list some open problems from the theoretic perspective. We first list some room for improvements under the current setting
\begin{itemize}
    \item \textbf{Not adaptive to noise $\sigma$: } From Section~\ref{sec: training oracle 1 (appendix)}, we get $\sin(\hat{B}_X, B_X)$ scales with the noise $\sigma$, which suggests less sample number $n_0,n_1$ is requires to get a proper estimation of $B_X$. In our algorithm, however, we directly treat $\sigma = \Theta(1)$ and therefore may result in unnecessary exploration.
    \item \text{Bound dependence on $\min\{\kappa^2(B_X^\text{source}),k\}$: } This extra dependence comes from the instability (or non-uniqueness) of eigendecomposition. For example, when $\E_{\nu_\text{target}}[B_W w_0 w_0^\top B_W^\top] = \frac{1}{k}I_d$, there are infinite number of eigenvector sets. On the other hand, given a fixed $B_W^\text{source}$, current methods of obtaining $W'$ are highly sensitive to the eigenvector sets from the target. A direct method is of course constructing a confidence bound around the estimated $\hat{B}_W^\text{target}$ and finding the best $W'$ under such set. But this method is inefficient. Whether there exists some efficient method, like a regularized optimization, remains to be explored in the future.
    \item \textbf{Require prior knowledge of $\Bar{\kappa}, \underline{\sigma}$: } Finally, can we estimate and use those parameters during the training remains to be open?
\end{itemize}

Besides that specific problem, it is always meaningful to extend this setting into more complicated geometries and non-linear/non-realizable models. Specifically,
\begin{itemize}
    \item \textbf{More complicated geometry.} One open problem is to get guarantees when $\calW_\text{source}, \calW_\text{target}$ is no longer a unit ball. (e.g., eclipse). Another problem is, instead of considering the geometry of $\calW$, we should consider the geometry of $\psi_W(\calW)$. 
    \item \textbf{Nonlinear models.} Consider nonlinear $\phi_X, \phi_W$ is always challenging. In \cite{du2021fewshot, tripuraneni2020theory}, they provide some guarantees under the passive by using kernel methods or considering a general model. Can we extend this to the active setting?
    \item \textbf{Non-realizable model.} Like many representation learning papers, we assume the existence of a shared representation, which suggests more source tasks always help. In practice, however, such representation may not exist or is more over-complicated than the candidate models we assume. Under such a misspecification setting, choosing more tasks may lead to negative transfer as shown in Figure~\ref{fig: pendulum 1} in the experiments. Can we get any theoretical guarantees under such a non-realizable setting?
\end{itemize}

%% file: appendix-experiment.tex
\section{Experiment details}
\label{sec: experiments (appendix)}

Here we provide detailed settings of three experiments -- synthetic data, pendulum simulator, and the real-world drone dataset, as well as more experimental results as supplementary. All the experiments follow a general framework proposed in Section~\ref{sec: framework} with different implementation approaches according to different settings, which we will specify in each section below. Note that in all these experiments, we only focus on a single target. % Experiments on multiple targets remain to be down in the future.

\subsection{Synthetic data}
\label{sec: synthetic data (appendix)}

\subsubsection{Settings}
\begin{table}[ht]
\begin{center}
\begin{tabular}{|l|l|l|l|}
\hline
                     & bilinear     & nonlinear $\psi_X$ & nonlinear $\phi_X$ \\ \hline
target number        & 800, 8000    & 800, 8000          & 800, 8000               \\ \hline
$d_X$                & 200          & 10                 & 20                 \\ \hline
$d_{\psi_X}$         & 200          & 200                & 20                 \\ \hline
$d_W$                & 80           & 80                 & 80                 \\ \hline
$k$                  & 4            & 4                  & 4                  \\ \hline
$\phi$ structure     & random matrix     & random matrix           & MLP with layers {[}20, 20, 4{]}    \\ \hline
inputs distribution  & $\calN(0,I)$ & $\calN(0,I)$       & $\calN(0,I)$       \\ \hline
label noise variance & 1            & 1                  & 1                  \\ \hline
\end{tabular}
\end{center}
\caption{Model used to generate the synthetic data.}
\label{table: Synth data generation}
\end{table}
\vspace{-15px}
\paragraph{Data generation} We show the model and corresponding parameters used to generate the synthetic data in Table.~\ref{table: Synth data generation}. Some additional details include, 1) When generating random matrix $B_X$ for bi-linear and unknown non-linear $\psi_X$, we tried different seeds (denoted as \textit{embed\_matrix\_seed} in the codes) and deliberately make the matrix ill-conditioned (so $\kappa(B_W)$ is large). Because most of them behave similarly so we only present partial results here. 2) When generating random MLP for nonlinear $\phi_X$, we only use the unbiased linear layer and ReLU layers.

In the main paper Table~\ref{table: synthetic}, we use target number = 8000 cases to show more contrast.

The nonlinear Fourier feature kernel $\psi_X$ is defined as 
$\psi_X(x) = \cos(A x + B)$, where $ A \in \fR^{d_{\psi_X} \times d_X}, B \in \fR^{d_{\psi_X}} $ and each entry of $A,B$ is i.i.d. Gaussian.
% \begin{align*}
%     \psi_X(x) = \cos(A x + B)
%     \text{ where, } A \in \fR^{d_{\psi_X} \times d_X}, B \in \fR^{d_{\psi_X}} \text{and each entry of A,B is i.i.d. gaussian}
% \end{align*}
% \guanya{typos? move ``where ...'' outside}

\paragraph{Training models and optimizer} Here we state the details of the model used during the learning, which might be different from the model used to generate the data. Specifically, for the bi-linear and unknown non-linear $\psi_X$, we use the exact $\fR^{d_{\psi_X} \times k}$ matrix structure as stated in the theorem. For the  nonlinear $\phi_X$, we use a slightly larger MLP with layers {[}20, 20, 20, 4{]} compared to the model used to generate the data to further test the adaptivity of our algorithm since the exact underlying structure of MLP is usually unknown in reality. As for the joint training approach, we use Adam with $lr=0.1$ for the bi-linear and unknown non-linear $\psi_X$, and SGD with $lr=0.1$ for nonlinear $\phi_X$ as the optimizer (The learning rate is large because this is an easy-to-learn synthetic data) We mixed all the target and source data and do joint GD-based methods on them. Notice that the goal for those experiments is not to achieve the SOTA but to have a fair comparison. So all those hyper-parameters are reasonable but not carefully fine-tuned.

\paragraph{Detailed implementation for AL strategy} Both the input space $\calX$ and the task space $\calW$ of synthetic data lie perfectly in a ball and the underlying model is linear in terms of $w$. Therefore, we can use the almost similar algorithms as proposed in Algo~\ref{algo: main} for target-aware and Algo.~\ref{algo: main-targetAgnostic}. We slightly adjust parameter dependence on $d_X, d_W, k$ but the general scaling between different stages in each epoch remains the same. Another difference is that, instead of using the MLLAM as specified in Section~\ref{sec: training oracle 1 (appendix)}, we do a joint-GD since the implementation of MLLAM in a non-idealistic setting (nonlinear $\phi_X, \psi_X$ is unclear and challenging.)

\paragraph{Metrics}
We consider the worst-case distance between ground truth and estimate columns space $U, \hat{U}$ as 
$
    \text{dis}(U, \hat{U}) = \min_{u} \|u_i^\top \hat{U}\|_2.
$
Such distance will be used in both computing the similarity between ground truth and estimated input space $B_W, \hat{B}_W$. In addition, it will also be used in measuring the change of $q_2$ across each epoch so we can save task numbers by maintaining the same $q_2$ as long as the change is small, which we will specify in the next paragraph.

\paragraph{Saving task number approach.} In addition to the comparison between target-agnostic AL, target-aware AL, and the passive, we also consider the saveTask case, where we reduce the number of times recomputing the $q_1$. Specifically, we denote $W_{j-1},W_{j} \in \fR^{d_W times k}$ as the exploration source tasks in the previous and current epoch. And only switch to the new target-agnostic exploration set when 
$
    \text{dis}(\text{rowSpace}(B_{j-1}), \text{rowSpace}(B_{j})) \leq 0.8 
$
where $0.8$ is some heuristic threshold parameter.

\subsubsection{Results}
\begin{figure}[H]
    \centering
    \includegraphics[scale = 0.42]{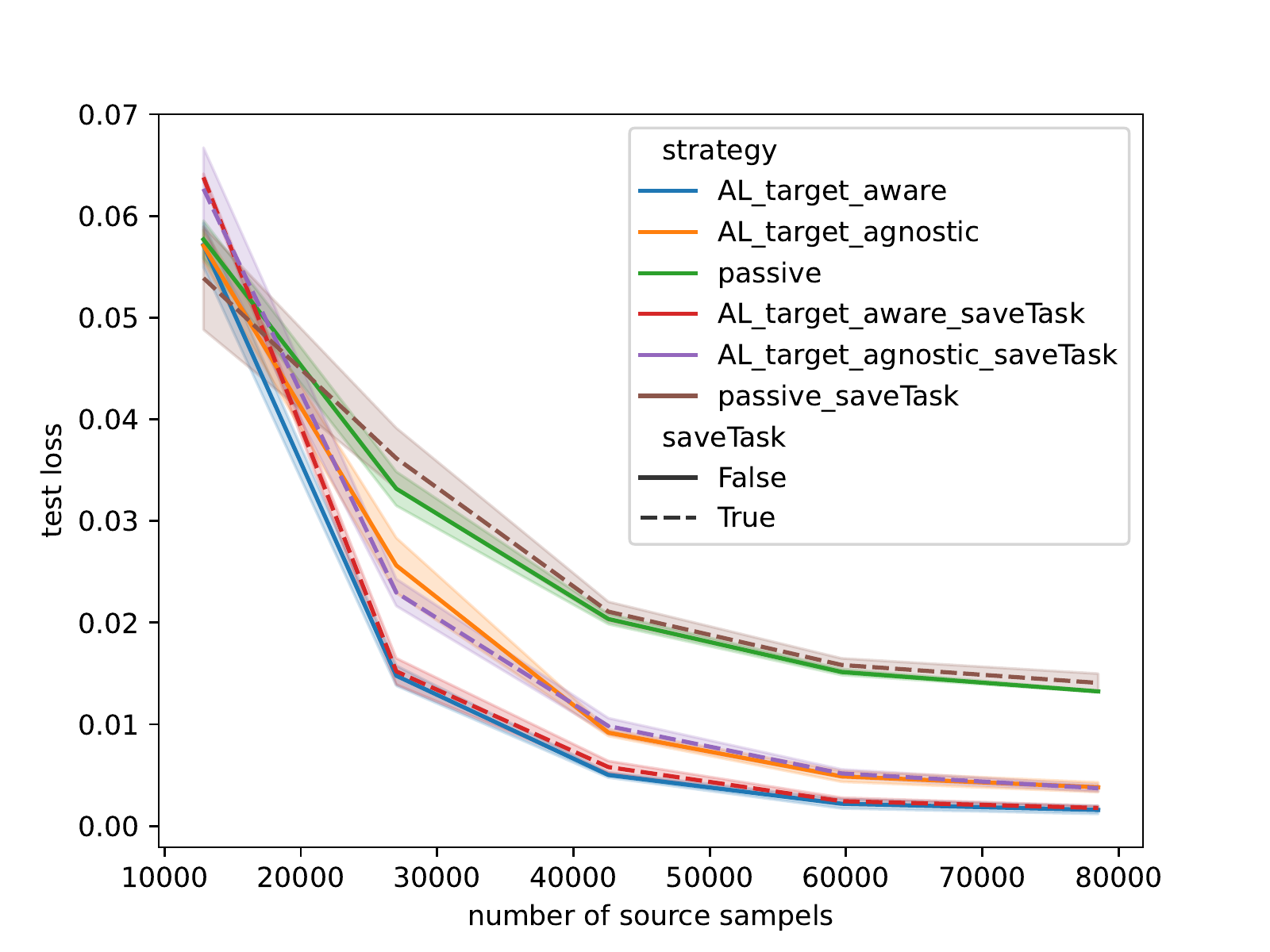}
    \includegraphics[scale = 0.42]{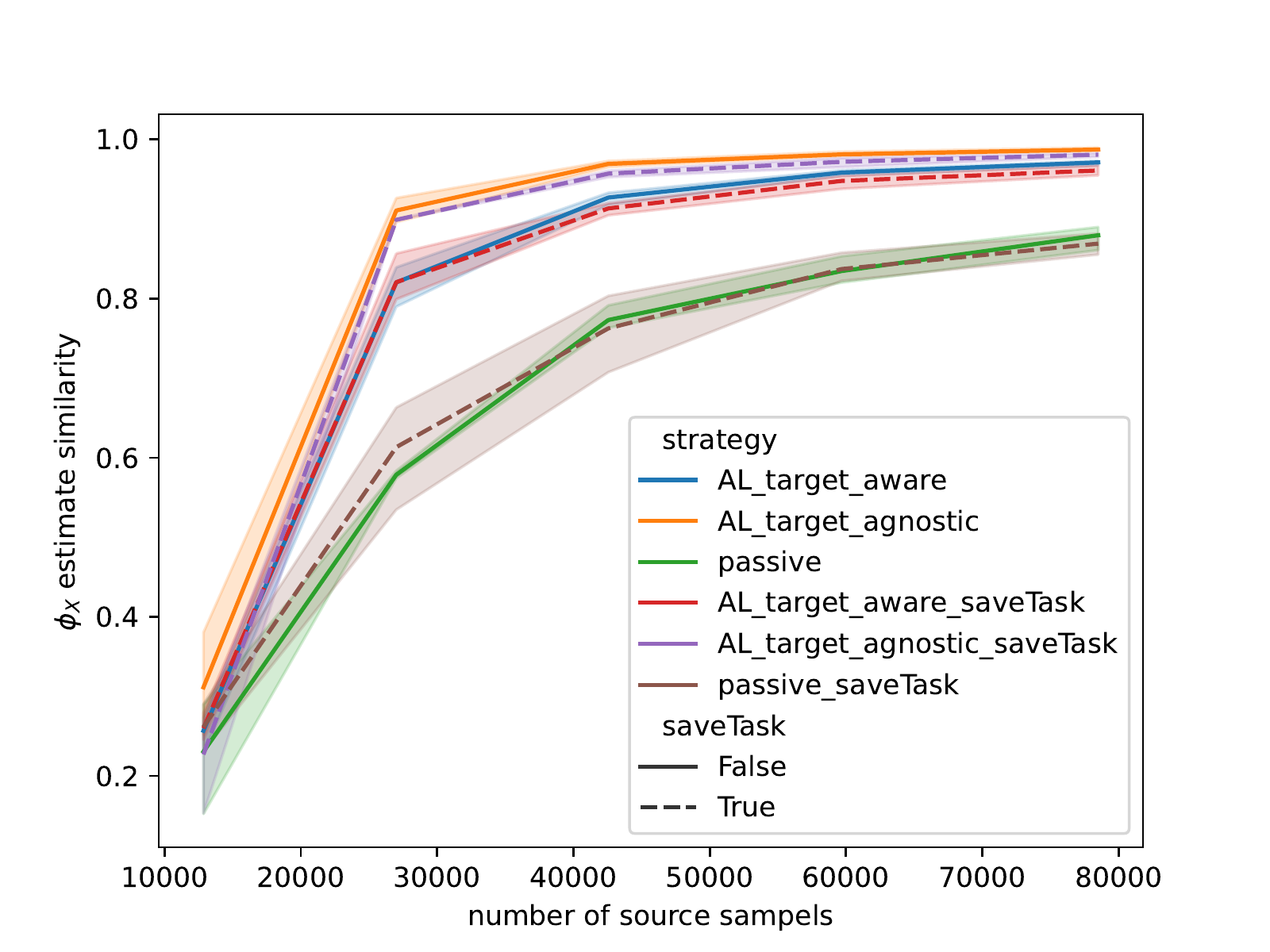}
    \includegraphics[scale = 0.42]{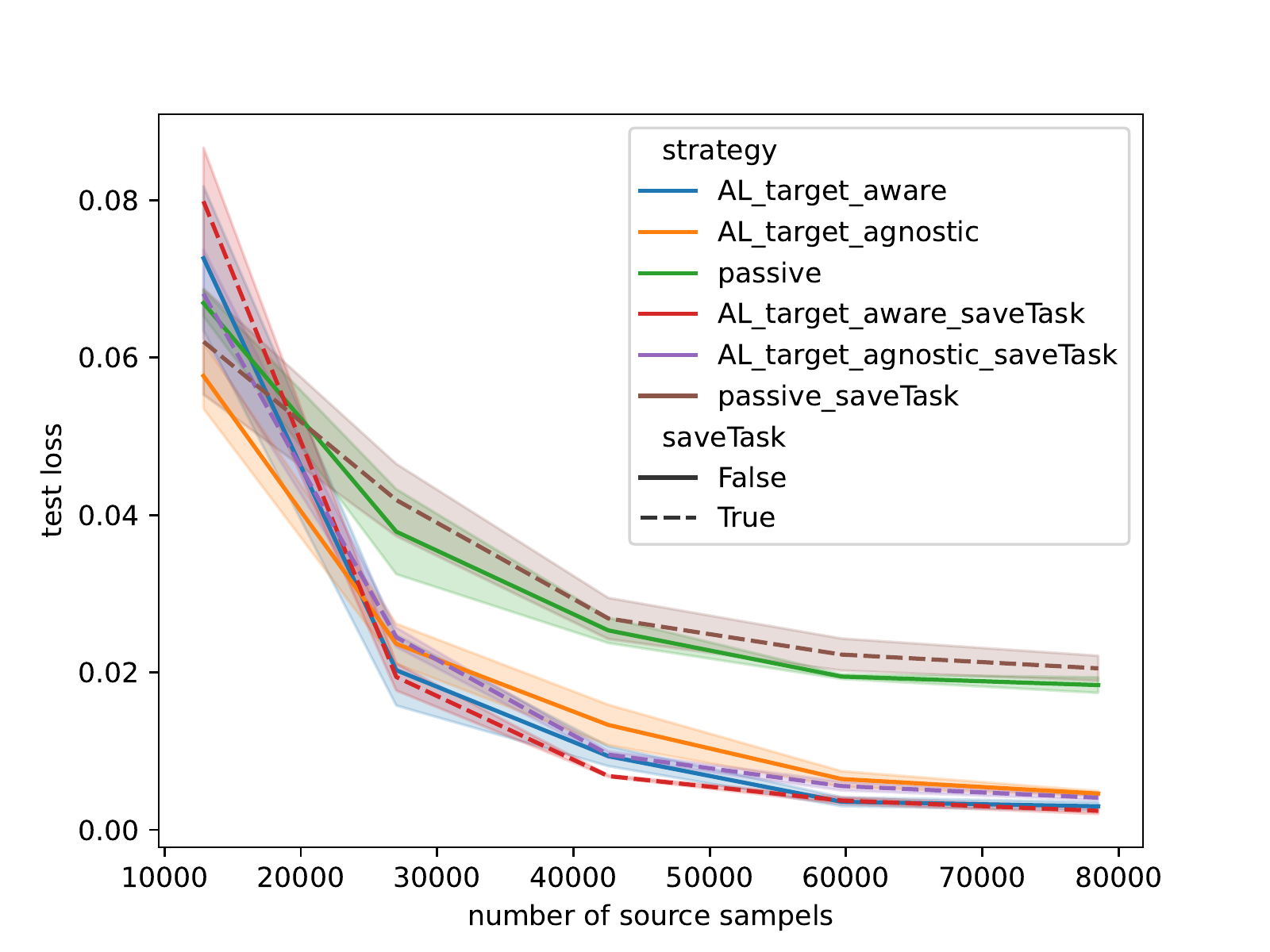}
    \includegraphics[scale = 0.42]{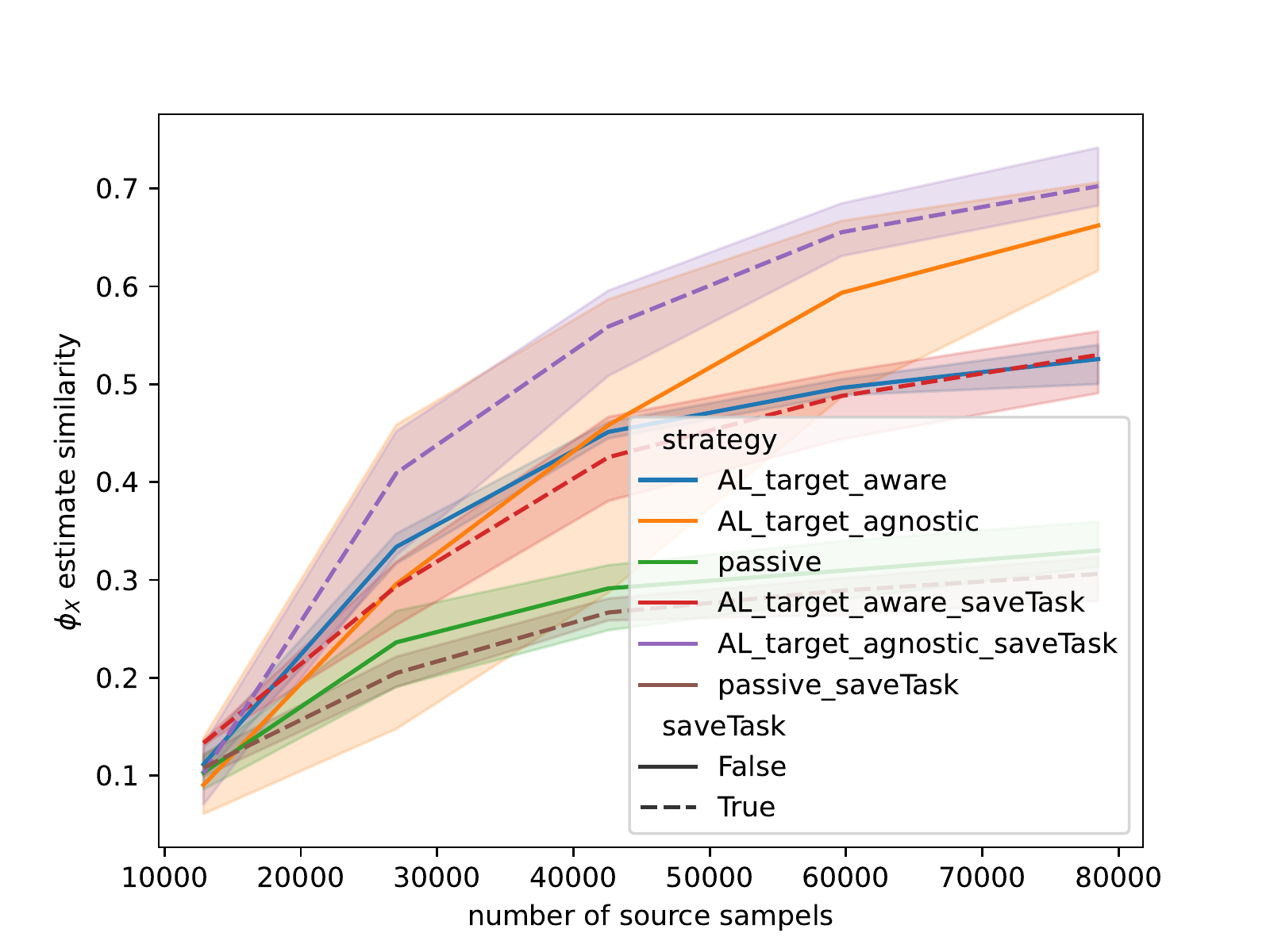}
    \includegraphics[scale = 0.42]{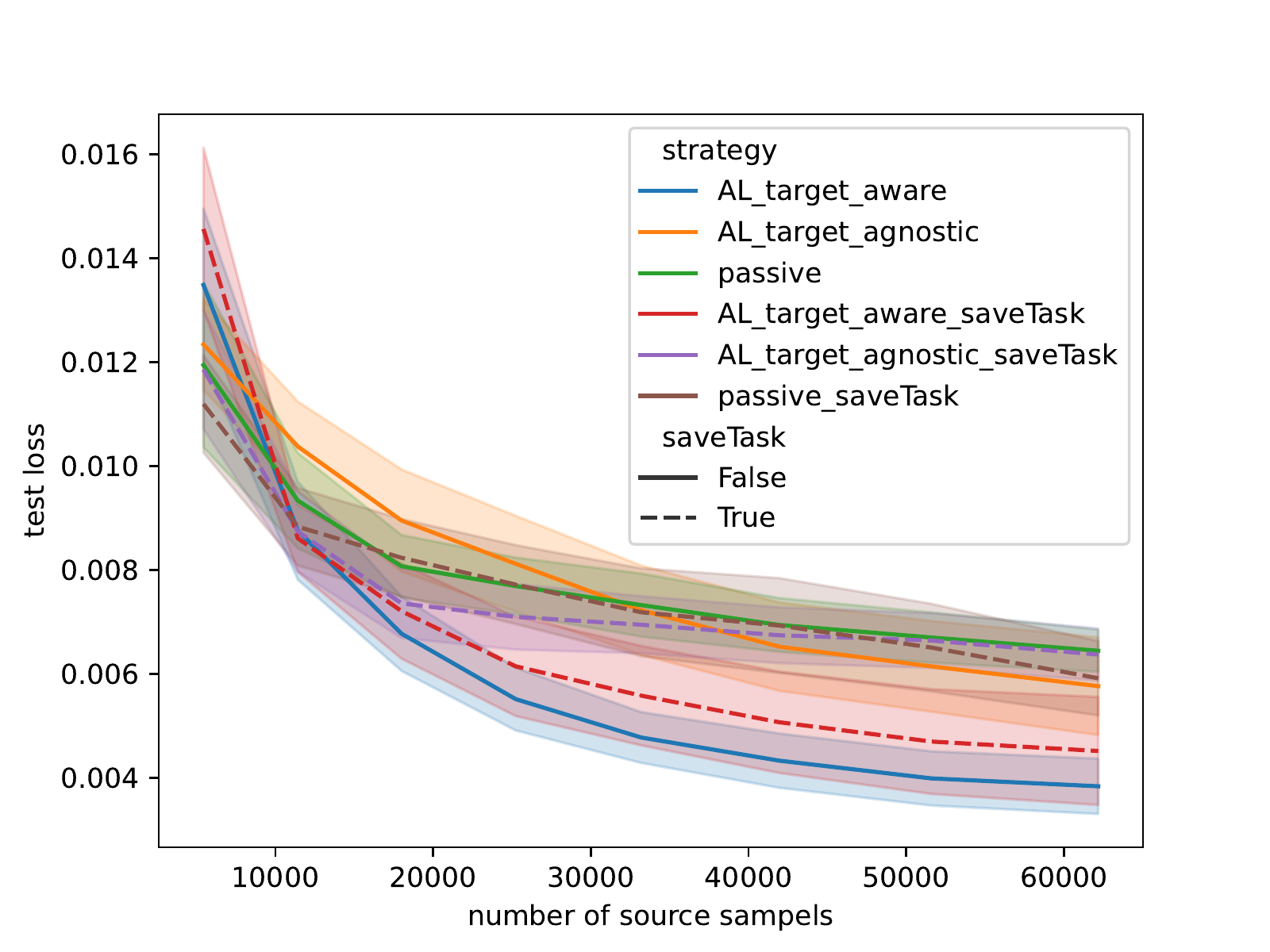}
    \caption{
    % \textbf{Results on synthetic data} Left side presents the test loss. The right side presents the similarity between the column space of the ground truth $\phi_X$ and the estimated $\hat{\phi}_X$ for the first two linear $\phi_X$ cases and presents the accuracy on estimating Target-relevant source for nonlinear $\phi_X$. \textbf{Top:} Results of the linear model. The target-aware AL gets the lowest test loss while the passive gets the highest. Note that in the linear case, it is possible to improve the passive one by using a more advanced algorithm (e.g. algorithms in \cite{thekumparampil2021sample, tripuraneni2020theory}). But our purpose is to have a universal comparison over various instantiations. For the estimation on $\phi$, the target-agnostic AL gets the best estimation which aligns with our design intuition that target-agnostic AL should have a universal good estimation in all directions. It is a little surprising to us that the passive one performs worst. We conjecture the reason that the GD-based oracle is not that good for joint-task training and should again have better performance when using \cite{thekumparampil2021sample, tripuraneni2020theory}. \textbf{Middle:} Results of the nonlinear kernel. \textbf{Bottom: } Result of non-linear representation.
    % }
    \textbf{Results on synthetic data with 8000 target sample} Left side presents the test loss and the right side presents the similarity between the column space of the ground truth $\phi_X$ and the estimated $\hat{\phi}_X$. Notice that how to measure the similarity on neural networks is unclear so we skip this result.
    \textbf{Top and middle:} Results of the nonlinear kernel. The target-aware AL gets the lowest test loss while the passive gets the highest. In terms of saveTask, we notice that reducing task switch number does not affect the performance a lot. From the left figure, the target-agnostic AL gets the best estimation which aligns with our design intuition that target-agnostic AL should have a universal good estimation in all directions. It is a little surprising to us that the passive one performs worst. We conjecture the reason that the GD-based oracle is not that good for joint-task training and should again have better performance when using \cite{thekumparampil2021sample, tripuraneni2020theory}. \textbf{Bottom: } Result of non-linear representation. Here we notice that the saving task strategy leads to slightly worse performance. While the target-aware AL still gives the worst test loss, the difference between passive and target-agnostic AL is small due to the complexity of the shallow net. 
    } 
    \label{fig:synthetic 1}
\end{figure}

\begin{figure}[H]
    \centering
    \includegraphics[scale = 0.42]{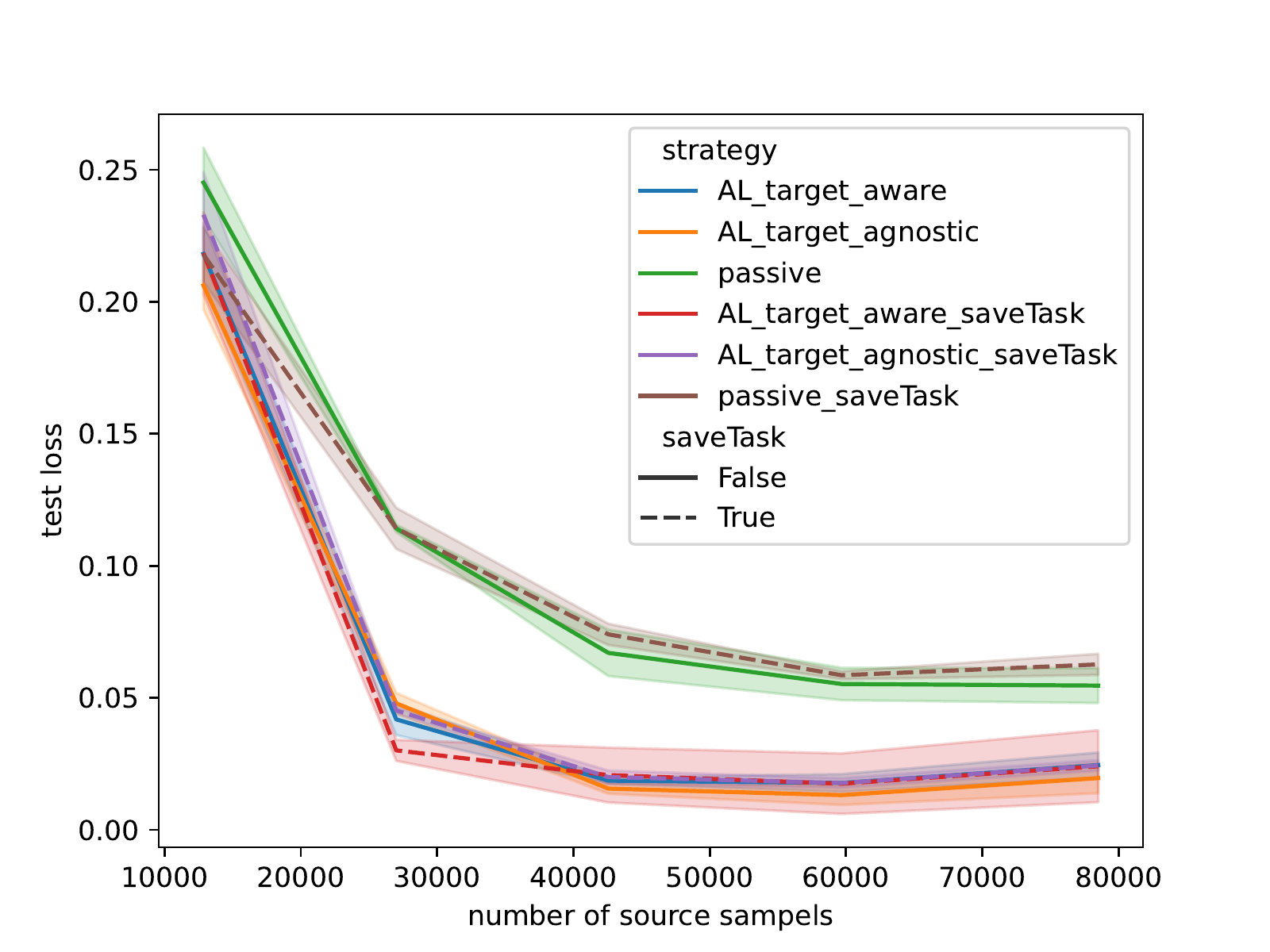}
    \includegraphics[scale = 0.42]{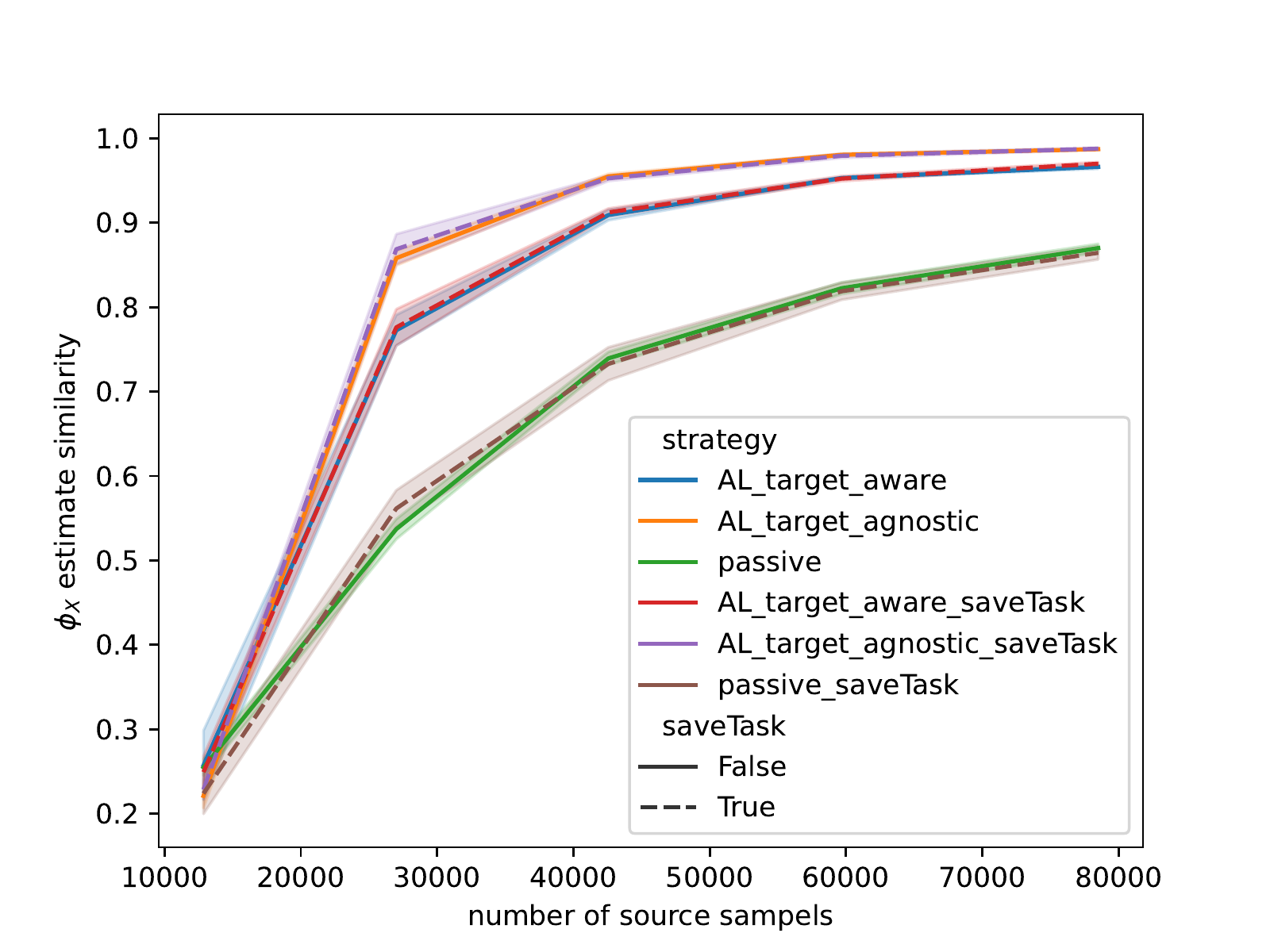}
    \includegraphics[scale = 0.42]{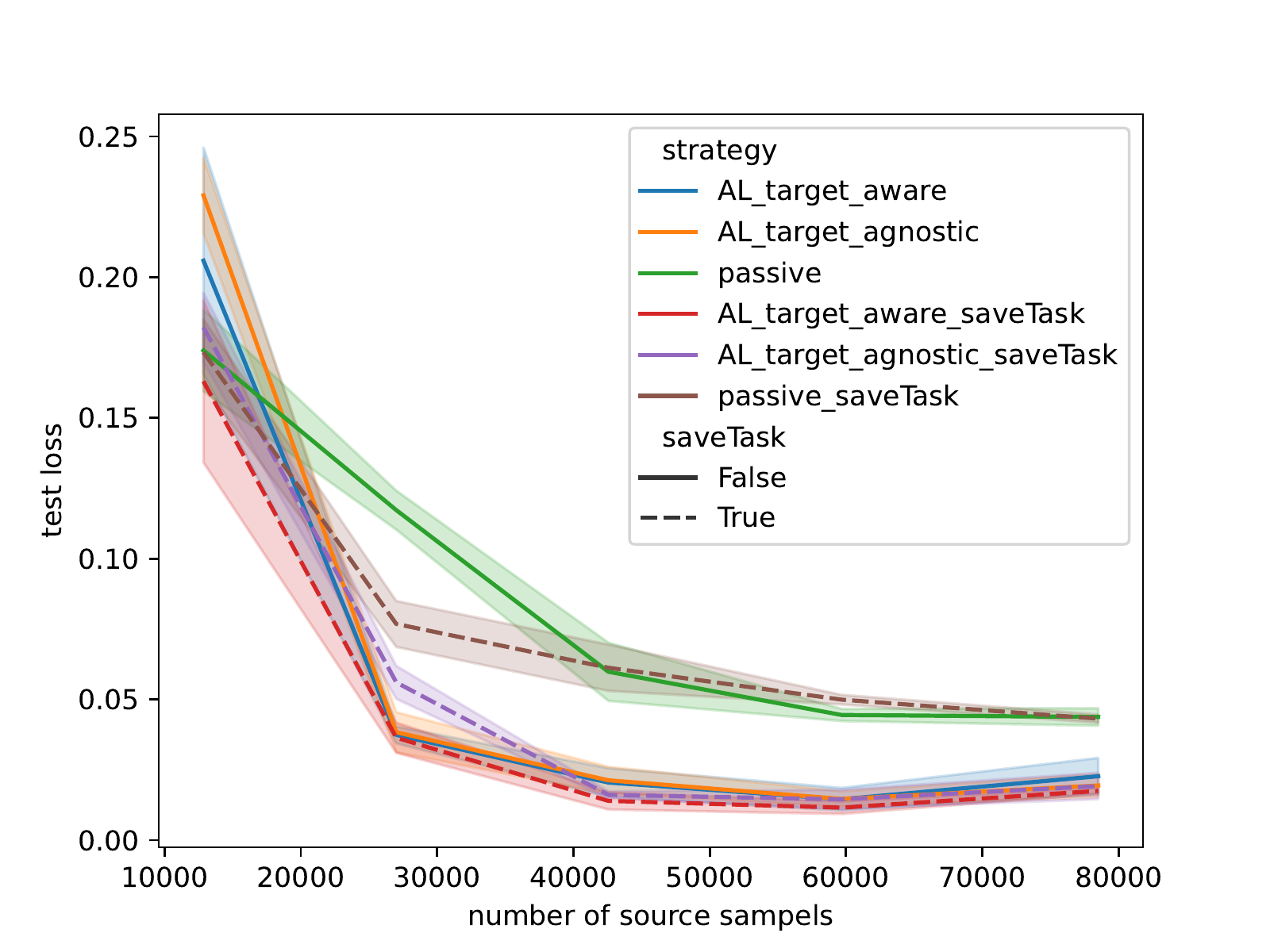}
    \includegraphics[scale = 0.42]{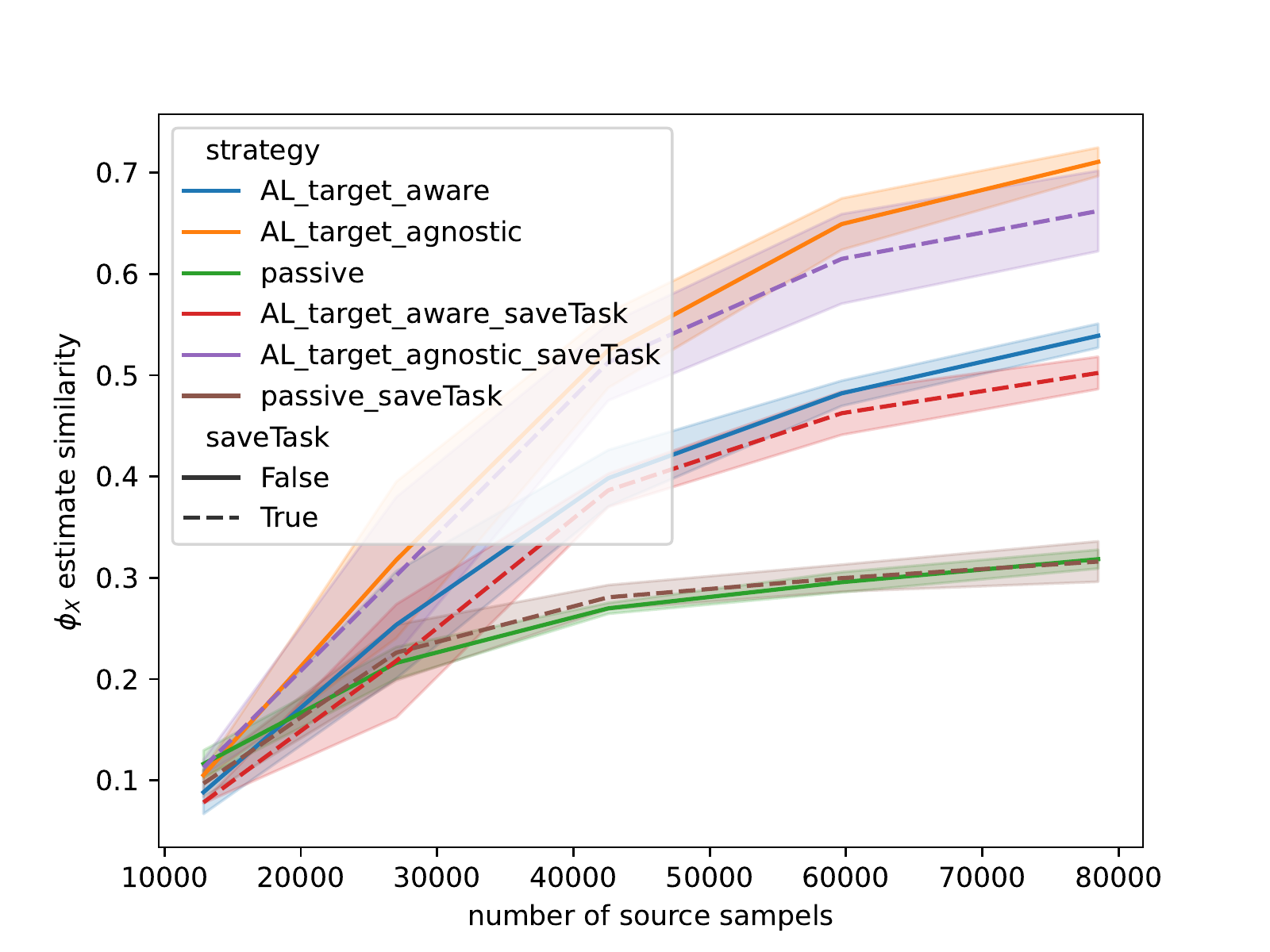}
    \includegraphics[scale = 0.65]{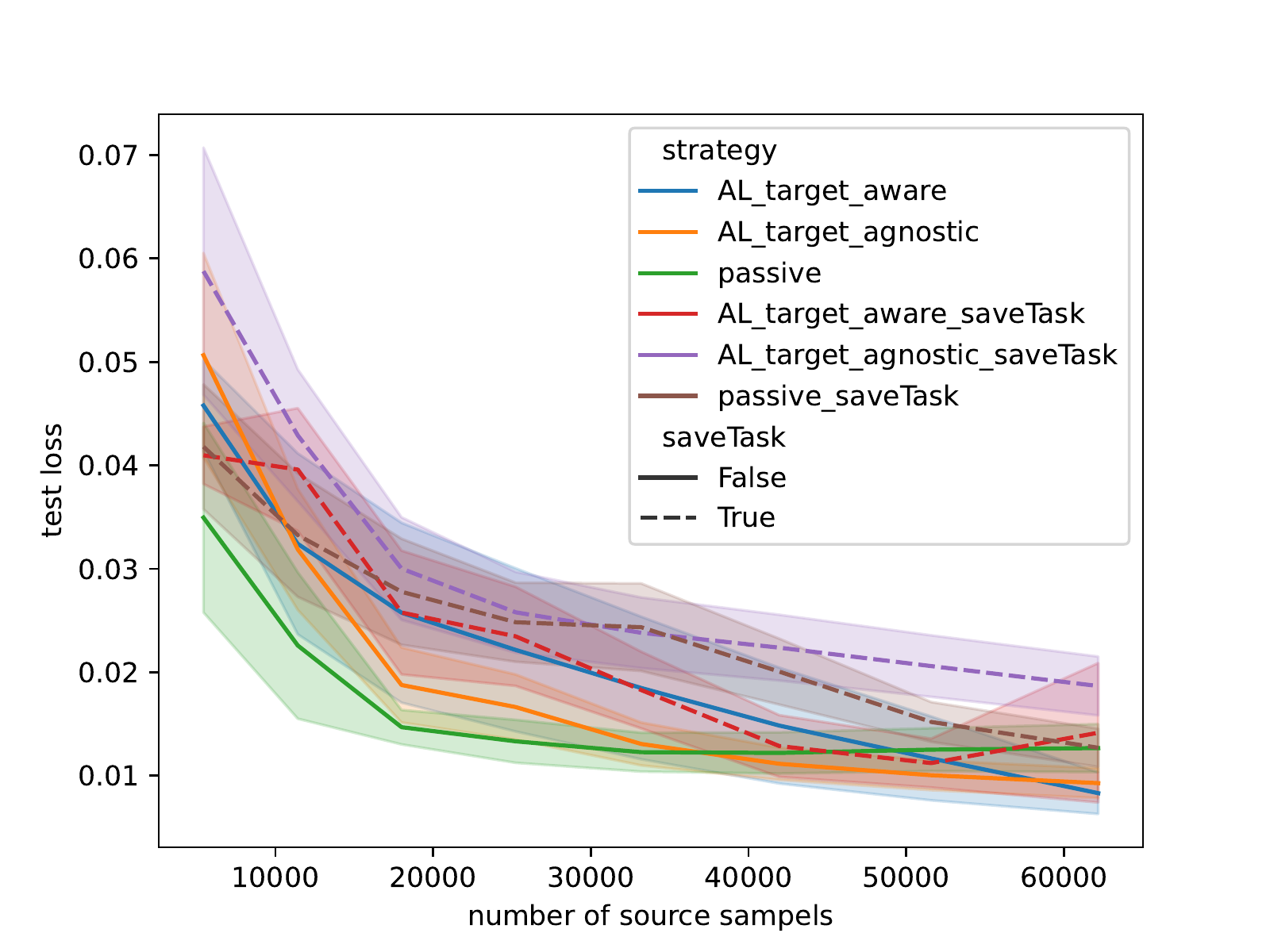}
    \caption{
    \textbf{Results on synthetic data with 800 target sample} \textbf{Top and middle:} The bilinear and nonlinear $\psi_X$ case gives a similar performance as before. \textbf{Bottom: } For $\phi_W$ as a neural net, we notice here the AL does not show an advantage until the very end where the passive stops decreasing. This may suggest for nonlinear representation, more target data may be needed for a beneficial source selection compared to the bilinear $\phi$.
    } 
    \label{fig:synthetic 2}
\end{figure}

\newpage
\subsection{Pendulum simulator}
\label{sec: pendulum (appendix)}

\subsubsection{Settings}
\paragraph{Data generation} 
We consider the following continuous-time pendulum dynamics model adopted from~\cite{shi2021meta}:
\begin{equation*}
    m l^2 \ddot{\theta}-m l \hat{g} \sin \theta=u+f(\theta, \dot{\theta}, w)
\end{equation*}
where $\theta,\dot{\theta},\ddot{\theta},u$ are angle, angular velocity, angular acceleration, and control, $m,l,\hat{g}$ are mass, pole length, and the gravity estimation, and finally, $f$ is the unknown residual dynamics term to be learned with $w$ the environment parameter. The ground truth $f$ is given by
\begin{align*}
    & F = \|R\|_2^2 \cdot R, R=c-\left[\begin{array}{c}
        l \dot{\theta} \cos \theta \\
        -l \dot{\theta} \sin \theta
        \end{array}\right] \\
    f(\theta, \dot{\theta}, w)
    & =\underbrace{\vec{l} \times F}_{\text {air drag }}-\underbrace{\alpha_1 \dot{\theta} - \alpha_2\dot{\theta}|\dot{\theta}|}_{\text {damping }}+\underbrace{m l(g-\hat{g}) \sin \theta}_{\text {gravity mismatch }} \\
    & w = [c_x, c_y, \alpha_1,  \alpha_2, \hat{g}, 0 \text{ or } 1]
\end{align*}
where $c=[c_x,c_y]$ is external wind, $\alpha_1,\alpha_2$ are damping coefficients and $g$ is the true gravity.

% \guanya{Rewrote a bit. Isn't $w$ 7-d?}

We let $x=[\theta,\dot{\theta}]$ denote the input to $f$. Notice here the last element of $w$ is a dummy feature. For the source tasks, we always have $w[6] = 0$ since all the source parameters are known. For the single target task, we have $w_\text{actual\_target}$ to generate the data, so $w_\text{actual\_target}[6] = 0$. But the learner only observes the $w_\text{target}=[0,0,0,0,0,1]$, which indicates the unknown environment of the target. In the simulator, we collect data using a stochastic policy to approximate i.i.d. data distribution. 

It is easy to see that $f$ is highly nonlinear regarding $x,w$. Therefore we use the known nonlinear feature operator $\psi$ to make it close to the linear model with some misspecification: 
\begin{align*}
    & \psi_X \text{ is the Fourier feature kernel which has been defined in the synthetic data section}\\
    & \psi_W(w) = [l_x, l_y, g, \alpha_1, \alpha_2, CxCy, Cx^2, Cx^2C_y, C_x^3, Cy^2, Cy^2C_x, C_y^3, 0 \text{ or } 1]
\end{align*}

Other common parameters are specified in Table.~\ref{table: pendulum model}.
\begin{table}[ht]
\begin{center}
\begin{tabular}{|l|l|l|l|l|l|l|l|l|}
\hline
target number & $d_X$ & $d_{\psi_X}$ & $d_W$ & $d_{\psi_W}$ & $k$ & $\phi$ structure & inputs distribution & label noise variance \\ \hline
4000          & 2     & 60           & 13    & 6            & 8   & bilinear         & (See details above) & 0.5                  \\ \hline
\end{tabular}
\end{center}
\caption{Model parameters for pendulum simulator.}
\label{table: pendulum model}
\end{table}

\paragraph{Training models and optimizer} We again use the bilinear model. For the training methods, we first do joint-GD as before using AdamW with $lr = 0.01, wd=0.05, \text{batch\_size} = 512$. Then after joint training, we freeze the $\phi_X$ parts and only trained on the targets to get the non-shared embed $\phi_W(w_\text{target})$. Another modification is that, since we are in the misspecification setting, using data collected in stage 3 might amplify the errors when estimating the target-related source. To tackle this negative transfer learning, we only use the data collect from stage 2 in previous the epochs to compute $q_3$. While in the synthetic data, all data, including one from stage 3, collected in previous epochs can be used. 

\paragraph{Detailed implementation for AL strategy} The input space $\calX$ and task space $\calW$ of this pendulum data again lie perfectly in a ball after some normalization. Nevertheless, the underlying model is no longer linear in terms of $w$, which adds some extra difficulties to the optimal design on $w$. Here we use the adaptive sampling methods mentioned in the main paper. That is, we will iteratively sample from $\calW_\text{source}$ and find the ones that minimize follows.
\begin{align*}
    \min_{\{w_i\} \in \calW_\text{source}} \|\hat{B}_{W,j}^\text{source} \psi_W(w_i)  - u_i\sqrt{\lambda_i}\|
\end{align*}
where $ u_i\sqrt{\lambda_i}$ is defined in line~\ref{line: eigdecompo}. Other parts of the algorithm can still be implemented as in the synthetic data section.

\paragraph{Using learned $f$ for control} To show that a better dynamics model can transfer to better control performance, we deploy the following nonlinear controller $\pi(x,\hat{f})$ as a function of $\hat{f}$ (prediction result of $f$ in the target task): 
\begin{equation*}
    u = -ml\hat{g}\sin\theta - \hat{f}(\theta,\dot{\theta}) - ml^2(K_P\theta + K_D\dot{\theta})
\end{equation*}
Here we focus on the regulation task, i.e., $\|x\| \rightarrow 0$. It is worth noting that the above controller is guaranteed to be exponentially stable: $\|x\| \rightarrow \eta$ exponentially fast, where $\eta$ is an error ball whose size is proportional to $\|f-\hat{f}\|_\infty$.

\subsubsection{Results}
In the main paper, we use the unobservable actual target as $[0, 0, 1, 0.5, 0, 0]$. Here we give more results in Figure.~\ref{fig: pendulum 1}

\begin{figure}[ht]
    \centering
    \includegraphics[scale = 0.37]{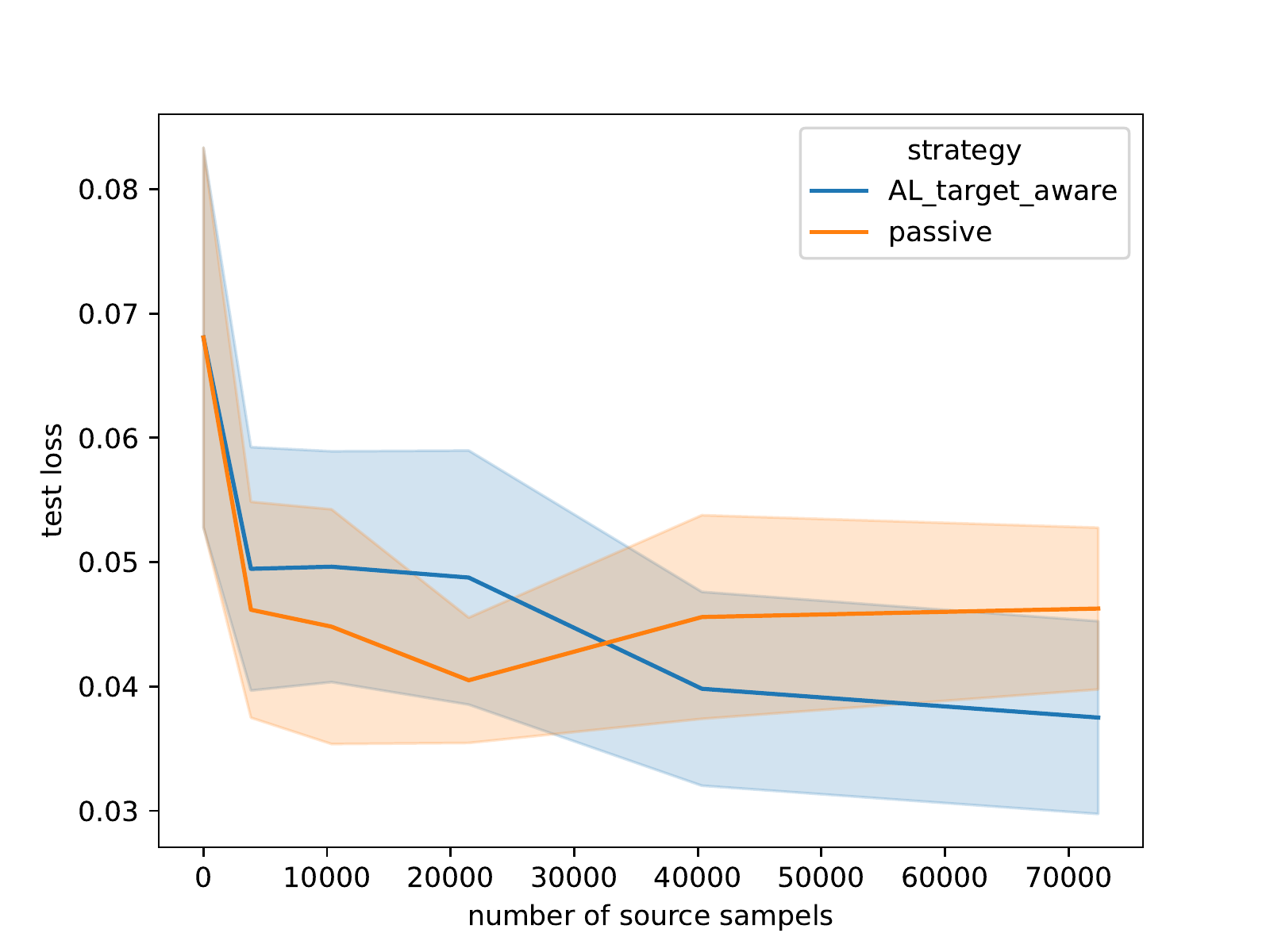}
    \includegraphics[scale = 0.37]{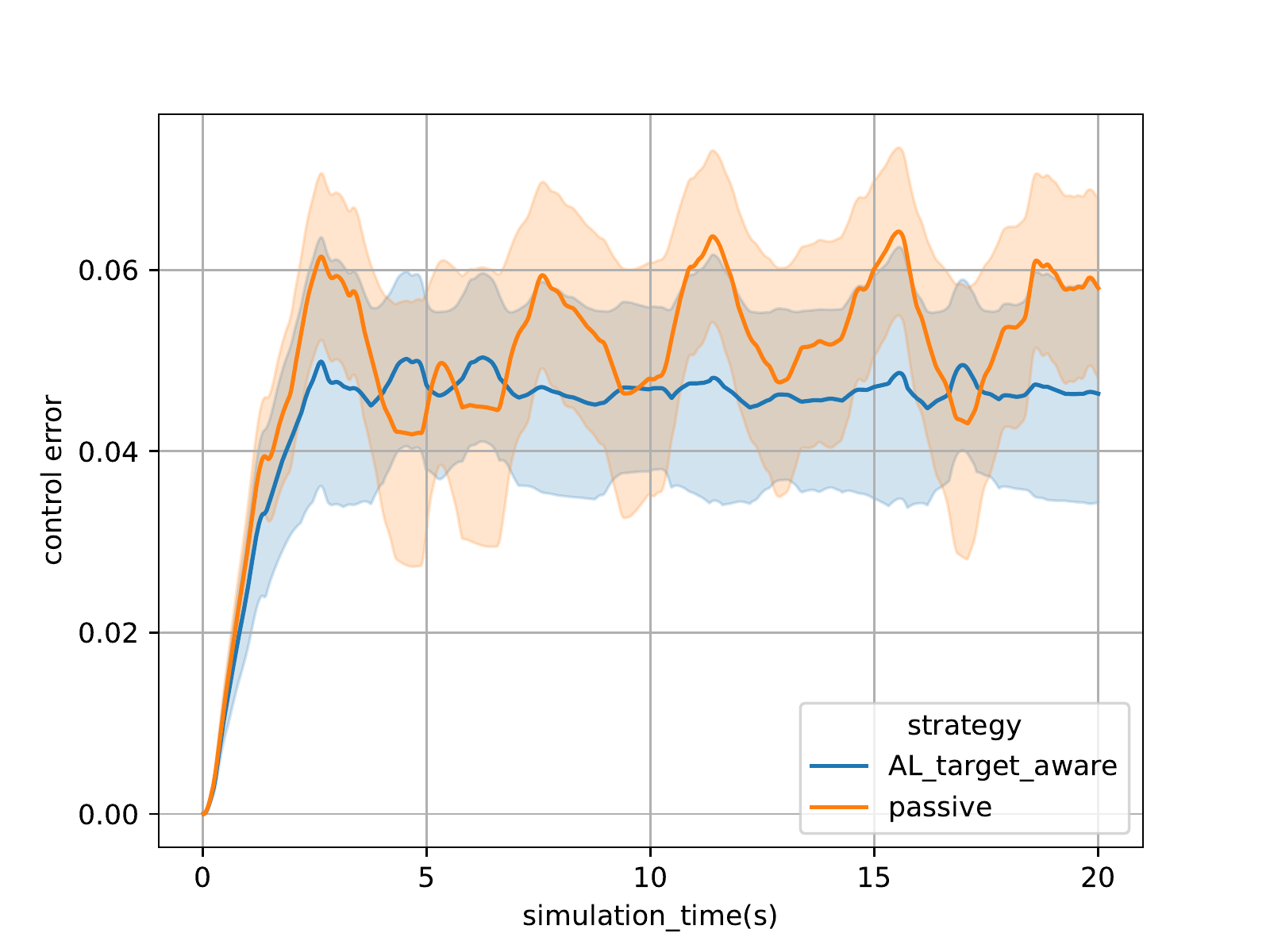}
    \includegraphics[scale = 0.37]{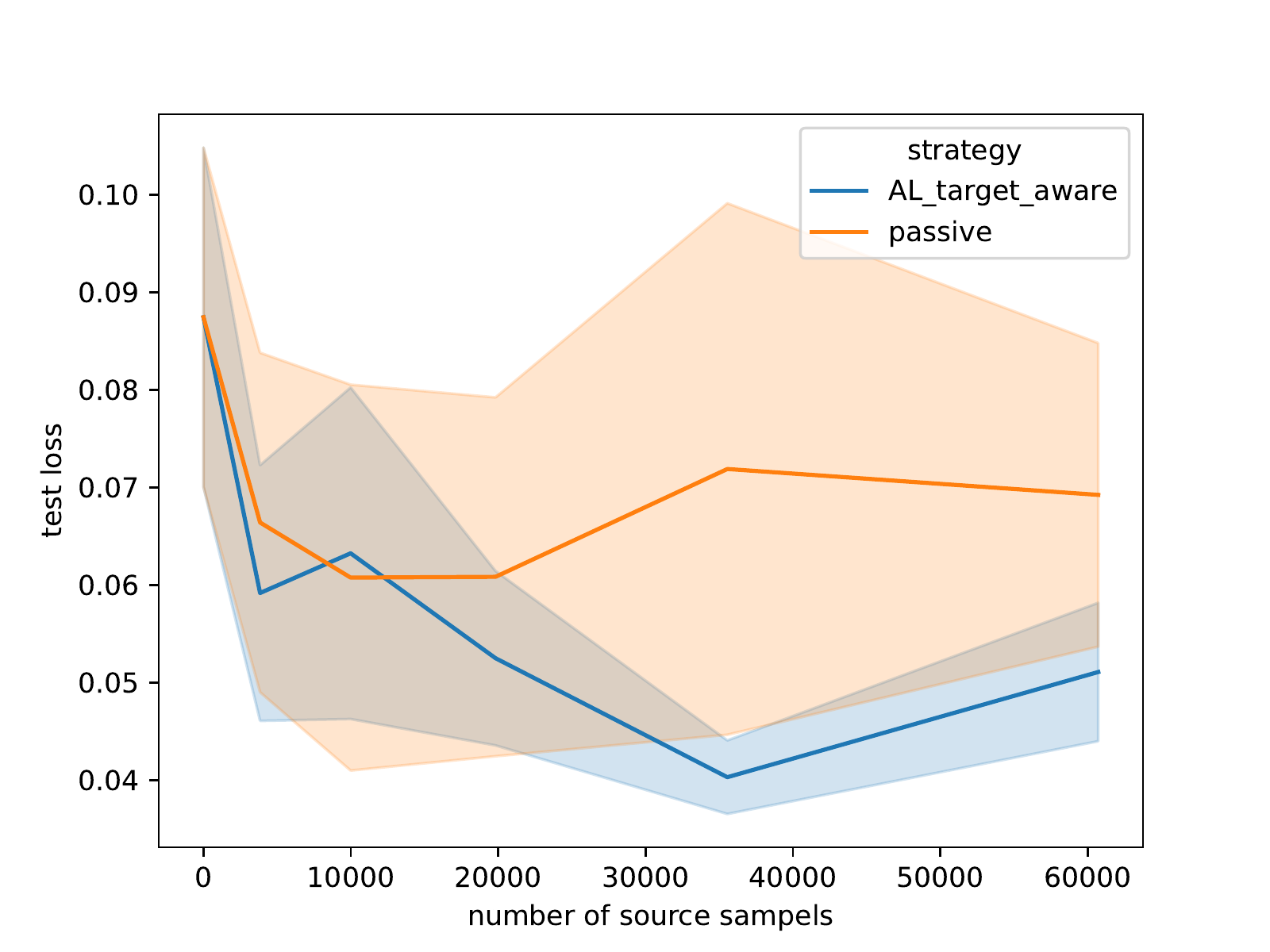}
    \includegraphics[scale = 0.37]{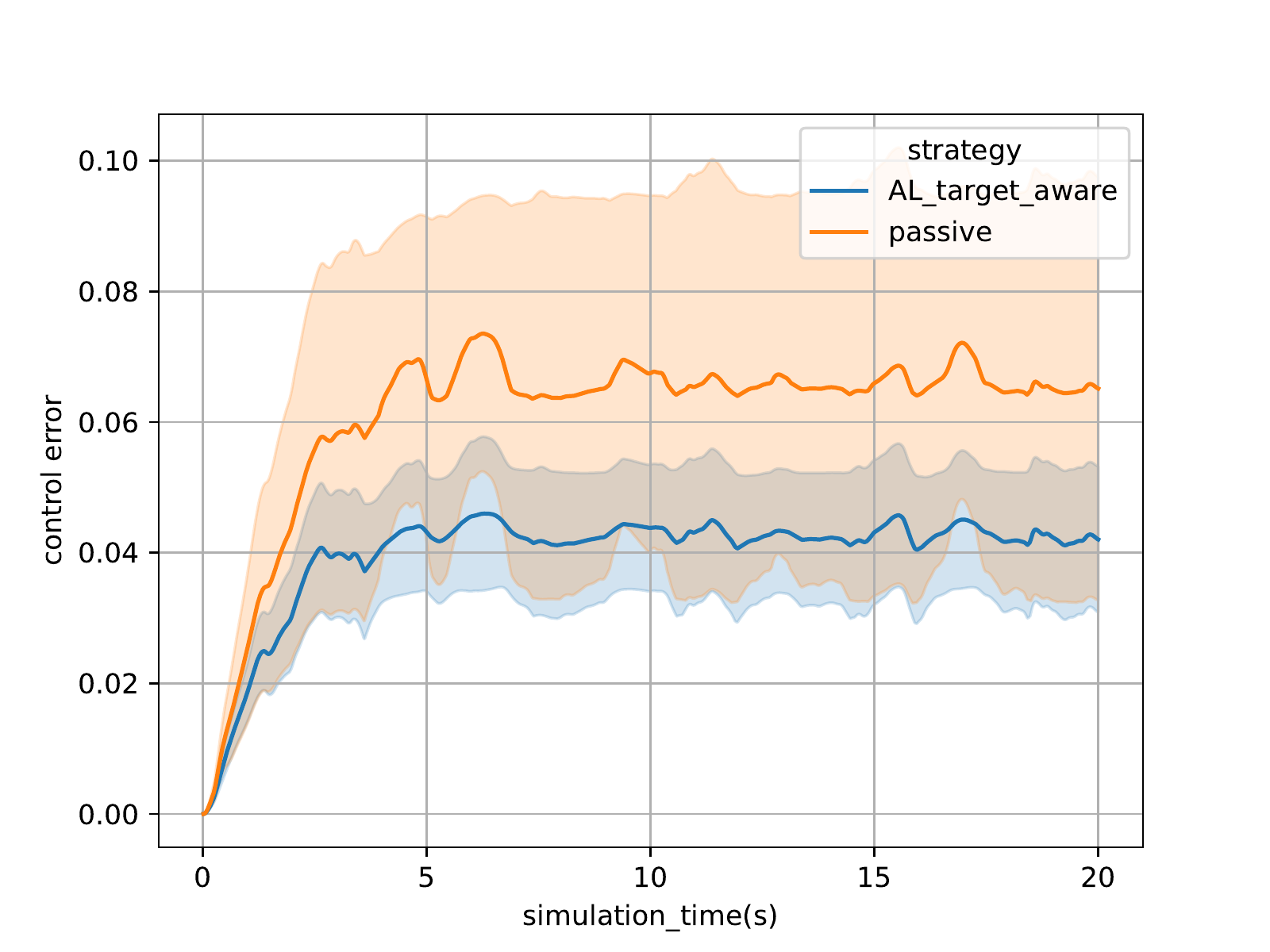}
    \includegraphics[scale = 0.37]{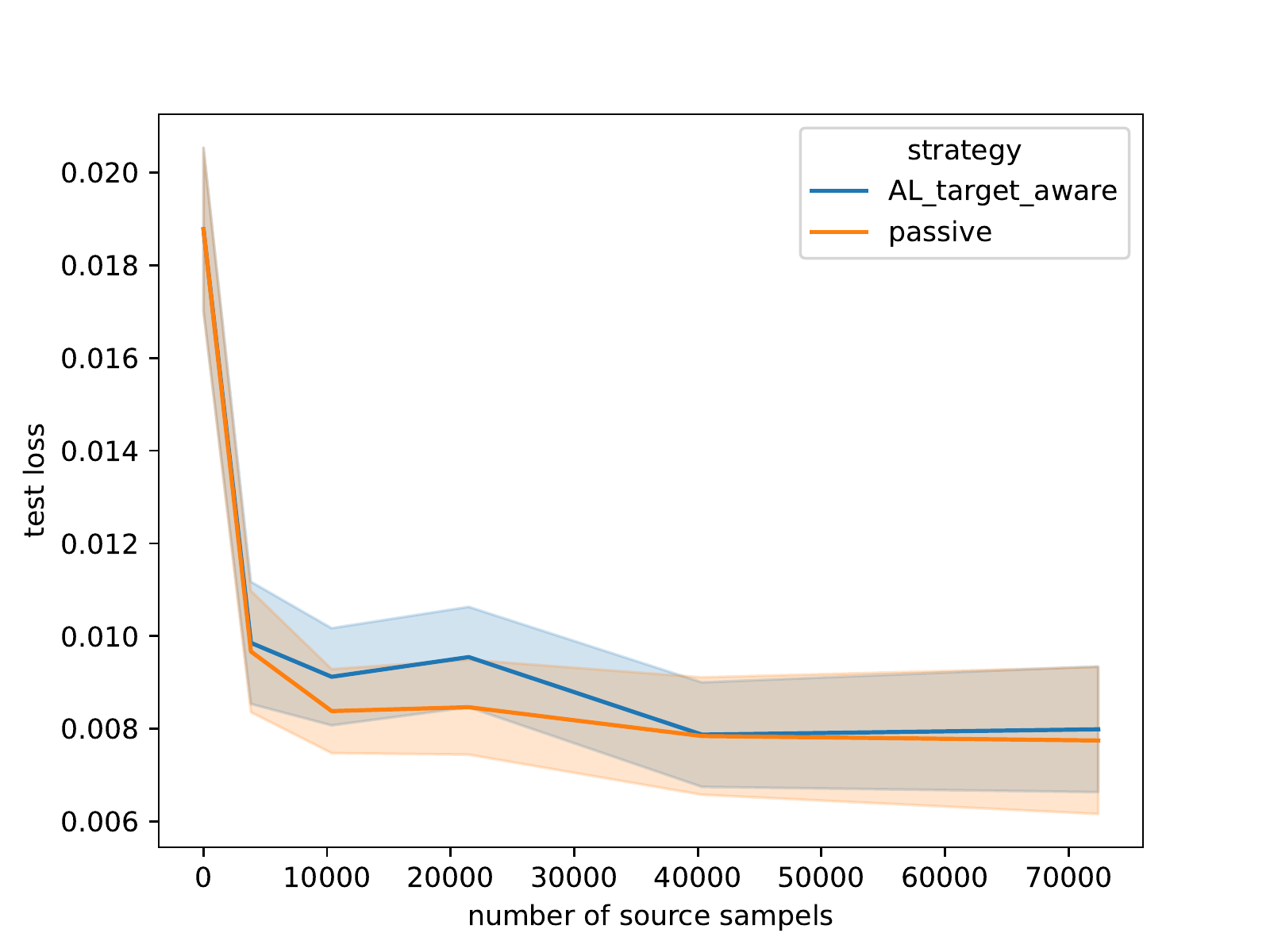}
    \includegraphics[scale = 0.37]{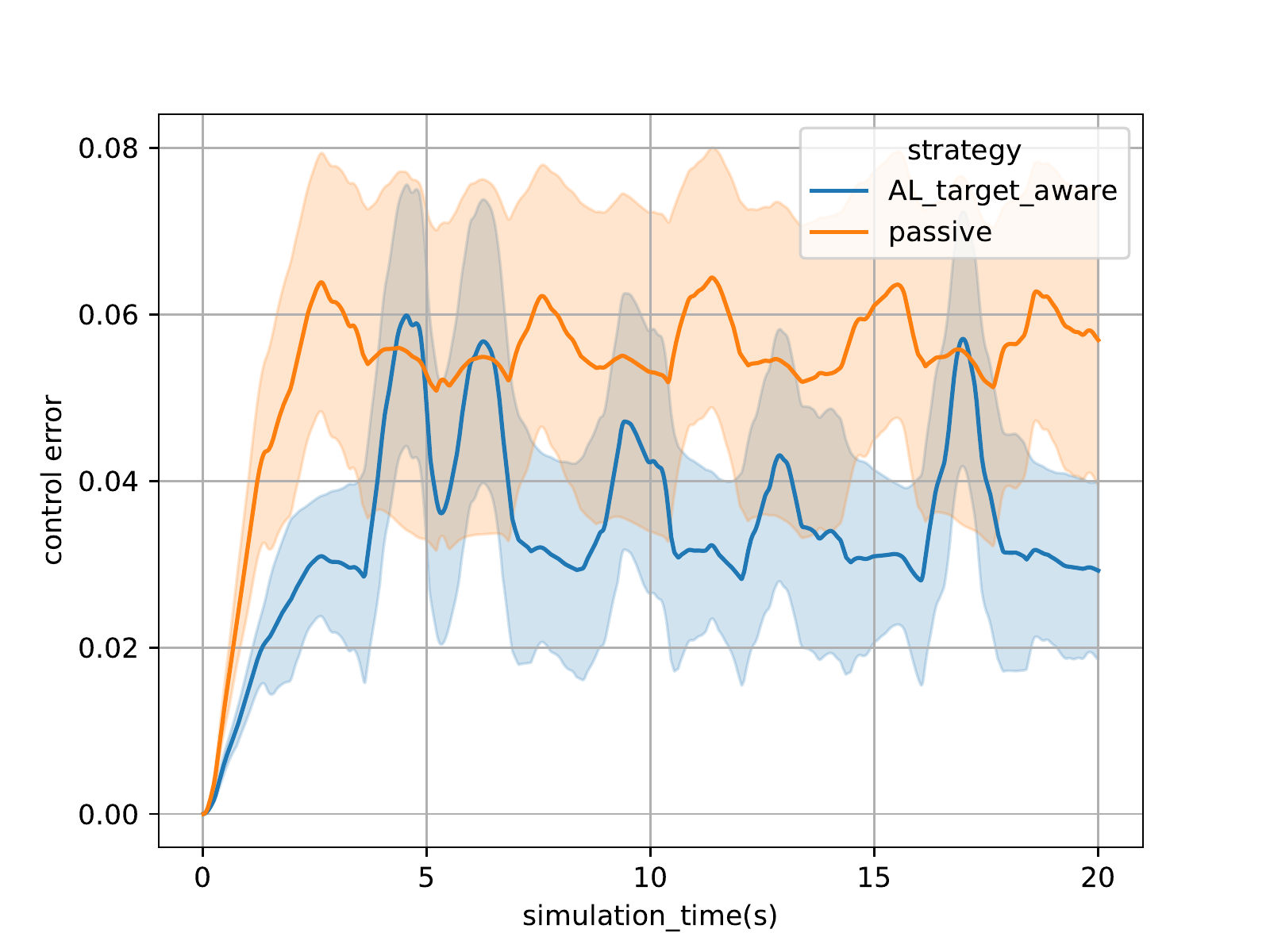}
    \includegraphics[scale = 0.37]{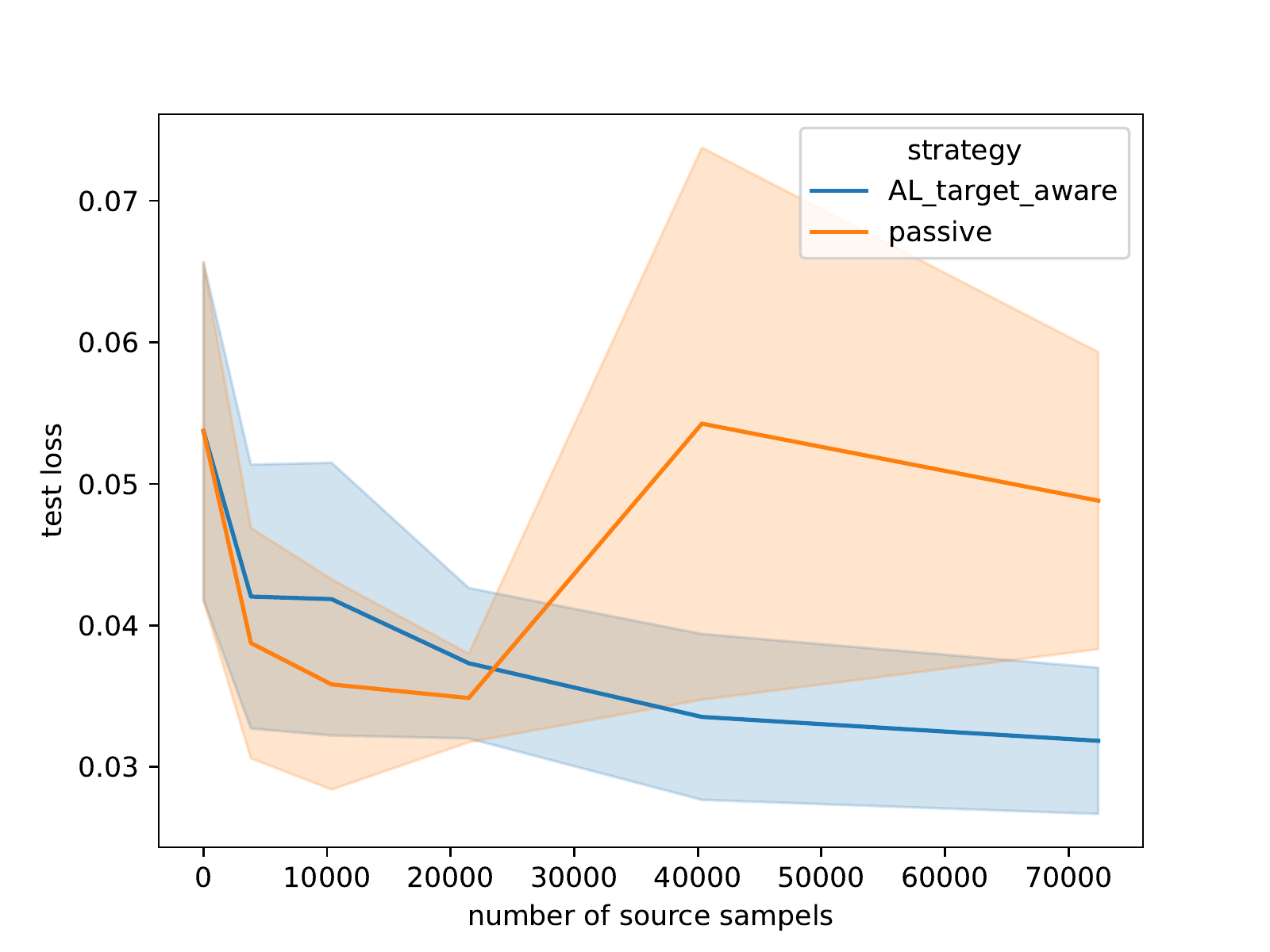}
    \includegraphics[scale = 0.37]{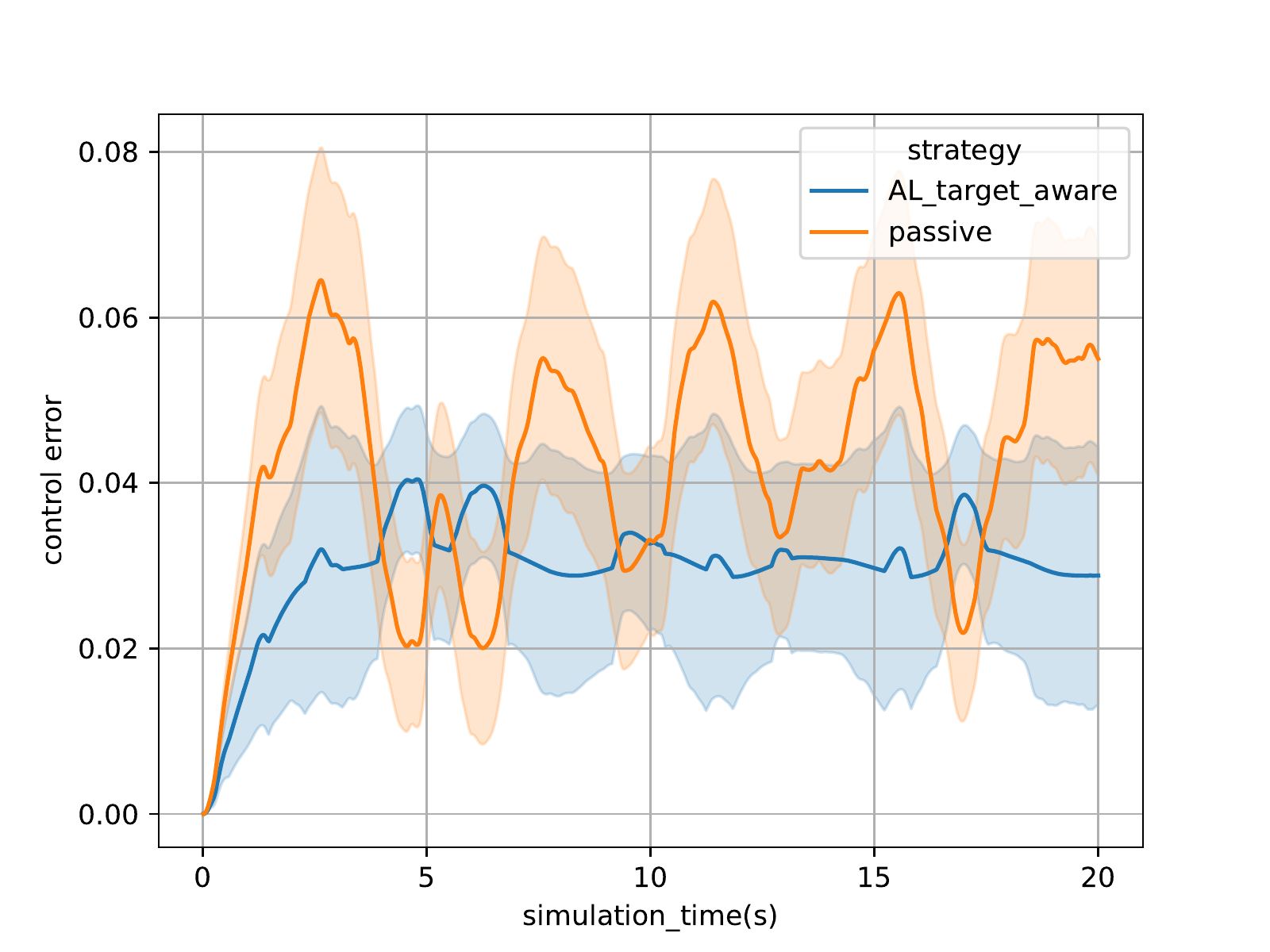}
    \caption{\textbf{Results on pendulum simulator for a specific target}. 
    \textbf{Left: } The test loss of the estimated model $\hat{f}$. The passive strategy suffers from negative transfer while the active strategy steadily decreases. 
    \textbf{Right: } The control error using final output $\hat{f}$. Here we use a model-based nonlinear policy $\pi(x,\hat{f})$. The model learned from active strategy leads to better control performance.
    From top to bottom, we have the unobservable $w_\text{actual\_target}$ as $[0, 0, 0.5, 0, 0.5, 0],[0, 0, 1, 1, -1, 0],[0, -1, 0.5, 0, 0.5, 0],[0, 0.1, 0, -1, 0.5, 0]$. Overall, although AL does not always have a dominating advantage, most times it is more stable and can gain better test loss at the end.
    }
    \label{fig: pendulum 1}
\end{figure}

\subsection{Real-world drone flight dataset}
\label{sec: drone (appendix}

\subsubsection{Settings}

\paragraph{The training model and optimizer} Here we use two layer MLP model as specified below. For the training methods, we do joint-GD as before using AdamW with $lr = 0.005$ and batch\_size$=1000$. Other common parameters are specified in Table.~\ref{table: drone}.
\begin{table}[ht]
\begin{tabular}{|l|l|l|l|l|l|l|l|l|}
\hline
target number & $d_X$  & $d_{\psi_X}$ & $d_W$                & $d_{\psi_W}$   & $k$ & $\phi$ structure                          \\ \hline
500           & 11     & 11           & 18 \text{one-hot}    & 18             & 2   & MLP with hidden layers $[11, 2]$           \\ \hline
\end{tabular}
\caption{Model parameters for drone dataset.}
\label{table: drone}
\end{table}

\paragraph{Data generation} We use the same data as stated in the main paper. 

\paragraph{Detailed implementation for AL strategy} Unlike the previous two settings where the task space $\calW$ is continuous, here we consider a discrete task space. Therefore the Algo.~\ref{algo: main} no longer works. Therefore, here we use a similar technique as the Algorithm proposed in \cite{chen2022active}, which can be seen as a special case under the general Algo.~\ref{algo: main-general}. We want to emphasize that this choice is due to the limitation of real-world datasets, i.e., we can not arbitrarily query $w$ to sample, and the main purpose is to show the potential of such a framework in real-world robotics applications.

\subsubsection{Results}
In the main paper, we provide the result when assuming a bilinear underlying model. Here we further show the effectiveness of our methods under nonlinear $\phi_X$. 

\begin{figure}[ht]
    \centering
    \includegraphics[scale = 0.6]{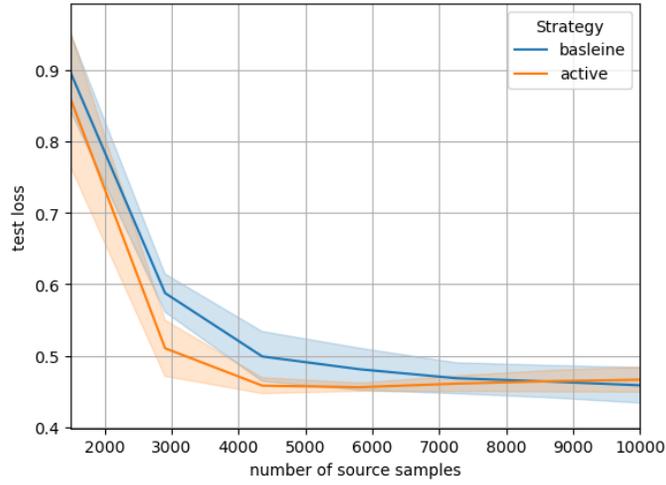}
    \caption{\textbf{Results on the real drone dataset}  with target \texttt{drone\_type\_A\_30\_z} by using a neural net model. Our active strategy could converge faster than the passive strategy in the neural net model setting. Active strategy is able to converge faster than uniform sampling with smaller variances in the latter stage.} 
    \label{fig: drone 1}
\end{figure}

\begin{figure}[ht]
    \centering
    \includegraphics[scale = 0.5]{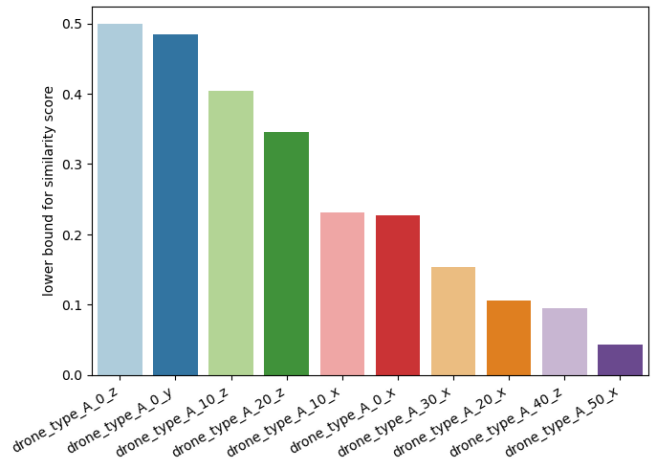}
    \caption{ Top 10 the most similar source tasks. Again, given the target environment, the algorithm successfully finds the other \texttt{drone\_type\_A} environments as relevant sources, which aligns with our observation in the main paper.} 
    \label{fig: drone 2}
\end{figure}